\documentclass{article}

\usepackage[preprint]{neurips_2025}

\usepackage[utf8]{inputenc} 
\usepackage[T1]{fontenc}    
\usepackage{url}            
\usepackage{booktabs}       
\usepackage{amsfonts}       
\usepackage{nicefrac}       
\usepackage{microtype}      
\usepackage{xcolor}         

\usepackage{amsmath,amsfonts,bm}

\newcommand{\domainParam}{\phi}
\newcommand{\domainParamSet}{\Phi}

\usepackage{xspace}

\newcommand{\optimalPi}[0]{\ensuremath{\pi^*}\xspace}
\newcommand{\expectation}[0]{\ensuremath{\mathbb{E}}\xspace}
\newcommand{\dataset}[0]{\ensuremath{\mathcal{D}}\xspace}

\def\eqref#1{equation~\ref{#1}}

\def\1{\bm{1}}

\DeclareMathAlphabet{\mathsfit}{\encodingdefault}{\sfdefault}{m}{sl}
\SetMathAlphabet{\mathsfit}{bold}{\encodingdefault}{\sfdefault}{bx}{n}

\newcommand{\E}{\mathbb{E}}

\newcommand{\R}{\mathbb{R}}

\newcommand{\Cov}{\mathrm{Cov}}

\DeclareMathOperator*{\argmax}{argmax}
\DeclareMathOperator*{\argmin}{argmin}

\usepackage{subfigure}
\usepackage{amsmath}
\usepackage{amssymb}
\usepackage{mathtools}
\usepackage{amsthm}
\usepackage{array}
\usepackage{enumitem}
\usepackage{algorithm}
\usepackage{algpseudocode}
\usepackage{graphicx}
\usepackage{subcaption}
\usepackage{pdflscape}
\usepackage{pdfpages}
\usepackage{wrapfig}
\usepackage{makecell}
\usepackage{mine}
\usepackage[para]{footmisc}

\title{Safe Domain Randomization via Uncertainty-Aware Out-of-Distribution Detection and Policy Adaptation}

\author{%
  Mohamad H.~Danesh\textsuperscript{1}\thanks{Correspondence: \texttt{mohamad.danesh@mail.mcgill.ca}\\}
  \And
  Maxime Wabartha\textsuperscript{1}%
  \And
  Stanley Wu\textsuperscript{1}%
  \And
  Joelle Pineau\textsuperscript{1,2}%
  \And
  Hsiu-Chin Lin\textsuperscript{1}%
}

\begin{document}

\maketitle

\begingroup
  \renewcommand\thefootnote{\arabic{footnote}}
  \footnotetext[1]{McGill University, Mila - Quebec AI Institute}
  \footnotetext[2]{Meta AI}
\endgroup

\begin{abstract}
Deploying reinforcement learning (RL) policies in real-world involves significant challenges, including distribution shifts, safety concerns, and the impracticality of direct interactions during policy refinement. Existing methods, such as domain randomization (DR) and off-dynamics RL, enhance policy robustness by direct interaction with the target domain, an inherently unsafe practice. 
We propose Uncertainty-Aware RL (\ourMethod), a novel framework that prioritizes safety during training by addressing Out-Of-Distribution (OOD) detection and policy adaptation \textit{without} requiring direct interactions in target domain. \ourMethod employs an ensemble of critics to quantify policy uncertainty and incorporates progressive environmental randomization to prepare the policy for diverse real-world conditions. By iteratively refining over high-uncertainty regions of the state space in simulated environments, \ourMethod enhances robust generalization to the target domain without explicitly training on it. We evaluate \ourMethod on MuJoCo benchmarks 
and a quadrupedal robot, 
demonstrating its effectiveness in reliable OOD detection, improved performance, and enhanced sample efficiency compared to baselines.
\end{abstract}


\section{Introduction}\label{sec:intro}
Deploying RL policies in real-world poses substantial safety concerns. 
The online nature of RL often proves impractical due to the risks and costs associated with the continuous interaction with the environment~\citep{sutton2018reinforcement, levine2020offline}. 
The adoption of RL in real-world applications, such as robotics \citep{kober2013reinforcement} and industrial control \citep{spielberg2019toward}, has highlighted the importance of addressing this challenge. 
A widely adopted approach is to train a policy in a simulated environment before deploying it in a real-world setting.
However, real-world environments often deviate from simulations; events such as sensor noise or unmodeled environmental changes can significantly degrade policy performance \citep{zhao2020sim, danesh2021out}.


To mitigate these challenges, recent research in off-dynamics RL focuses on aligning source domain dynamics with those of the target domain. These methods penalize transitions that are overly reliant on source domain characteristics, thereby encouraging policies to generalize better to target domain \citep{eysenbach2021offdynamics}. Furthermore, robust control strategies have been proposed to optimize policies for worst-case scenarios \citep{iyengar2005robust}, though they often face difficulties arising from distributional shifts between the data collection policy and the learned policy, especially during fine-tuning \citep{lee2022offline, zheng2023adaptive}. 
Moreover, domain randomization (DR) enhances policy robustness by training in the source domain with variability to improve transferability to the target domain \citep{tobin2017domain}. This approach has gained traction as it allows training on pre-collected static datasets \citep{levine2020offline}, generated from perturbed environments in simulation to reflect potential target domain conditions. However, designing simulated environments that accurately capture the variability and dynamics of the target domain remains a significant challenge. Furthermore, repetitively testing policies trained through DR on physical hardware introduces safety concerns, as the policies may exhibit unexpected and potentially unsafe behaviors \citep{pmlr-v100-mehta20a}.




\begin{figure*}[t]
    \centering
    \includegraphics[width=0.7\linewidth]{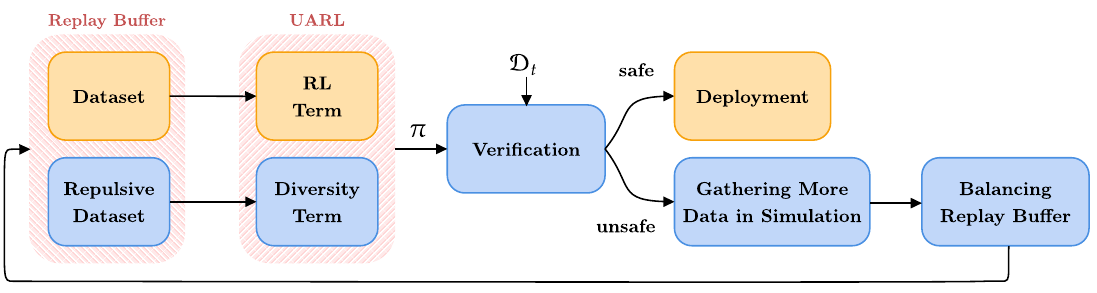}
    \caption{
    Overview of \ourMethod. \colorbox{sbBlue025}{Blue} shows our contributions while \colorbox{sbOrange025}{orange} is adapted from existing RL methods. Agent processes nominal and repulsive datasets through RL and diversity terms. The verification module assesses policy safety in the target domain to avoid deploying policies in OOD scenarios.
    }
    \label{fig:overview}
\end{figure*}

In this paper, we introduce \textbf{U}ncertainty-\textbf{A}ware \textbf{RL} (\ourMethod), a novel approach that encourages safe policy training and deployment under distribution shift \textbf{without direct training on the target domain}. 
\ourMethod quantifies policy uncertainty by measuring the diversity of an ensemble of critics and iteratively refining over high-uncertainty regions of the state space. 
\ourMethod integrates training data with components designed to promote predictive diversity in OOD scenarios and a verification module for safe adaptation (\autoref{fig:overview}).
We validate \ourMethod on MuJoCo benchmark \citep{todorov2012mujoco} and quadruped locomotion \citep{hutter2018towards} assessing safety and sample efficiency. 



\textbf{Contributions}. Our key contribution is an RL framework that \textbf{quantifies policy uncertainty} in designated randomized environments and introduces a \textbf{verification module} that automatically determines when to terminate the progressive randomization process \textit{without} target domain interactions. This mechanism prevents over-randomization and ensures policies meet calculated uncertainty thresholds before deployment, fundamentally differentiating from conventional DR approaches. 
\footnote{Source code: \href{https://github.com/modanesh/UARL}{https://github.com/modanesh/UARL}}

\section{Related Work}\label{sec:related}

\textbf{Offline RL} learns from pre-collected datasets when online interaction is impractical \citep{levine2020offline, kostrikov2022offline}, in contrast to traditional RL's free environment interaction \citep{sutton2018reinforcement}. A key challenge is \textit{distributional shift}, where training data poorly represents target domain. Prior work addresses this through behavior policy alignment \citep{jaques2019way, fujimoto2021minimalist}, conservative critic training \citep{NEURIPS2020_0d2b2061}, and diverse critic ensembles \citep{NEURIPS2021_3d3d286a}. Recent approaches like iterative behavior regularization \citep{ma2023iteratively} and model-based OOD correction \citep{mao2024offline} show promise but require direct target domain interaction or high computational costs. \ourMethod instead combines uncertainty-aware adaptation with simulated fine-tuning for safe policy refinement without target domain interaction. 
Furthermore, methods such as MOBILE \citep{pmlr-v202-sun23q}, PBRL \citep{bai2022pessimistic}, and RORL \citep{yang2022rorl} enhance robustness by penalizing OOD actions during training. 
While existing methods optimize for performance in OOD conditions through avoidance, our work focuses on a complementary objective: \textit{explicit OOD detection and generalization } before deployment. 

\textbf{Off-Dynamics RL} adapts RL policies to new domains by penalizing transitions associated with source domain, promoting transferability to target domain \citep{eysenbach2021offdynamics}. Existing approaches \citep{niu2022trust, xu2023crossdomain, lyu2024crossdomain} rely on target domain interactions for policy refinement, introducing safety risks in high-stakes applications. While conservative regularization mitigates some risks, these methods lack explicit mechanisms for detecting OOD scenarios, leaving them vulnerable to unforeseen shifts \citep{hendrycks2021unsolved}. Our approach addresses these limitations by avoiding target domain interactions during training. 

\textbf{Domain randomization} is a widely used robustness strategy, trains agents in randomized environments to improve generalization \citep{tobin2017domain, andrychowicz2020learning, Lee2020Network, mozian2020learning}. However, there is no clear guidance on how the randomization should be designed and it relies on the risky assumption that policies can be safely validated in the target domain. 
In \autoref{sec:mainalgorithm}, we demonstrate how \ourMethod extends DR by providing a validation module that indicates whether the policy has been exposed to sufficient randomization. 


Further discussion on related work is provided in \aref{app:related}.

\section{Background}
\subsection{Reinforcement Learning}
RL problems commonly model the world as a Markov Decision Process $M = \langle S, A, T, R, \gamma \rangle$, where $S$ is state space, $A$ is action space, $T(s'|s,a)$ is the state transition function, $R(s,a)$ is the reward function and $\gamma \in [0,1)$ is discount factor \citep{bellman_markovian_1957, sutton2018reinforcement}. The objective is to find the optimal policy that maximizes the expected cumulative return: $\optimalPi = \argmax_{\pi} \expectation_{s;a\sim \pi(\cdot|s)}\left[\sum_{t=0}^\infty\gamma^t R(s, a)\right]$ with $\pi(a|s)$ representing the probability of taking $a \in A$ in $s \in S$ under policy $\pi$. 
In off-dynamics RL, the policy is learned using a static, pre-collected dataset $\dataset = \{(s,a,s',r)\}$ obtained by a behavior policy $\pi_b$ under a source dynamics model $T_{\mathrm{src}}$, where $r$ is the immediate reward received after taking $a$ in $s$ and transitioning to $s'$\citep{eysenbach2021offdynamics}. Thus, off-dynamics RL is inherently off-policy and must cope with mismatch between the source dynamics of $\dataset$ and the target dynamics encountered at deployment.
Given $\dataset$, the Q-function update rule is defined as:
\begin{align}\label{eq:update_q}
    Q^{\pi}_{k+1} \leftarrow \argmin_{Q}\expectation_{(s,a,r,s')\sim\dataset} \bigl(Q(s,a) - (r + \gamma\expectation_{a'\sim\pi_k(\cdot|s')}[Q^\pi_k(s',a')])\bigr)^2 && \textcolor{sbBlueEq}{\text{\scalebox{.8}{policy evaluation}}}
\end{align}
and the policy improvement step remains:
\begin{align}\label{eq:policy_improvement}
    \pi_{k+1}(\cdot|s) \leftarrow \argmax_{\pi}\expectation_{s\sim\dataset,a\sim\pi_k(\cdot|s)}[Q_{k+1}(s,a)] && \textcolor{sbBlueEq}{\text{\scalebox{.8}{policy improvement}}}
\end{align}
By iteratively evaluating and improving the policy, with appropriate assumptions, actor-critic RL converges to a near-optimal policy \citep{eysenbach2021offdynamics, levine2020offline}. However, a key challenge in off-dynamics RL is the distributional shift between the source dynamics used to collect $\dataset$ and the true target dynamics under which the learned policy will operate.

\subsection{Domain Randomization} 
DR trains policies over randomized domain parameters $\domainParam \in \domainParamSet$. The domain parameters could include physical properties (mass, center-of-mass, and inertia tensor), sensory noise (from cameras, joint encoders, and force-torque sensors), and/or environmental factors (lighting conditions, texture, and color).
Each set of domain parameters directly affects the state transition function $T_\domainParam(s'|s, a)$.

The training process starts by sampling domain parameters $\domainParam \in \domainParamSet$ at the beginning of each episode, running standard RL using the state transition based on $\domainParam$, and training until convergence. Since the policy and the critics are trained over multiple random parameters, the Q-function that generalizes over all possible environmental parameters will be:
\begin{equation}
\mu(s,a) \approx \expectation_{\domainParam \sim \domainParamSet}[Q^\pi(s,a;\domainParam)]
\end{equation}
As a result, the trained Q-function captures common structures across different domains, reducing overfitting to a single environment. 
This process is often used to train a policy for an unknown target domain $E_t$ with unknown domain parameter $\domainParam_t$. 
If the choice randomization covers the true dynamics of the target domain, $\domainParam_t \in \domainParamSet$, the policy will be able to perform in the target domain.

\subsection{Ensemble Diversification}
\label{sec:background}
Ensembles benefits from an extra loss term to enhance diversity beyond initial randomization (\autoref{sec:related}). To do so, DENN incorporates a diversity term, $\mathcal{L}_{\text{div}}$, into the loss function and leverages \textbf{repulsive locations}, strategically chosen data points near the training distribution boundaries, to encourage ensemble disagreement in high-uncertainty regions \citep{ijcai2020p296}.

Let $\mathcal{X} = \{x_1, x_2, \dots, x_n\}$ denote the inputs and $Y$ be the corresponding output space of the nominal dataset. One approach to defining repulsive locations is to add noise to the input data: $\mathcal{X}' = \{x + \epsilon : x \in \mathcal{X}\}$, where $\epsilon$ represents noise. 
By introducing these repulsive locations, uncertainty is enforced at the boundary of the training distribution and effectively propagates into OOD regions, thereby enhancing OOD detection, which is crucial for robust model performance. To leverage these repulsive locations, DENN employs an ensemble of models, each denoted as $f_i: \mathcal{X} \rightarrow Y$ that are constrained to differ from a reference function $g: \mathcal{X} \rightarrow Y$, which is trained once on the nominal dataset and serves as a consistent baseline for diversity promotion. The conventional supervised learning loss function is then augmented by:
\begin{align}\label{eq:denn_div}
    \mathcal{L}(f_i, g, \mathcal{X}, \mathcal{X}') = \frac{1}{|\mathcal{X}|} \sum_{x,y \in \mathcal{X}} (f_i(x) - y)^2 \color{sbBlueEq} \underbrace{+ \frac{\lambda}{|\mathcal{X}'|} \sum_{x \in \mathcal{X}'} \exp(-|| f_i(x) - g(x) ||^2 / 2 \delta^2)}_{\textcolor{sbBlueEq}{\text{diversity term $\mathcal{L}_{\text{div}}$}}}
\end{align}
where $\lambda$ is the diversity coefficient and $\delta$ controls the diversity between two models at data point $x \in \mathcal{X}'$. The diversity term $\mathcal{L}_{\text{div}}$ penalizes the similarity between $f_i$ and the reference function $g$, which leads to different predictions at inputs $\mathcal{X}'$, thus making $f_i$ diverse with respect to $g$.

In \autoref{sec:training}, we extend the ensemble diversification method into the training regime. Our goal is to strategically choose the state-space so that $f_i$ are diverse when the environment dynamic is OOD.

\textbf{Notation and Lipschitz Assumptions.}
Let $\domainParam\in\domainParamSet$ be the domain-parameter vector (e.g. mass, inertia, etc.), and let $Q^\pi_\domainParam(s,a)$ be the fixed-point Q-function under dynamics $T_\domainParam$. We assume for all $(s,a),(s',a')\in S\times A$:
\begin{equation*}
  |R(s,a)-R(s',a')|\le L_R\|(s,a)-(s',a')\|\quad
  W_1\bigl(T_\domainParam(\cdot\mid s,a),\,T_\domainParam(\cdot\mid s',a')\bigr)\le L_T\|(s,a)-(s',a')\|
\end{equation*}
where $L_R$ is the reward's Lipschitz constant and $L_T$ the transition kernel's (w.r.t. the 1-Wasserstein metric).

\begin{wrapfigure}{r}{0.2\linewidth}
    \centering
    \includegraphics[width=\linewidth]{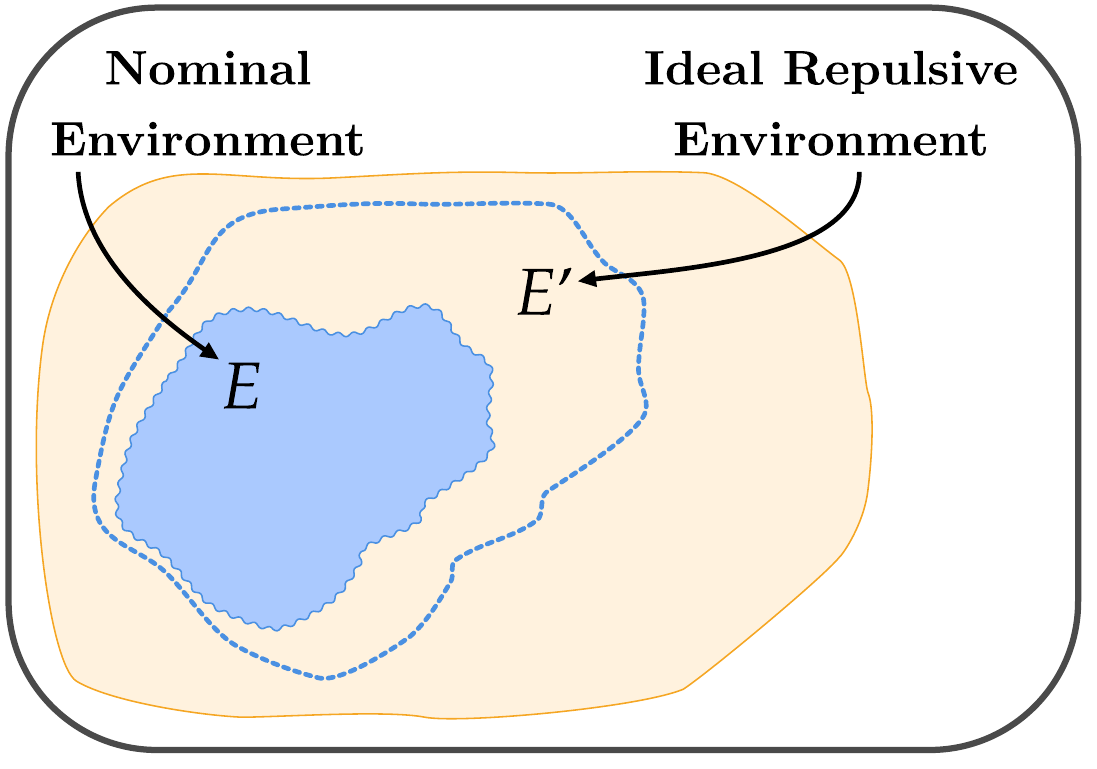}
    \caption{Conceptual visualization of nominal $E$ and repulsive $E'$ environments.}
    \label{fig:nominal-repulsive}
\end{wrapfigure}

\section{Uncertainty-Aware RL}\label{sec:uarl}
Our goal is to train a policy that can be deployed in an unknown target domain $E_t$ without directly refining or validating in the target domain. 
We assume that we have access to a limited target domain dataset $\dataset_t$, where $\dataset_t$ could act as a proxy for the target domain $E_t$, allowing us to evaluate policy performance on $E_t$ without the risk of unsafe deployments.

\begin{definition}[Nominal and Repulsive Environments]
We adapt the terminology of DENN to our problem.
We define the \textit{nominal environment} $E$ as the environment where the policy can be trained and deployed, and the \textit{repulsive environment} $E'$ as the regions where variabilities are injected into the dynamics of $E$ via parameter randomization as illustrated in \autoref{fig:nominal-repulsive}. 
We assume access to datasets $\dataset$ and $\dataset'$, collected by the same behavior policy from $E$ and $E'$, respectively.
\end{definition}
\begin{definition}[Out-of-Distribution]\label{def:ood}
A target domain $E_t$ with parameter $\domainParam_t$ is \textbf{ID} if
$\domainParam_t\in\mathcal{C}$; otherwise it is \textbf{OOD}.
\end{definition}
In practice, we define the coverage set as $\mathcal{C}=\{\domainParam\mid \text{KL}(\hat p(\domainParam)\|p_\domainParamSet)\le\tau_{\text{KL}}\}$, i.e. the subset of parameter space where the empirical density $\hat p(\domainParam)$ (estimated from the repulsive dataset $\dataset'$) deviates from the intended sampling density $p_{\domainParamSet}$ by at most $\tau_{\mathrm{KL}}>0$. Here, $\tau_{\mathrm{KL}}$ is the maximum KL-divergence we tolerate to call $\domainParam$ ``in-distribution,'' and $\tau$ denotes the critic-variance threshold used for deployment gating (more details in \autoref{sec:exp_uncertainty}).

\begin{proposition}[Critic gap under OOD]\label{proposition:critic}
If the target parameter $\domainParam_t\notin\domainParamSet$ (i.e. OOD), then
\begin{equation}
  \|Q^\pi_{\domainParam_t} - Q^\pi_{\domainParam}\|_\infty
  \ge\frac{\gamma\bigl(L_R + L_T\|Q^\pi_{\domainParam_t}\|_{\mathrm{Lip}}\bigr)}{1+\gamma}
         \|\domainParam_t - \domainParam\|
\end{equation}
In particular, this gap is strictly positive whenever $\domainParam_t\neq\domainParam$.
\end{proposition}
Thus, any nonzero shift $\|\domainParam_t-\domainParam\|$ forces at least a proportional change in the critic's fixed point, which our ensemble variance will reliably detect. Proof is provided at \aref{app:theory_effects} (\proofref{proof:critic}).

\subsection{Offline RL with Diverse Critics}\label{sec:training}

In this section, we extend the diversity term of DENN for differentiating the nominal and the repulsive environments defined in \autoref{sec:uarl}.
We address two critical departures from the supervised learning setting: the absence of labeled data, and the temporal structure of value estimation. Crucially, our approach eliminates DENN's requirement for a pretrained reference function by leveraging the \textit{Bellman target} as both the prediction objective and implicit OOD reference. 
We employ an ensemble of Q-functions $\{Q_i\}_{i=1}^N$ and
 modify $\mathcal{L}_{\text{div}}$ in \autoref{eq:denn_div} to the following:
\begin{align}\label{eq:denn_rl}
    \mathcal{L}^{\text{RL}}_{\text{div}} = \sum_{i} \exp \left( - || Q_i(s,a) - (r + \gamma Q_i(s',\pi(a'|s'))) ||^2 / 2 \delta^2 \right)
    & ; (s,a,s',r) \sim \dataset'
\end{align} 
Based on temporal difference learning, we can consider $r+\gamma Q_i(s',\pi(a'|s'))$ as the learner's target value, and compare that against the predicted value of $Q_i(s,a)$. 
By combining \autoref{eq:update_q} and \autoref{eq:denn_rl}, our overall policy evaluation step for each $Q_i$ in the ensemble will be:
\begin{align}\label{eq:our_loss}
    & Q^{\pi}_{i,k+1} \leftarrow \argmin_{Q_i} \expectation_{s,a,s',r \sim \dataset;a'\sim \pi(\cdot|s')} \left[ ( Q_i(s,a) - (r + \gamma Q^\pi_i (s', a'))^2 \right] \nonumber \\
    & \hspace{1.4cm} \begin{aligned}[t]
        & \color{sbBlueEq} \underbrace{\textcolor{sbBlueEq}{+ \lambda (\expectation_{s,a,s',r \sim \dataset';a'\sim \pi(\cdot|s')} \left[\sum_i \exp(-|| Q_i(s,a) - (r+\gamma Q_i(s',a')) ||^2 / 2 \delta^2)\right])}}_{\textcolor{sbBlueEq}{\text{diversity term $\mathcal{L}^{\text{RL}}_{\text{div}}$}}}
    \end{aligned}
\end{align}
\autoref{eq:denn_rl} and \autoref{eq:our_loss} form the core of our work, promoting diversity among the Q-functions in our ensemble. The diversity term $\mathcal{L}^{\text{RL}}_{\text{div}}$ in \autoref{eq:our_loss} encourages each $Q_i$ to diverge from its own Bellman target on the repulsive dataset $\dataset'$, while the first term ensures accurate Q-value estimation on the nominal dataset $\dataset$, 
providing a clearer separation between  $\dataset$ and $\dataset'$.



\begin{assumption}[Well-trained critic ensemble]
\label{assump:trained_ensemble}
Let $\{Q_i^\pi\}_{i=1}^N$ be an ensemble produced by optimizing \autoref{eq:our_loss} until the training loss stabilises.
\end{assumption}

Under \assumref{assump:trained_ensemble} we know that the policy will perform well in the nominal environment $E$ but not necessarily in the repulsive environment $E'$. As enforced by the repulsive term in our objective, the ensemble $Q^{\pi}$ exhibits low variance on state-action pairs $(s, a)$ drawn from $E$ and high variance for those from $E'$. Our objective is to determine whether the policy can generalize to a target domain $E_t$ given limited samples from it.


Further theoretical discussion can be found in \aref{app:theory}.

\subsection{Fine-tuning Policy With Balancing Replay Buffer}\label{sec:finetuning}

When training an RL policy with data from multiple domains, such as source and target domains in off-dynamics RL, the differences in their distributions can hinder effective policy improvement \citep{nakamoto2023calql}. To address this, we weight each sample by the critic's uncertainty, stabilizing fine-tuning and accelerating convergence. Specifically, each sample $(s,a)$ is assigned a weight based on:
\begin{equation}
w(s, a) = 
\begin{cases}
    \sigma^{-2}(\{Q_i(s, a)\}_{i=1}^N) & \text{if } (s, a) \in \dataset' \\
    \sigma^{2}(\{Q_i(s, a)\}_{i=1}^N) & \text{if } (s, a) \in \dataset
\end{cases}
\end{equation}
where $\sigma^2(\cdot)$ denotes the variance of critics' predictions. Our method exclusively utilizes source domain simulations, avoiding target domain fine-tuning. The balancing replay buffer explicitly weights transitions by their ensemble disagreement: increasing sampling probability for high-uncertainty states from diverse simulations while downweighting low-uncertainty regions. This dynamic weighting maintains policy stability through familiar scenarios while prioritizing adaptation to challenging OOD conditions encountered across simulations.

\tikzexternaldisable
\begin{algorithm*}[t]
\small
\caption{\ourMethod}\label{alg:overall}
\begin{algorithmic}[1]
    \BeginBox[fill=gray!20]
    \State {\bfseries Require:} 
    \Comment{Before training}
    \Statex \quad target domain dataset $\dataset_t$, behavior policy $\pi_b$, original environment $E_0$, threshold $\tau$, ensemble size $N$.
    \EndBox
    \BeginBox[fill=gray!10]
    \State $\dataset_0 \gets$ rollouts of $\pi_b$ over $E_0$ \Comment{\autoref{sec:uarl}}
    \State $E_1 \gets$ expanding $E_0$ by increasing the parameter randomization range
    \State $\dataset_1 \gets$ rollouts of $\pi_b$ over $E_1$
    \EndBox
    \BeginBox[fill=gray!20]
    \State Train policy $\pi_0$ and critics $\{Q_i(s,a)\}_{i=1}^N$ with \autoref{eq:our_loss} and \autoref{eq:policy_improvement} using nominal\\ and repulsive datasets: $(\dataset_0, \dataset_1)$ \Comment{\autoref{sec:training}}
    \State $\sigma^2$ $\gets$ variance of the critics over $\dataset_t$, calculated as $\frac{1}{N} \sum_{j=0}^{N-1} \left( Q_{0}^{(j)} - \frac{1}{N} \sum_{k=0}^{N-1} Q_{0}^{(k)} \right)^2$
    \State $i \gets 0$
    \EndBox
    \BeginBox[fill=gray!10]
    \While{$\sigma^2$} $> \tau$ \Comment{Continue until policy is safe}
        \State $E_{i+2} \gets$ expanding $E_{i+1}$ by increasing the parameter randomization range \Comment{\autoref{sec:uarl}}
        \State $\dataset_{i+2} \gets$ rollouts of $\pi_i$ over $E_{i+2}$
        \State $\dataset_{\text{nom}} \gets$ balanced replay buffer with $\dataset_0, \dataset_1, \dots, \dataset_{i+1}$ \Comment{\autoref{sec:finetuning}}
        \State Fine-tune $\pi_{i+1}$ and $Q_{i+1}$ with \autoref{eq:our_loss} and \autoref{eq:policy_improvement} using nominal\\ and repulsive datasets: $(\dataset_{\text{nom}}, \dataset_{i+2})$ \Comment{\autoref{sec:training}}
        \State $\sigma^2$ $\gets$ variance of the critics over $\dataset_t$, calculated as $\frac{1}{N} \sum_{j=0}^{N-1} \left( Q_{i+1}^{(j)} - \frac{1}{N} \sum_{k=0}^{N-1} Q_{i+1}^{(k)} \right)^2$
        \State $i \gets i + 1$
    \EndWhile
    \State Deploy policy $\pi_i$
    \EndBox
    \Statex \BoxedString[fill=gray!20]{\textcolor{gray}{$\triangleright$ Detailed hyperparameter explanations provided in the \aref{app:hyperparams}}}
\end{algorithmic}
\end{algorithm*}
\tikzexternalenable

\subsection{\ourMethod Algorithm}
\label{sec:mainalgorithm}


\ourMethod builds upon a curriculum of DR.
Starting from an initial environment $E_0$, we introduce a repulsive environment $E_1$ through parameter randomization and gradually expand the randomization range. 
This progression nudges the agent's exploration towards the target environment $E_t$ (\autoref{fig:exploration}). Importantly, $E_{t'}$ represents a subset of $E_t$ sufficient for effective policy performance. Focusing on $E_{t'}$ balances comprehensive scenario coverage and computational efficiency, avoiding the exhaustive exploration of DR.

Initially, we execute a behavior policy for $n$ episodes in an unaltered simulation environment $E_0$, generating an offline dataset $\dataset_0$. We then introduce slight randomization to a single parameter in the environment to reach to $E_1$, collecting a repulsive dataset $\dataset_1$ using the same behavior policy. These two datasets, $\dataset_0$ (nominal) and $\dataset_1$ (repulsive), are used in our loss function in \autoref{eq:our_loss}: $\dataset_0$ contributes to the standard RL loss, while $\dataset_1$ informs the diversity term $\mathcal{L}^{\text{RL}}_{\text{div}}$.

We detect OOD scenario through the critics' variance on $\dataset_t$. 
Low ensemble variance indicates $\dataset_t$ is aligned with the nominal dataset ($\dataset_0$) and high variance suggests a deviation from it. 
When the uncertainty measure falls below a threshold $\tau$ (\autoref{sec:exp_uncertainty}), $\dataset_t$ is considered sufficiently aligned with $E_0$, and the policy is deemed ready for deployment.

\begin{wrapfigure}{}{0.25\linewidth}
    \centering
    \includegraphics[width=\linewidth]{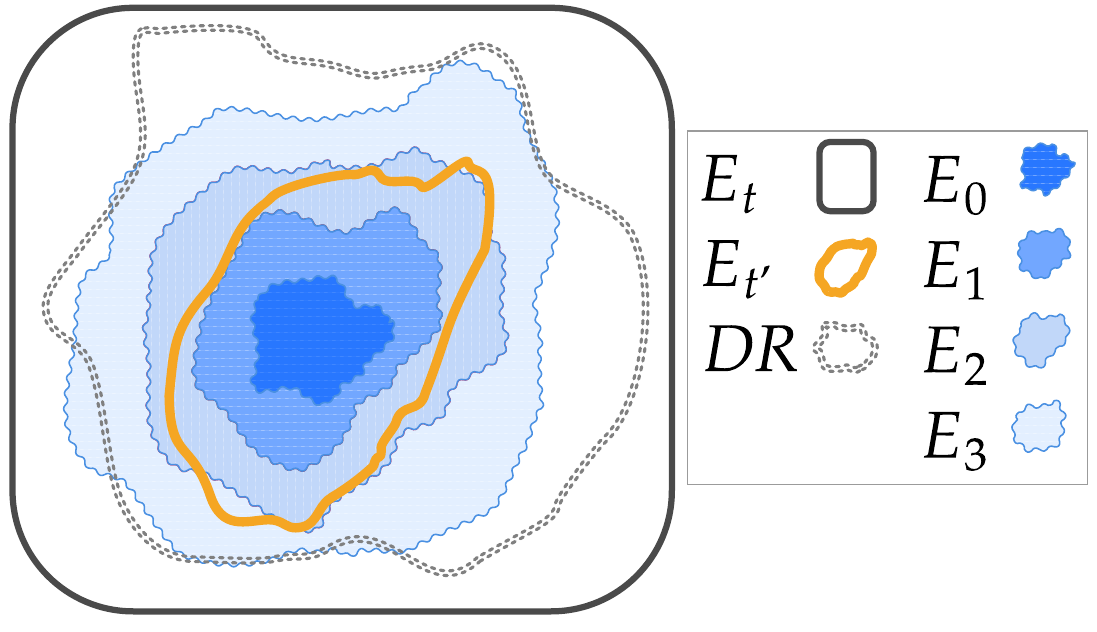}
    \caption{Conceptual visualization of state space expansion from $E_0$ (initial) to $E_1, \ldots,E_3$ by increasing randomization. \$E\_t\$ denotes the theoretical target environment; \$E\_{t'}\$ is where the agent performs effectively.}
    \vspace{-1.em}
    \label{fig:exploration}
\end{wrapfigure}


If the disagreement of the critics is above $\tau$, we expand the randomized parameter to reach to $E_2$ from $E_1$ and collect the new repulsive dataset $\dataset_2$. We combine $\dataset_0$ and $\dataset_1$ as the nominal datasets using our balancing replay buffer method (\autoref{sec:finetuning}) and fine-tune the policy. This process continues iteratively until the uncertainty criterion is satisfied for deployment. \autoref{alg:overall} outlines the iterative fine-tuning process, leading to a policy with sufficient certainty for deployment. 

If the condition in line $8$ of \autoref{alg:overall} is never violated, it suggests ineffective DR. High uncertainty despite varied training environments indicates that the agent has not generalized to target domain conditions, possibly due to inadequate randomization or a mismatch with target domain data. In such cases, training should be stopped, and the DR or target domain data needs revisiting. Rather than a shortcoming of \ourMethod, this is a limitation of the DR or data quality.

In contrast to prior work on DR, our approach does not require direct policy validation in the target domain. Similarly, unlike existing methods in off-dynamics RL, our method avoids direct policy refinement in the target domain. Both characteristics are critical for ensuring safety in real-world deployment.
Moreover, the formulation presented in \autoref{eq:our_loss} introduces only a ``diversity term'' into the policy evaluation step; thus, \ourMethod seamlessly integrates into any offline RL algorithm. In \autoref{sec:exp}, we incorporate \ourMethod into CQL \citep{NEURIPS2020_0d2b2061}, AWAC \citep{nair2020awac}, and TD3BC \citep{fujimoto2021minimalist}.


\section{Experiments}\label{sec:exp}
We compare \ourMethod~ against several state-of-the-art RL methods to address the following key questions: 
(1) How effectively does \ourMethod differentiate between ID and OOD samples? (\autoref{sec:exp_uncertainty})
(2) How effective is the validation module in reducing policy uncertainty while enhancing performance as a deployment gatekeeper? (\autoref{sec:exp_gatekeeper})
(3) How \ourMethod can be applied to real-world settings? (\autoref{sec:anymal}) Further ablation studies are provided in \aref{app:ablation}.

\textbf{Experimental setups}. 
We validate our work with MuJoCo benchmarks. We introduce systematic randomization to three key parameters in each environment: initial noise scale, friction coefficient, and the agent's mass. By isolating the effect of each parameter, we assess the method's adaptability to dynamic conditions. Specifically, the noise scale is multiplied by $10^2$, while both the friction coefficient and the agent's mass are increased proportionately per iteration (see \aref{app:random_params} for details). To ensure robust results, we aggregate the outcomes of $5$ random seeds, reporting the mean and a $95\%$ confidence interval (CI). 
 
\textbf{Dataset}.
While D4RL \citep{fu2020d4rl} is standard for offline RL, it lacks behavior policy checkpoints needed for \ourMethod's repulsive dataset generation. We created a new dataset using MuJoCo \texttt{v4} with Gymnasium \citep{towers_gymnasium_2023}, extending beyond D4RL's environments to include \ant and \swimmer. Using converged SAC \citep{haarnoja2018soft} as the behavior policy, we generated $999$ rollouts per environment. Our datasets' metrics align with D4RL's, suggesting generalizable findings. Additionally, we created $\dataset_t$ by running the behavior policy for $99$ episodes with randomized parameters to simulate target domain variability.

\begin{figure*}[t]
    \centering
    \includegraphics[width=\textwidth]{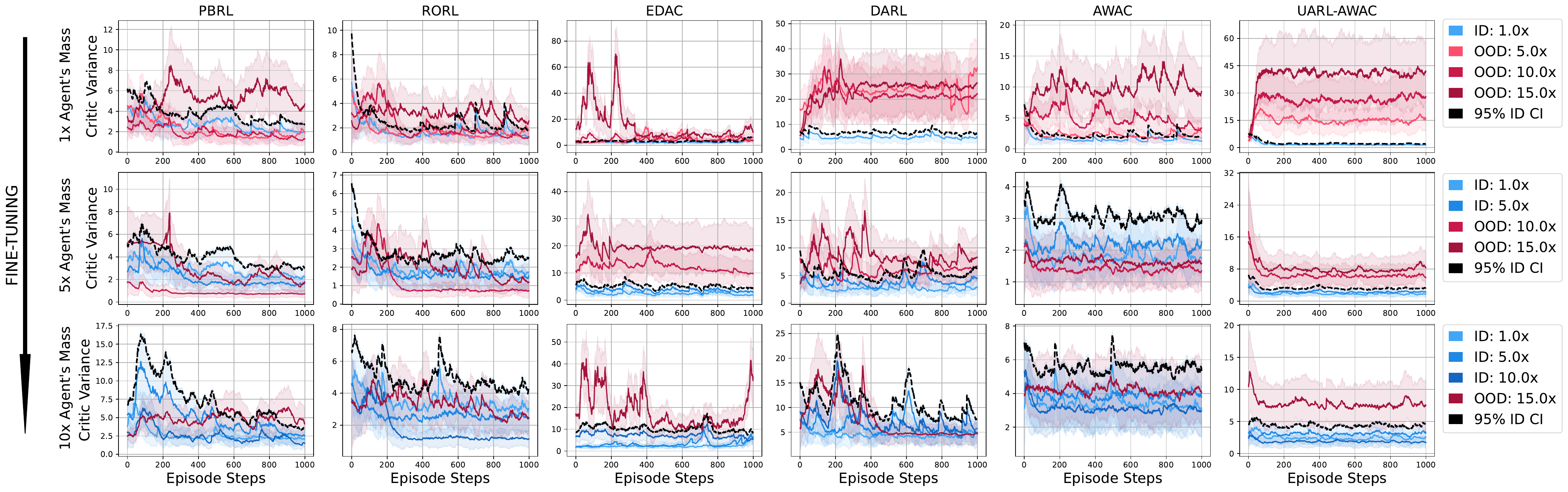}
    \caption{Critic variance across $100$ rollouts in the \Ant environment for PBRL, RORL, EDAC, DARL, and AWAC-based methods. The randomized hyperparameter is agent mass. Each column represents an algorithm, and each row represents a fine-tuning iteration with an expanded ID range by multiplying the agent's mass vector by a constant: $1\text{x} \rightarrow 5\text{x} \rightarrow 10\text{x}$. The black line indicates the $95\%$ CI of critic variances for ID samples, serving as an OOD detection threshold. 
    Clear ``OOD-awareness'' places blue curves below and red curves above the black curve; \ourMethod-AWAC does so consistently, unlike the baselines.}
    \vspace{-1.0em}
    \label{fig:uncertainty_baselines}
\end{figure*}

\textbf{Baselines}.
We benchmark \ourMethod against three method categories: \textbf{Offline and Offline-to-Online Baselines:} CQL \citep{NEURIPS2020_0d2b2061}, which learns conservative value estimations to address overestimation issues; AWAC \citep{nair2020awac}, which enforces policy imitation with high advantage estimates; TD3BC \citep{fujimoto2021minimalist}, an offline RL method that combines TD3 \citep{fujimoto2018addressing} with behavioral cloning; and EDAC \citep{NEURIPS2021_3d3d286a}, which learns an ensemble of diverse critics by minimizing the cosine similarity between their gradients. \ourMethod is applied to CQL, AWAC, and TD3BC, but not to EDAC due to its inherent diversity mechanism. These baselines are implemented based on the CORL framework \citep{tarasov2022corl} without additional hyperparameter tuning.
\textbf{Off-Dynamics Baselines:} We compare against H2O \citep{niu2022trust}, VGDF \citep{xu2023crossdomain}, and a robust algorithm leveraging offline data, RLPD \citep{ball2023efficient}. These methods were selected for their ensemble-based critics and focus on robust policy learning. However, they differ fundamentally from \ourMethod in their assumptions about target domain accessibility and data usage. These methods are implemented based on the ODRL benchmark \citep{lyu2024odrlabenchmark}. 
\textbf{Uncertainty-Aware Baselines:} We evaluate \ourMethod against PBRL \citep{bai2022pessimistic} and RORL \citep{yang2022rorl}, designed to enhance robustness, as well as EDAC \citep{NEURIPS2021_3d3d286a} and DARL \citep{zhang2023darl}, which focus on improving uncertainty estimation. Please refer to \aref{app:related} for details of the baseline methods.

\textbf{Evaluation Criteria}.
We evaluate the baselines and \ourMethod based on: (1) OOD accuracy, defined as the model's ability to differentiate ID and OOD environments; (2) Cumulative return on evaluation environments during training, which measures overall performance; (3) Sample efficiency, defined as the number of samples required to converge to a reasonable performance \citep{huang2024open}.

\textbf{Diversity Loss Hyperparameters}.
The few hyperparameters $\lambda$ and $\delta$ in $\mathcal{L}^{\text{RL}}_{\text{div}}$ (\autoref{eq:our_loss}) balance the RL objective with critic diversity and are mostly adaptive and set intuitively. We set $\lambda$ adaptively to maintain $\mathcal{L}^{\text{RL}}_{\text{div}}$ at $\approx10\%$ of the total loss, ensuring meaningful diversity without overwhelming the primary objective. It is noteworthy that when $\lambda = 0$, \autoref{eq:our_loss} reduces to the standard RL loss. 
Stochastic initialization causes critics to diverge in underrepresented state-action regions.
Following DENN, we selected $\delta=10^{-2}$ from the recommended range $[10^{-3}, 10^{-1/2}]$ \cite{ijcai2020p296}. These choices provide a strong baseline, with potential for further improvement through fine-tuning (\autoref{app:hyperparams_ablation}).

\subsection{OOD Detection}\label{sec:exp_uncertainty}

A key objective of \ourMethod is to enhance uncertainty awareness in RL policies. To assess OOD detection, we set a threshold based on the \nth{5} percentile of the variances observed in $100$ ID rollouts, which refer to rollouts generated in environments with variations the agent has encountered during training. During deployment, data points exceeding this threshold are classified as OOD, allowing us to evaluate the performance on recognizing OOD situations.


\autoref{fig:uncertainty_baselines} demonstrates \ourMethod-AWAC's superior OOD detection capabilities in the \Ant environment with randomized agent mass, outperforming standard AWAC and baselines. 
The $95\%$ CI of critic variances on ID samples serves as a $95\%$ hypothesis-testing threshold.
This is unlike PBRL or RORL, which conflate ID and OOD dynamics, leading to overconfident policies in uncertain regions. Moreover, while EDAC's and DARL's OOD detection performance is incosistent, \ourMethod maintains consistent critic variance distinctions between ID and OOD states. 
For example, EDAC achieves strong separation in the \nth{1} fine-tuning iteration (\nth{2} row) but rarely seprates ID from OOD by the \nth{2} iteration (\nth{3} row), and DARL excels only in offline training (\nth{1} row). In contrast, \ourMethod adapts dynamically as the ID mass range expands across three fine-tuning stages. This adaptability ensures robust OOD detection even as environmental knowledge grows, sustaining high critic variance for OOD states while refining uncertainty estimates. Such reliability in adversarial or distributionally shifted scenarios underscores \ourMethod's advantages for safe real-world deployment, where baseline methods exhibit either inconsistent performance (EDAC/DARL) or systemic failures (PBRL/RORL) in distinguishing novel dynamics. Further investigation is provided in \autoref{sec:exp_gatekeeper}. 

We analyze the OOD detection capabilities of \ourMethod in depth in \aref{app:uncertainty}, showing the improved performance of our method across various environments and randomized parameters.

\subsection{Validation Module as a Deployment Gatekeeper}\label{sec:exp_gatekeeper}

\begin{wrapfigure}{}{0.55\linewidth}
\centering
    \includegraphics[width=\linewidth]{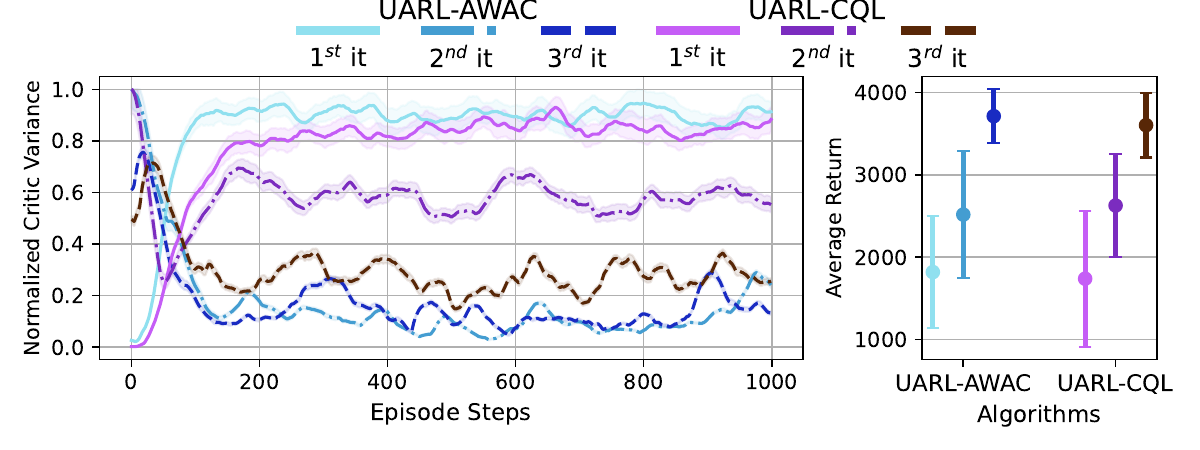}
    \caption{
    Normalized critic variance (left) and deployment return (right) over $100$ rollouts for $3$ fine-tuning phases of (\ourMethod-) AWAC and CQL under \Ant mass randomization; lighter shades = earlier phases, darker = later phases.
}
\label{fig:gatekeeper}
\end{wrapfigure}

A limitation of DR is its reliance on the assumption that policies can be safely deployed in the target domain for validation. In \autoref{sec:mainalgorithm}, \ourMethod extends DR by providing a validation module that indicates whether the policy has been exposed to sufficient randomization.

To demonstrate the validation module's role in preventing unsafe deployment under OOD conditions, we analyze the relationship between critic variance on $\dataset_t$ and the performance of policy deployed in the target environment $E_t$. We train \ourMethod with AWAC and CQL backbones across three fine-tuning phases, recording the critic variance on $\dataset_t$ after each phase, computed as the mean epistemic uncertainty across all ensemble critics; and the average return by executing the current policy in the target environment $E_t$ \textit{without} applying our validation module's OOD checks.

\autoref{fig:gatekeeper} reveals a strong negative correlation between critic variance on $\dataset_t$ and direct deployment returns in $E_t$. As fine-tuning progresses and the environment variations expand closer to $E_t$, we observe decreasing critic variance alongside improving performance, an expected pattern given the policy's gradual adaptation to target-domain characteristics. These results verify our validation module's design: critic variance on $\dataset_t$ serves as an effective proxy for policy performance in $E_t$, enabling safety-critical deployment decisions. By blocking deployment when critic variance exceeds thresholds, we prevent catastrophic policy failures.

This experiment complements our OOD detection analysis in \autoref{sec:exp_uncertainty} by quantifying the \textit{consequences} of ignoring uncertainty signals. The sharp performance cliff at the variance threshold underscores why our validation module is essential. even small increases in perceived uncertainty beyond the ID distribution correlate with drastic performance degradation in practice. Consistent correlations across AWAC and CQL validate our method.

\subsection{Zero-Shot Sim-to-Real Transfer}\label{sec:anymal}
To assess real-world applicability, we validate on ANYmal-D \citep{7758092}, a 12-DoF quadruped subject to motor-current, thermal, and payload constraints. Policies are trained in Isaac Gym \citep{rudin2022learning} with $4096$ parallel sims using proprioceptive observations (base velocities, gravity vector, command history, joint states) and a learned actuator network for torque control \citep{hwangbo2019learning}. Before hardware deployment, we enforce both performance and uncertainty criteria using only offline data. First, we collect $10$ roll-outs of a safe default (behavior) policy on the real robot (target domain dataset $\dataset_t$), compute the critic-ensemble variance offline, and apply our deployment gate as described in \autoref{sec:uarl}.

\textbf{Phase 1: Initial Evaluation.}
With the nominal added base mass range $[-1,1]\mathrm{kg}$, the target-domain variance trace (\autoref{fig:anymal_real}, left) exceeds the threshold $\tau$. Although simulation roll-outs show stable walking, this elevated variance correctly vetoes deployment, flagging hidden risks such as motor over-current, thermal throttling, or gait instabilities under real loads.

\textbf{Phase 2: Fine-Tuning and Re-evaluation.}
We expand the nominal mass interval to $[-3,3]\mathrm{kg}$, fine-tune for $10^3$ environment steps, and recompute variance. All target-domain traces now lie below the updated $\tau$ (\autoref{fig:anymal_real}, middle), certifying the policy as safe. Crucially, \ourMethod selects this tighter interval instead of the naive $[-5,5]\mathrm{kg}$ used by default in Isaac Gym, reducing parameter space, accelerating convergence, and avoiding overly conservative exploration.

We similarly expand ground friction's range until it passes the variance gate.

\textbf{Zero-Shot Hardware Transfer.}
Once both performance and uncertainty conditions are met, we deploy the policy zero-shot on ANYmal-D for $60$-second trials on three surfaces, concrete, carpet, and epoxy-resin-coated terrazzo, each with distinct friction characteristics. To validate added base mass robustness, we attach a $3\mathrm{kg}$ payload. In every trial, the robot tracks velocity commands faithfully, exhibits no motor over-currents or thermal events, and completes all runs without failure. Recordings of real-world experiments are available at \url{https://sites.google.com/view/uarl/home}.

\begin{figure}[t]
  \centering
  \includegraphics[width=\textwidth]{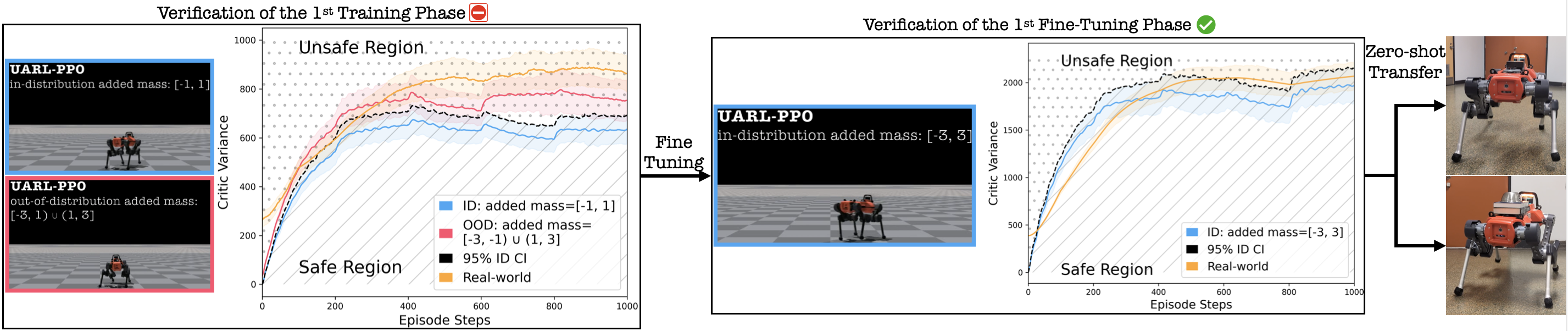}
  \caption{\small
    Critic ensemble variance for \colorbox{sbBlue025}{nominal}, \colorbox{sbRed025}{repulsive} and \colorbox{sbOrange025}{target domain} trajectories before (left) and after (middle) one fine-tuning step. The uncertainty gate (black dashed line) moves to accommodate the new nominal interval. \textbf{Right:} ANYmal-D executing the zero-shot transfer trial certified by our variance gate.}
    \vspace{-1.0em}
  \label{fig:anymal_real}
\end{figure}

Additional experiments on the quadruped robot is available at \aref{app:anymal_exps}.

\section{Conclusion \& Limitations}\label{sec:conc_limit}
\ourMethod addresses real-world RL deployment challenges through targeted, iterative adaptation in simulation. It prevents trial and error via a representative dataset from the target domain, deploying only when confident. \ourMethod improves efficiency and robustness with precise OOD detection, balanced data sampling, and gradual environment variation, without extensive randomization. Its uncertainty estimation is key for detecting OOD scenarios, particularly in robotics, where agents must navigate unexpected variations.


While \ourMethod shows clear benefits, it currently randomizes only one parameter at a time, extending to multi-parameter or adaptive schemes could improve its practicality. It also relies on a representative target domain dataset to end the curriculum; incorporating online data collection would help maintain proper confidence calibration.





\bibliography{neurips_2025}
\bibliographystyle{unsrt}

\newpage
\appendix
\doparttoc %
\faketableofcontents %

\part{Appendix}\label{appendix} %
\parttoc %

\clearpage

\section{Related Work}\label{app:related}
\textbf{Off-Dynamics RL} is a domain adaptation technique in RL that modifies the reward function to penalize transitions indicative of the source domain, encouraging the agent to learn policies transferable to the target domain \citep{eysenbach2021offdynamics}. Prior works in off-dynamics RL \citep{pmlr-v162-liu22p, lyu2024crossdomain, xu2023crossdomain, niu2022trust, eysenbach2021offdynamics} generally assume access to the target domain for policy refinement. However, this assumption introduces significant safety risks in complex, real-world systems. For instance, refining a policy for a self-driving car requires deploying potentially suboptimal policies in real-world scenarios, where failure can lead to severe consequences. 

Although these methods incorporate conservative regularization terms to constrain the learned policy within the target domain's support region, they are not designed to detect OOD shifts. Instead, they act conservatively within known domains. This distinction is critical: ``Robustness'' focuses on creating models resilient to adversaries, unusual situations, and Black Swan events, while ``Monitoring'' emphasizes detecting malicious use, monitoring predictions, and discovering unexpected model functionality \citep{hendrycks2021unsolved}. For example, in autonomous driving, while conservative regularization may limit policy deviations, it cannot guarantee safety when encountering unforeseen road conditions or new traffic laws. Thus, these methods offer partial safeguards but cannot fully mitigate risks associated with OOD scenarios in high-stakes environments.

In contrast, our work eliminates the need for policy refinement through interactions with the target domain. Instead, we employ an ensemble of critics as a proxy to evaluate whether the policy operates ``in-distribution'' within the target domain.

Our primary goal diverges fundamentally from prior work: we aim to develop reliable OOD detection capabilities rather than focusing solely on policy robustness. While existing methods enhance robustness to distribution shifts \citep{niu2022trust, xu2023crossdomain, lyu2024crossdomain}, our approach explicitly identifies novel situations requiring intervention, without direct interaction with the novel environment. This feature is essential for the safe deployment of RL systems in real-world scenarios.

Some works assume access to real-world data during training or policy refinement \citep{niu2022trust, xu2023crossdomain, lee2022offline, li2023proto}. By contrast, our method avoids this requirement for safety reasons, using real-world data solely for evaluation. This design choice enhances practicality in safety-critical applications, where real-world training data may be scarce or risky to collect. While we acknowledge the contributions of these methods to off-dynamics RL, our approach addresses limitations by prioritizing safety and data efficiency.

In an instance, target domain data is used to fill half of the replay buffer for training, assuming offline accessibility of this data \citep{niu2022trust}. However, the sensitivity of RL algorithms to data limitations makes it challenging to train performant agents with small datasets. In contrast, \ourMethod only uses target domain data for validation, demonstrating efficiency even with limited data availability.

Similarly, there are works with assumptions on limited online interactions with the target domain \citep{xu2023crossdomain, lyu2024crossdomain}. By comparison, \ourMethod requires only a pre-collected dataset, highlighting its efficiency and suitability for data-constrained settings.

\textbf{Offline RL}. Traditional RL allows policies to interact freely with environments to discover optimal strategies \citep{sutton2018reinforcement}. In contrast, offline RL addresses scenarios where online interaction is impractical or risky, learning solely from pre-collected datasets gathered by a behavior policy \citep{pmlr-v97-fujimoto19a, agarwal2020optimistic, ernst2005tree, NEURIPS2019_c2073ffa, levine2020offline, lange2012batch, kostrikov2022offline, wang2020critic, tarasov2024revisiting}. This approach enables applications in domains where real-time learning might be unsafe. However, offline RL faces a significant challenge in \textit{distributional shift}, where the training data may not accurately represent the deployment environment, potentially leading to suboptimal or unsafe behavior. Several approaches address this issue by encouraging the learned policy to resemble the behavior policy \citep{jaques2019way, wu2019behavior, Siegel2020Keep, fujimoto2021minimalist}, promoting caution through action alignment. Other methods explore training conservative critics for more cautious reward estimates \citep{NEURIPS2020_0d2b2061, pmlr-v139-kostrikov21a} or diversifying critics within the actor-critic framework to improve robustness \citep{NEURIPS2021_3d3d286a, bai2022pessimistic, pmlr-v139-wu21i}. Iterative refinement for behavior regularization in offline RL enhances policy robustness through conservative updates but requires direct interaction with the target domain \citep{ma2023iteratively}. In contrast, our method employs uncertainty-aware adaptation and simulated fine-tuning to ensure safe policy refinement without direct target domain interaction. Similarly, OOD state correction and action suppression in offline RL can be addressed using model-based regularizers, which, though effective, are computationally costly \citep{mao2024offline}. Our approach instead uses ensemble-based uncertainty estimation to detect OOD scenarios and iteratively adapt policies in simulation.

\textbf{Offline-to-Online RL}. Building upon offline RL, Offline-to-Online (O2O) RL leverages previously collected offline datasets to accelerate online RL training. This approach first pre-trains a policy with offline RL and then continues to fine-tune it with additional online interactions \citep{nair2020awac, lei2024unio, ijcai2024p615, zheng2022online}. However, naive O2O RL is often unstable due to the distributional shift between the offline dataset and online interactions. To address this, researchers have proposed techniques such as balanced sampling \citep{lee2022offline}, actor-critic alignment \citep{pmlr-v202-yu23k}, adaptive conservatism \citep{wang2024train}, return lower-bounding \citep{nakamoto2023calql}, adaptive update strategies \citep{zheng2023adaptive}, introducing online policies alongside offline ones \citep{zhang2023policy}, and using weighted replay buffers with density ratio estimation \citep{lee2022offline}. While these works primarily focus on maximizing cumulative rewards and addressing distributional shifts, our paper explicitly considers policy safety, particularly for OOD samples. Our work builds upon weighted samples but offers a more efficient solution by leveraging learned critics to assign weights to offline and online samples during fine-tuning. By focusing on high-reward regions overlooked by the behavior policy, a planning-based method for OOD exploration in O2O RL can be applied; although effective, this method requires direct target-domain interaction and thus introduces safety risks \citep{mcinroe2024ptgood}. In contrast, our approach avoids these risks by utilizing simulated environments for safe policy adaptation instead.

\textbf{Safe and Robust RL}. These methods address critical challenges in high-stakes applications like robotics \citep{garcia2015comprehensive}. While safe RL focuses on avoiding harmful actions by imposing policy constraints during training and deployment \citep{heger1994consideration}, robust RL aims to maintain performance stability under uncertainties or adversarial perturbations \citep{iyengar2005robust, nilim2005robust}. To enhance robustness, DR techniques train agents in varied, randomized environments \citep{andrychowicz2020learning, tobin2017domain, Lee2020Network, mozian2020learning}. However, this approach can be computationally expensive and challenging to tune for complex scenarios \citep{pmlr-v100-mehta20a}. Additionally, those methods assume that policy can be trained in OOD scenarios, which poses significant risks in real-world applications.

Our proposed method, \ourMethod, addresses these challenges by explicitly detecting OOD scenarios and adapting policies based on uncertainty estimation without direct interactions in OOD environments. This approach improves both safety and robustness without incurring the computational overhead associated with traditional DR techniques.

Recently, FISOR proposed a feasibility-guided diffusion model for safe offline RL, ensuring zero safety violations by identifying the largest feasible region in the offline dataset \citep{zheng2024safe}. While their method enforces hard safety constraints, it may limit policy flexibility in dynamic environments. In contrast, our approach balances safety and generalization by leveraging uncertainty quantification and iterative fine-tuning in simulated environments, enabling safe adaptation to OOD scenarios without direct interaction in the target domain. Additionally, a dynamic sampling framework can be used to enhance sample efficiency in safe RL, balancing reward maximization and safety through gradient-based optimization \citep{gu2024enhancing}. While their method improves sample efficiency, it requires direct interaction with the environment, posing safety risks. In contrast, our approach achieves sample efficiency through iterative fine-tuning in simulated environments, avoiding the need for direct interaction in the target domain.



\textbf{Curriculum learning}. 
In RL, curriculum learning aims to improve agent learning by progressively exposing the agent to increasingly complex tasks or environments \citep{narvekar2020curriculum, li2024causally}. Traditional approaches involve manually designing a sequence of tasks with increasing difficulty, where agents learn foundational skills before tackling more challenging scenarios \citep{graves2017automated}. Recent advances have explored more nuanced transfer methods, such as REvolveR \citep{pmlr-v162-liu22p}, which introduces a continuous evolutionary approach to policy transfer by interpolating between source and target robots through a sequence of intermediate configurations. While REvolveR focuses on morphological and kinematic transitions, it still shares the fundamental limitation of most curriculum learning approaches: relying on predefined task progressions that may not capture the unpredictability of real-world environments. 

In contrast, our approach focuses on uncertainty-driven adaptation, dynamically expanding the exploration space based on real-time uncertainty estimation rather than a predetermined task hierarchy. Unlike curriculum RL's structured task progression, \ourMethod enables continuous, adaptive learning that more closely mimics real-world environmental variability, particularly in scenarios with unpredictable and out-of-distribution events. Critically, our method avoids the safety risks associated with direct policy refinement in target domains by using an ensemble of critics to evaluate policy suitability without dangerous direct interactions.

\textbf{Ensemble Methods} combine multiple models to enhance performance and are widely used in ML applications \citep{goodfellow2016deep}. In OOD detection, ensembles offer robustness by leveraging model diversity \citep{Lee2020Robust, NIPS2017_9ef2ed4b}. While random initialization contributes to diversity, controlled diversity is often achieved by manipulating the loss function \citep{ijcai2020p296, mehrtens2022improving, Jain_Liu_Mueller_Gifford_2020, pang2019improving}. This allows for fine-tuning ensemble behavior, particularly through regularization techniques or bias-variance decomposition \citep{JMLR:v24:23-0041, arpit2022ensemble}, enhancing both ID accuracy and OOD detection by ensuring disagreement in uncertain regions, known as ``repulsive locations'' \citep{hafner2020noise}. In RL, ensembles play a key role in optimizing exploration strategies. For instance, ensemble critics for more efficient exploration \citep{osband2016deep} or leveraging Q-ensembles to augment Q-value estimates using the mean and standard deviation of the ensemble \citep{lee2021sunrise}. These methods also support uncertainty estimation, crucial for handling OOD scenarios in RL, by analyzing discordance between ensemble members' predictions \citep{ijcai2020p296, liu2019accurate, NIPS2017_9ef2ed4b, Jain_Liu_Mueller_Gifford_2020}. Building on these foundations, we employ critic disparity in actor-critic algorithms to effectively detect OOD instances, combining the strengths of ensemble methods in both OOD handling and RL optimization.

\clearpage

\section{Theoretical Analysis of \ourMethod}\label{app:theory}



To formalize this, we answer the following: (1) What constitutes OOD? (2) How does ensemble variance relate to OOD?

\textbf{Theory Roadmap.}  
We first fix notation and define the ensemble mean and variance (\autoref{eq:ensemble-mean}-\autoref{eq:ensemble-var}) and our notion of in- vs. out-of-distribution (\defref{def:ood}).  
In \aref{app:theory_effects} we prove two key facts in the tabular setting: 
\propref{proposition:critic} (a lower bound on $\|Q^\pi_{\domainParam_t} - Q^\pi_{\domainParam}\|_\infty$) and \propref{proposition:policy} (suboptimality of the nominal policy under OOD).  \aref{app:theory_variance} then shows that ensemble variance exactly quantifies the squared prediction error (\autoref{eq:distance}), certifying OOD.  
Finally, \assumref{assump:continuous-mdp} extends everything to continuous MDPs, \lemref{lem:bellman-perturb}, \autoref{thm:value-error}, \propref{prop:cov-bias-reduction}, and \autoref{thm:approx-convergence}, with full proofs provided.

\subsection{Definitions and Notations}

\begin{definition}[Mean, variance, and second moment of the ensemble]
The mean of the ensemble over $N$ independently initialized critics is
\begin{equation}
     \mu = \expectation_{\text{ens}} \left[ Q^{\pi}(s,a) \right]
        = \frac{1}{N} \sum_{i=1}^N Q_i^\pi(s, a)
\label{eq:ensemble-mean}
\end{equation}
and the variance is 
\begin{equation}   
\sigma^2 = \frac{1}{N} \sum_{i=1}^N \left( Q_i^\pi(s,a) - \mu \right)^2 = \expectation_{\text{ens}} \left[ \left( Q^{\pi}(s,a) \right)^2 \right] - \mu^2 
\label{eq:ensemble-var}
\end{equation}
This variance measures the disagreement among the ensemble of critics. A higher variance indicates greater uncertainty, while lower variance suggests model agreement. Here we treat variance purely as an epistemic measure (model-disagreement) as it does not by itself distinguish between irreducible aleatoric noise and true distributional shift \citep{charpentier2022disentangling}.
From \autoref{eq:ensemble-var}, the second moment is
\begin{equation}
    \expectation_{\text{ens}} \left[ \left(  Q^{\pi}(s,a) \right)^2 \right] = \sigma^2 + \mu^2 
     \label{eq:ensemble-squared}
\end{equation}

\end{definition}



\subsection{Effect of Out-of-Distribution Scenarios}\label{app:theory_effects}



\begin{proof}\label{proof:critic}
Proving \propref{proposition:critic}, by \lemref{lem:bellman-perturb},
\begin{equation}
  \|\mathcal{T}^\pi_{\domainParam}(Q^\pi_{\domainParam_t})
    - \mathcal{T}^\pi_{\domainParam_t}(Q^\pi_{\domainParam_t})\|_\infty
  \le\gamma(L_R + L_T\|Q^\pi_{\domainParam_t}\|_{\mathrm{Lip}})\|\domainParam-\domainParam_t\|.
\end{equation}
Meanwhile, since both operators are $\gamma$-contractions,
\begin{equation}
  \|Q^\pi_{\domainParam} - Q^\pi_{\domainParam_t}\|_\infty
  = \|\mathcal{T}^\pi_{\domainParam}(Q^\pi_{\domainParam})
    - \mathcal{T}^\pi_{\domainParam_t}(Q^\pi_{\domainParam_t})\|_\infty
\end{equation}
Applying the triangle inequality and then rearranging gives
\begin{equation}
  (1+\gamma)\|Q^\pi_{\domainParam}-Q^\pi_{\domainParam_t}\|_\infty
  \ge\gamma(L_R + L_T\|Q^\pi_{\domainParam_t}\|_{\mathrm{Lip}})\|\domainParam_t-\domainParam\|
\end{equation}
which is strictly positive whenever $\domainParam\neq\domainParam_t$.  Hence the two fixed points differ.
\end{proof}
Intuitively, any non-zero shift in dynamics $\|\domainParam_t-\domainParam\|$ forces at least a proportional change in the critic's fixed point, guaranteeing detectability via variance.

\begin{proposition}
A policy $\pi(s)=\arg\max_aQ^\pi(s,a)$ trained in $E$ is suboptimal in OOD $E_t$.
\label{proposition:policy}
\end{proposition}

\begin{proof}
By \propref{proposition:critic}, $\exists s$ such that $Q^t(s,\pi(s))<\max_{a'}Q^t(s,a')$, so $\pi$ is suboptimal in $E_t$.
\end{proof}
Since the Q-functions differ, the policy that was optimal under $\domainParam$ must mis-rank actions under $\domainParam_t$, hence is suboptimal.

\subsection{Variance of Ensemble as an OOD Indicator}\label{app:theory_variance}

\propref{proposition:critic} shows that varying environments lead to distinct optimal Q-functions. 
While the true value of $Q^t$ is unknown, the expected squared distance between the ensemble mean prediction and the optimal value is 
\begin{equation}
    \begin{aligned}
        \expectation_{\text{ens}} \left[ \left( Q^t(s,a) - Q^{\pi}(s,a) \right)^2 \right] = \left(Q^t(s,a)\right)^2 - 2 Q^t(s,a) \expectation_{\text{ens}} \left[ Q^{\pi}(s,a) \right] + \expectation_{\text{ens}}\left[ \left( Q^{\pi}(s,a) \right)^2 \right]
    \end{aligned}
\end{equation}

Since $Q^t(s,a)$ is the optimal Q function for $E_t$ and is deterministic, it is constant with respect to ensemble randomness. Therefore, the above equation can be simplified as 
\begin{equation}
\begin{aligned}
\expectation_{\text{ens}} \left[ \left( Q^t(s,a) - Q^{\pi}(s,a) \right)^2 \right] = \left( Q^t(s,a) \right)^2 - 2 Q^t(s,a) \mu + \expectation_{\text{ens}}\left[ \left( Q^{\pi}(s,a) \right)^2 \right] 
    \end{aligned}
\end{equation}

Substituting the definitions of $\mu$ and the second moment from \autoref{eq:ensemble-mean} and \autoref{eq:ensemble-squared}, yields 
\begin{equation}
\begin{aligned}
    \expectation_{\text{ens}} \left[ \left( Q^t(s,a) - Q^{\pi}(s,a) \right)^2 \right] = \sigma^2 + \mu^2 - 2Q^t(s,a) \mu + (Q^t(s,a))^2 = \sigma^2 + (\mu - Q^t(s,a))^2
    \end{aligned}
    \label{eq:distance}
\end{equation}
\autoref{eq:distance} shows variance is exactly the part of prediction error \textit{due to} ensemble disagreement, so high $\sigma^2$ is high OOD error. From \autoref{eq:distance}, we can see that the expected squared error between the ensemble prediction and the true Q value depends on the mean $\mu$ and variance $\sigma^2$ of the ensemble plus the true Q value $Q^t(s,a)$.

\begin{proposition}
Assuming successful optimization of \autoref{eq:our_loss}, sufficiently high ensemble variance $\sigma^2$ indicates $E_t \notin E$ (OOD).
\end{proposition}

\begin{proof}
From \autoref{eq:distance}, we can see that the variance of the critics directly contributes to the total expected squared distance. 
Our training objective in \autoref{eq:our_loss} specifically encourages agreement among Q-functions for ID samples from $E$ and diversity for OOD samples from $E'$. 

For ID environments where $E_t \in E$, the ensemble variance will be small due to the training objective. Conversely, when the environment is significantly different from the training distribution ($E_t \notin E$), the ensemble members will not have converged to a common prediction, resulting in high variance. Therefore, when we observe sufficiently high variance in ensemble predictions, it strongly indicates that $E_t \notin E$.
\end{proof}

\begin{proposition}
Under \autoref{eq:our_loss} and assuming sufficient domain randomization coverage, low ensemble variance $\sigma^2$ for samples from $E_t$ is necessary but not sufficient to conclude $E_t \in E$ (ID) without full coverage of the domain parameter space.
\end{proposition}

\begin{proof}
Our ensemble begins with random initialization.
Assuming the objective function (\autoref{eq:our_loss}) is optimized, the learned Q-functions converge to similar solutions for ID data:
\begin{equation}
    Q_i(s,a) \approx \mu(s,a), \forall i, (s, a) \sim D
\end{equation}
From the definition of variance in \autoref{eq:ensemble-var}, this implies small ensemble variance for ID data.

Given target-domain data $D_t$ from environment $E_t$, if the ensemble exhibits low variance for $(s,a) \sim D_t$, this suggests the ensemble members agree on these samples. Since our training specifically promotes disagreement for OOD samples, consistent agreement indicates that $E_t$ exhibits dynamics similar to those encountered during training. Specifically, if the domain randomization process successfully covered the target domain parameter $\domainParam_t \in \domainParamSet$, the low variance suggests $E_t \in E$.

However, low variance alone cannot guarantee that $E_t \in E$ without full coverage of the domain parameter space, as there could exist regions where the ensemble coincidentally agrees despite lacking training exposure.
\end{proof}

\subsection{Continuous MDPs with Function Approximation}
\textbf{Why function-approximation theory?}
The preceding propositions were stated for tabular MDPs.  Real
robotics tasks have continuous states, continuous actions, and neural
critics.  In this subsection we show that our variance-based OOD
certificate \emph{still holds} under standard smoothness and
approximation assumptions, and that the replay-buffer weighting
reduces the sim-to-real bias by at least a constant factor.

Thus, to extend our variance-based OOD detection framework to realistic settings, we analyze continuous-state, continuous-action MDPs under function approximation.  We assume the following smoothness and approximation conditions, which allow us to bound how perturbations in the domain parameters translate into errors in value estimation and policy performance.

Even with neural critics in continuous domains, high ensemble variance lower-bounds the true value-gap,
justifying the deployment gatekeeper used by \ourMethod.

\begin{assumption}[Smooth Continuous MDP]\label{assump:continuous-mdp}
Let $\mathcal{M}_\domainParam=(S,A,T_\domainParam,R,\gamma)$ be an MDP with $S\subset\R^n$ and $A\subset\R^m$ compact.  Let $F$ be a class of real-valued functions on $S\times A$.  There exist constants $L_R,R_{\max},L_T,\epsilon_F\ge0$ such that:
\begin{enumerate}
  \item The reward $R(s,a)$ is $L_R$\textit{–Lipschitz} in $(s,a)$, i.e.
    \begin{equation}
      |R(s,a) - R(s',a')|
      \le L_R\bigl\|(s,a)-(s',a')\bigr\|,
      \quad
      \forall (s,a),(s',a')\in S\times A,
    \end{equation}
    and satisfies $\lvert R(s,a)\rvert\le R_{\max}.$
  \item For any two domain parameters $\domainParam_1,\domainParam_2\in\domainParamSet$,
    \begin{equation}
      W_1\bigl(T_{\domainParam_1}(\cdot\mid s,a),\,T_{\domainParam_2}(\cdot\mid s,a)\bigr)
      \le L_T\|\domainParam_1-\domainParam_2\|,
      \quad
      \forall (s,a)\in S\times A,
    \end{equation}
    where $W_1$ is the 1-Wasserstein distance. We acknowledge that real-world contact dynamics can violate global Lipschitz continuity. In such cases one must either restrict to local Lipschitz regions or accept that our bounds hold only piecewise.
  \item The critic class $F$ can uniformly approximate the true Q-function up to error $\epsilon_F$.  Defining the projection operator
  \begin{equation}
    \Pi_F[f] = \arg\min_{g\in F}\|f - g\|_\infty,
  \end{equation}
  we require
  \begin{equation}
    \|Q^\pi_\domainParam - \Pi_F[Q^\pi_\domainParam]\|_\infty \le \epsilon_F,
  \end{equation}
  meaning there exists some $\tilde Q\in F$ with $\|Q^\pi_\domainParam - \tilde Q\|_\infty\le\epsilon_F$.  
\end{enumerate}
\end{assumption}

\assumref{assump:continuous-mdp} ensures both the Bellman operator's smooth dependence on $\domainParam$ and that our function class $F$ can approximate the true value functions to within $\epsilon_F$.

\subsubsection{Bellman Operator Perturbation}

To connect our ensemble-variance OOD criterion with value-function robustness, we first quantify how small changes in the domain parameters $\domainParam$ perturb the Bellman operator.  This will let us bound the error in the learned Q-functions when moving from one domain to another.

\begin{lemma}[Operator Perturbation]\label{lem:bellman-perturb}
Under \assumref{assump:continuous-mdp}, for any $Q\in F$ and any two domain parameters $\domainParam_1,\domainParam_2$, with $\mathcal{T}^\pi$ being policy dependent Bellman operator,
\begin{equation}
  \bigl\|\mathcal{T}^\pi_{\domainParam_1}Q - \mathcal{T}^\pi_{\domainParam_2}Q\bigr\|_\infty
  \le\gamma\bigl(L_R + L_T\|Q\|_{\mathrm{Lip}}\bigr)\|\domainParam_1-\domainParam_2\|.
\end{equation}
where:
\begin{equation}
    \|Q\|_{\mathrm{Lip}} := \sup_{(s,a)\neq(s',a')}\frac{|Q(s,a)-Q(s',a')|}{\|(s,a)-(s',a')\|}
\end{equation}
\end{lemma}

\begin{proof} 
For each $(s,a)$, apply
\begin{align}
  \bigl|\mathcal{T}^\pi_{\domainParam_1}Q(s,a)-\mathcal{T}^\pi_{\domainParam_2}Q(s,a)\bigr|
  \le |R(s,a)-R(s,a)| 
  + \gamma\bigl|\E_{s'\sim T_{\domainParam_1}}[Q(s',\pi(s'))]
    - \E_{s'\sim T_{\domainParam_2}}[Q(s',\pi(s'))]\bigr|.
\end{align}
The first term vanishes.  For the second, use the 1-Wasserstein bound on $T_\domainParam$ and the Lipschitz seminorm of $Q$, then take a supremum.  
\end{proof}
This lemma quantifies how smoothly the Bellman operator moves under small MDP perturbations, a key ingredient for all subsequent bounds.

\subsubsection{Simulation-to-Real Value Error}

Next, we bound the difference between the true value function in the target domain $\domainParam_t$ and our learned Q-function in the nominal domain $\domainParam$.  By Banach's fixed-point theorem:

\begin{theorem}[Value Function Error]\label{thm:value-error}
Under \assumref{assump:continuous-mdp}, let $Q^\pi_\domainParam$ and $Q^\pi_{\domainParam_t}$ be the unique fixed points of $\mathcal{T}^\pi_{\domainParam}$ and $\mathcal{T}^\pi_{\domainParam_t}$.  Then
\begin{equation}
  \bigl\|Q^\pi_\domainParam - Q^\pi_{\domainParam_t}\bigr\|_\infty
  \le\frac{\gamma}{1-\gamma}\bigl(L_R + L_T\|Q^\pi_{\domainParam_t}\|_{\mathrm{Lip}}\bigr)\|\domainParam-\domainParam_t\|.
\end{equation}
\end{theorem}

\textit{Proof Sketch.}  Since each Bellman operator is a $\gamma$-contraction in $\|\cdot\|_\infty$, \lemref{lem:bellman-perturb} implies
\begin{align}
  \|Q^\pi_\domainParam - Q^\pi_{\domainParam_t}\|_\infty
  =\bigl\|\mathcal{T}^\pi_\domainParam Q^\pi_\domainParam - \mathcal{T}^\pi_{\domainParam_t} Q^\pi_{\domainParam_t}\bigr\|_\infty
  \le\gamma\|Q^\pi_\domainParam - Q^\pi_{\domainParam_t}\|_\infty
    + \gamma\bigl(L_R + L_T\|Q^\pi_{\domainParam_t}\|_{\mathrm{Lip}}\bigr)\|\domainParam-\domainParam_t\|,
\end{align}
and rearrange.

We now have a worst-case, linear-in-$\|\domainParam-\domainParam_t\|$ bound on the value-function gap in any continuous MDP.

\subsubsection{Bias Reduction via Variance-Weighted Bellman Operator}

To pull our learned operator closer to the true target operator, we introduce uncertainty-weighted sampling.  Intuitively, high-variance (OOD) samples should receive larger weights in the Bellman update, so that our operator focuses more on states likely to match the target domain.

\begin{assumption}[Variance–Distance Correlation]\label{assump:covariance}
There exists $\rho>0$ such that for any subset $S\subset\dataset\cup\dataset'$,
\begin{equation}
  \Cov_{(s,a)\sim S}\!\bigl(\sigma^2(s,a),\|\domainParam_{s,a}-\domainParam_t\|\bigr)
  \ge\rho.
\end{equation}
\end{assumption}
Each tuple $(s,a,\domainParam_{s,a})\in\dataset\cup\dataset'$ carries its generating domain parameter $\domainParam_{s,a}\in\domainParamSet$.

Let
\begin{equation}
  w(s,a) =
  \begin{cases}
    \sigma^2(s,a), & (s,a)\in\dataset,\\
    \sigma^{-2}(s,a), & (s,a)\in\dataset',
  \end{cases}
  \quad
  \bar w = \tfrac{1}{|\dataset\cup\dataset'|}\sum_{(s,a)} w(s,a).
\end{equation}
Note that \assumref{assump:covariance} can fail if all critics jointly converge to the same biased estimate (e.g. at the mean of the randomization range), so robustness to slight negative covariance should be analyzed in future extensions.

\begin{proposition}[Bias Reduction]\label{prop:cov-bias-reduction}
Under \assumref{assump:continuous-mdp} and \assumref{assump:covariance}, define the weighted Bellman operator
\begin{equation}
  (\mathcal{T}_w Q)(s,a)
  = \frac{1}{\sum_{(s,a)}w(s,a)}
    \sum_{(s,a)\in\dataset\cup\dataset'}
    w(s,a)\Bigl[R(s,a)
    + \gamma\E_{s'\sim T_{\domainParam_{s,a}}}\bigl[Q(s',\pi(s'))\bigr]\Bigr].
\end{equation}
Then
\begin{align}
  \|\mathcal{T}_w Q - \mathcal{T}^\pi_{\domainParam_t}Q\|_\infty
  &\le \gamma(L_R + L_T\|Q\|_{\mathrm{Lip}})
    \E_w\bigl[\|\domainParam_{s,a}-\domainParam_t\|\bigr],\\
  \E_w\bigl[\|\domainParam_{s,a}-\domainParam_t\|\bigr]
  &\le \E_{\mathrm{unif}}\bigl[\|\domainParam_{s,a}-\domainParam_t\|\bigr] - \frac{\rho}{\bar w},
\end{align}
so the weighted update strictly reduces the average parameter-distance by at least $\rho/\bar w$.
\end{proposition}

\textit{Proof Sketch.}  The first inequality follows by applying \lemref{lem:bellman-perturb} inside the weighted sum; the second uses \assumref{assump:covariance} and the definition of $\bar w$.

Weighting updates by critic variance actively \textit{shrinks} the average dynamics gap by at least $\rho/\bar w$, reducing sim-to-real bias.

\subsubsection{Convergence of Weighted Fitted-Q Iteration}

Finally, we combine the above bounds with standard finite-sample analysis to guarantee that our weighted fitted-Q iteration recovers a near-optimal policy in the target domain.

\begin{theorem}[Approximate Convergence]\label{thm:approx-convergence}
Under \assumref{assump:continuous-mdp} and \assumref{assump:covariance}, suppose we run $K$ iterations of weighted fitted-Q with $N$ regression samples per iteration.  Then for any $\eta>0$ and confidence $1-\delta$,
\begin{equation}
  V^{\pi_K}_{\domainParam_t}(s)
  \ge\max_\pi V^\pi_{\domainParam_t}(s)
    - \eta - \mathcal{O}(\epsilon_F),
  \quad\forall s\in S,
\end{equation}
provided
\begin{equation}
  N \ge\frac{C_1}{\eta^2}\log\frac{1}{\delta},
  \qquad
  K \ge\mathcal{O}\!\Bigl(\frac{1}{(1-\gamma)^2\eta}\Bigr).
\end{equation}
We note that providing at least order-of-magnitude estimates for $C_1$ and its explicit dependence on the state and action dimensions would greatly aid in assessing practical sample requirements; we defer this detailed characterization to future work.
\end{theorem}

This final theorem ties it all together: with polynomial samples and iterations, weighted fitted Q-iteration recovers an $\eta$-optimal policy in the new domain.

The proof follows standard fitted Q-iteration analysis (e.g. \cite{levine2020offline}) with the additional bias reduction from \propref{prop:cov-bias-reduction}. Deriving explicit expressions for the constant $C_1$ in terms of $L_R,R_{\max},L_T,\epsilon_F$ is left to future work.

\subsubsection{Implementation Notes}
\textbf{Practical Hyperparameter Selection}.
In practice, one first performs small pilot studies on held-out ID and OOD environments to estimate the empirical variances $\sigma_{\mathrm{in}}^2$ and $\sigma_{\mathrm{out}}^2$.  To guarantee a maximum false-alarm rate $\alpha$ on ID data, the ensemble size $N$ is chosen to grow on the order of $\log(1/\alpha)$, and the threshold is set to
\begin{equation}
  \tau = \sigma_{\mathrm{in}}^2 + z_{1-\alpha}\frac{\sigma_{\mathrm{in}}}{\sqrt{N}},
\end{equation}
where $z_{1-\alpha}$ is the $(1-\alpha)$-quantile of the standard normal.  With this rule, ID variances fall below $\tau$ with probability at least $1-\alpha$, while OOD variances typically exceed $\tau$, all using a modest ensemble size.

\subsection{Theory Summary}

In essence, our analysis demonstrates that the variance across an ensemble of Q-functions provides a rigorous, quantitative signal for detecting when a new environment lies outside the training distribution.  \lemref{lem:bellman-perturb} shows that small changes in the MDP parameters $\domainParam$ induce proportionally bounded perturbations in the Bellman operator, and \autoref{thm:value-error} propagates these to a linear bound on the resulting value-function error.  Because our training objective explicitly drives low variance on ID samples and permits high variance on repulsive (OOD) ones, observing large ensemble disagreement directly corresponds to large value-error under the true dynamics, thus certifying OOD.

\subsection{Algorithmic Implications}

Building on this diagnostic, we define a variance-weighted Bellman operator (\propref{prop:cov-bias-reduction}) that automatically emphasizes samples whose critic variance correlates with proximity to the target domain.  Under a mild covariance assumption, this weighting provably reduces the average ``distance'' between sampled and true domain parameters, and when embedded within fitted-Q iteration yields an $\eta$-optimal policy in the target environment with only polynomial dependence on the accuracy and confidence parameters (\autoref{thm:approx-convergence}).  Together, these results justify both our use of ensemble variance as an OOD detector and our weighting scheme as a robust means of bias reduction in sim-to-real and domain-randomization scenarios.

\clearpage

\section{Limitations}\label{app:limits}

\begin{table}[t]
\centering
\caption{Computational Overhead of CQL vs. CQL+\ourMethod, on a single Nvidia 4090 GPU.}
\begin{tabular}{lcc}
\hline
\textbf{Metric} & \textbf{Baseline} & \textbf{\ourMethod} \\
\hline
Memory Usage & $\sim$2 GB & $\sim$4 GB \\
Memory Increase & \multicolumn{2}{c}{50\%} \\
\hline
Training Time per Iteration & $\sim$0.5 seconds & $\sim$0.55 seconds \\
Computational Time Increase & \multicolumn{2}{c}{10\%} \\
\hline
Diversity Loss Calculation & Not Applicable & Required \\
\hline
\end{tabular}
\label{tab:computational_overhead}
\end{table}

While \ourMethod demonstrates promising results in enhancing safety in RL, several key limitations warrant discussion.

\autoref{tab:computational_overhead} highlights the computational overhead associated with \ourMethod. The approach introduces moderate resource demands, with a $10\%$ increase in training time and $50\%$ higher memory usage compared to baseline methods which depends on the size of the repulsive dataset. However, these costs are offset by the potential for improved policy safety and robustness, particularly in detecting and adapting to OOD scenarios.

A primary limitation of \ourMethod is the sequential randomization of a single hyperparameter during iterative adaptation. While this controlled exploration aids stability, it may fail to capture the complex interplay between multiple environmental parameters. Future work could explore simultaneous multi-parameter randomization to better simulate real-world uncertainties.

The method also relies on manually defined parameter ranges for randomization, determined using domain expertise. This reliance may limit generalizability across diverse tasks and environments. Developing an automated mechanism to adaptively determine parameter ranges could enhance scalability and reduce human intervention.

Another challenge is the dependence on a proxy dataset $\dataset_t$ from the target environment, which critically influences uncertainty estimation and policy refinement. If this dataset is incomplete or unrepresentative, it can lead to suboptimal adaptations or overconfident policies. Techniques to ensure dataset quality and representativeness will be crucial for robust performance.

Finally, while \ourMethod has demonstrated effectiveness on MuJoCo benchmark tasks, its applicability to more complex, high-dimensional, real-world scenarios remains an open question. Extensive validation across diverse robotic and control domains will be necessary to establish its broader relevance and effectiveness.

\clearpage

\section{Experiment Setup}\label{app:exp_setup}

\subsection{Hyperparameters and Network Architectures}\label{app:hyperparams}
As mentioned in \autoref{sec:exp}, our implementations of baselines and \ourMethod are based on Clean Offline Reinforcement Learning (CORL)\footnote{\href{https://github.com/corl-team/CORL}{github.com/corl-team/CORL}} \citep{tarasov2022corl}. CORL is an Offline RL library that offers concise, high-quality single-file implementations of state-of-the-art algorithms. The results produced using CORL can serve as a benchmark for D4RL tasks, eliminating the need to re-implement or fine-tune existing algorithm hyperparameters. Thus, without tuning any hyperparameter, we use the already provided ones for our experiments, for either baselines or \ourMethod. Following, we present the hyperparameters used in our experiments and the network architectures for baselines.

\clearpage

\subsubsection{AWAC}
\begin{table}[ht]
\caption{AWAC Hyperparameters.}\label{table:awac_params}
\centering
\begin{tabular}{cll}
\toprule
& Hyperparameter & Value \\
\midrule
\multirow{3}{*}{\shortstack{AWAC\\\citep{nair2020awac}}} & Scaling of the advantage estimates & $0.33$ \\
& \makecell[l]{Upper limit on the exponentiated\\advantage weights} & $100$ \\
\midrule
\multirow{5}{*}{Common}
& Discount factor~$\y$      & $0.99$ \\
& Replay buffer capacity    & $2$M \\
& Mini-batch size           & $256$ \\
& Target update rate~$\tau$ & $5 \times 10^{-3}$ \\
& Policy update frequency   & Every $2$ updates \\
\midrule
\multirow{2}{*}{Optimizer} 
& (Shared) Optimizer        & Adam~\citep{DBLP:journals/corr/KingmaB14} \\
& (Shared) Learning rate    & $3 \times 10^{-4}$ \\
\bottomrule
\end{tabular}
\end{table} 

\begin{tcolorbox}[title={\begin{pseudo}\label{pseudo:AWAC}AWAC Network Details\end{pseudo}}]
\textbf{Critic $Q$ Networks:}

$\triangleright$ AWAC uses 2 critic networks with the same architecture and forward pass.
\begin{verbatim}
l1 = Linear(state_dim + action_dim, 256)
l2 = Linear(256, 256)
l3 = Linear(256, 256)
l4 = Linear(256, 1)
\end{verbatim}

\textbf{Critic $Q$ Forward Pass:}
\begin{verbatim}
input = concatenate([state, action])
x = ReLU(l1(input))
x = ReLU(l2(x))
x = ReLU(l3(x))
value = l4(x)
\end{verbatim}

\vspace{-8pt}
\hrulefill

\textbf{Policy $\pi$ Network (Actor):}
\begin{verbatim}
l1 = Linear(state_dim, 256)
l2 = Linear(256, 256)
l3 = Linear(256, 256)
l4 = Linear(256, action_dim)
\end{verbatim}

\textbf{Policy $\pi$ Forward Pass:}
\begin{verbatim}
x = ReLU(l1(state))
x = ReLU(l2(x))
x = ReLU(l3(x))
mean = l4(x)
log_std = self._log_std.clip(-20, 2)
action_dist = Normal(mean, exp(log_std))
action = action_dist.rsample().clamp(min_action, max_action)
\end{verbatim}
\end{tcolorbox}

\clearpage

\subsubsection{CQL}
\begin{table}[ht]
\caption{CQL Hyperparameters.}\label{table:cql_params}
\centering
\begin{tabular}{cll}
\toprule
& Hyperparameter & Value \\
\midrule
\multirow{3}{*}{\shortstack{CQL\\\citep{NEURIPS2020_0d2b2061}}} & Scaling the CQL penalty & $1$ \\
& Target Action Gap & $-1$ \\
& Temperature Parameter & $1$ \\
\midrule
\multirow{5}{*}{Common}
& Discount factor~$\gamma$      & $0.99$ \\
& Replay buffer capacity    & $2$M \\
& Mini-batch size           & $256$ \\
& Target update rate~$\tau$ & $5 \times 10^{-3}$ \\
& Policy update frequency   & Every $2$ updates \\
\midrule
\multirow{3}{*}{Optimizer} 
& (Shared) Optimizer        & Adam~\citep{DBLP:journals/corr/KingmaB14} \\
& (Shared) Policy learning rate    & $3 \times 10^{-5}$ \\
& (Shared) Critic learning rate    & $3 \times 10^{-4}$ \\
\bottomrule
\end{tabular}
\end{table} 

\begin{tcolorbox}[title={\begin{pseudo}\label{pseudo:CQL}CQL Network Details\end{pseudo}}]
\textbf{Critic $Q$ Networks:}

$\triangleright$ CQL uses 2 critic networks with the same architecture and forward pass.
\begin{verbatim}
l1 = Linear(state_dim + action_dim, 256)
l2 = Linear(256, 256)                
l3 = Linear(256, 1)                              
\end{verbatim}

\textbf{Critic $Q$ Forward Pass:}
\begin{verbatim}
input = concatenate([state, action]) 
x = ReLU(l1(input))                  
x = ReLU(l2(x))                                         
value = l3(x)                        
\end{verbatim}

\vspace{-8pt}
\hrulefill

\textbf{Policy $\pi$ Network (Actor):}
\begin{verbatim}
l1 = Linear(state_dim, 256)          
l2 = Linear(256, 256)                
l3 = Linear(256, 256)                
l4 = Linear(256, 2 * action_dim)     
\end{verbatim}

\textbf{Policy $\pi$ Forward Pass:}
\begin{verbatim}
x = ReLU(l1(state))                 
x = ReLU(l2(x))                     
x = ReLU(l3(x))                     
x = l4(x)                        
mean, log_std = torch.split(x, action_dim, dim=-1)
normal = Normal(mean, std)
action_dist = TransformedDistribution(normal, TanhTransform(cache_size=1))
action = action_dist.rsample()
\end{verbatim}
\end{tcolorbox}

\clearpage

\subsubsection{TD3BC}
\begin{table}[ht]
\caption{TD3BC Hyperparameters.}\label{table:td3bc_params}
\centering
\begin{tabular}{cll}
\toprule
& Hyperparameter & Value \\
\midrule
\multirow{4}{*}{\shortstack{TD3BC\\\citep{fujimoto2021minimalist}}} & Scaling factor ($\alpha$) & $2.5$ \\
& Noise added to the policy's action & $0.2$ \\
& \makecell[l]{Maximum magnitude of noise\\added to actions} & $0.5$ \\
\midrule
\multirow{5}{*}{Common}
& Discount factor~$\gamma$      & $0.99$ \\
& Replay buffer capacity    & $2$M \\
& Mini-batch size           & $256$ \\
& Target update rate~$\tau$ & $5 \times 10^{-3}$ \\
& Policy update frequency   & Every $2$ updates \\
\midrule
\multirow{2}{*}{Optimizer} 
& (Shared) Optimizer        & Adam~\citep{DBLP:journals/corr/KingmaB14} \\
& (Shared) Learning rate    & $3 \times 10^{-4}$ \\
\bottomrule
\end{tabular}
\end{table}

\begin{tcolorbox}[title={\begin{pseudo}\label{pseudo:TD3BC}TD3BC Network Details\end{pseudo}}]
\textbf{Critic $Q$ Networks:}

$\triangleright$ TD3BC uses 2 critic networks with the same architecture and forward pass.
\begin{verbatim}
l1 = Linear(state_dim + action_dim, 256)
l2 = Linear(256, 256)
l3 = Linear(256, 1)                            
\end{verbatim}

\textbf{Critic $Q$ Forward Pass:}
\begin{verbatim}
input = concatenate([state, action])
x = ReLU(l1(input))
x = ReLU(l2(x))
value = l3(x)                     
\end{verbatim}

\vspace{-8pt}
\hrulefill

\textbf{Policy $\pi$ Network (Actor):}
\begin{verbatim}
l1 = Linear(state_dim, 256)
l2 = Linear(256, 256)
l3 = Linear(256, action_dim) 
\end{verbatim}

\textbf{Policy $\pi$ Forward Pass:}
\begin{verbatim}
x = ReLU(l1(state))
x = ReLU(l2(x))
x = Tanh(l3(x))
action = max_action * x
\end{verbatim}
\end{tcolorbox}

\clearpage

\subsubsection{EDAC}
\begin{table}[ht]
\caption{EDAC Hyperparameters.}\label{table:edac_params}
\centering
\begin{tabular}{cll}
\toprule
& Hyperparameter & Value \\
\midrule
\multirow{2}{*}{\shortstack{EDAC\\\citep{NEURIPS2021_3d3d286a}}} 
& Diversity coefficient $\eta$ & $1.0$ \\
& Target entropy & $-\text{action\_dim}$ \\
\midrule
\multirow{5}{*}{Common}
& Discount factor~$\gamma$      & $0.99$ \\
& Replay buffer capacity    & $2$M \\
& Mini-batch size           & $256$ \\
& Target update rate~$\tau$ & $5 \times 10^{-3}$ \\
& Policy update frequency   & Every $2$ updates \\
\midrule
\multirow{2}{*}{Optimizer} 
& (Shared) Optimizer        & Adam~\citep{DBLP:journals/corr/KingmaB14} \\
& (Shared) Learning rate    & $3 \times 10^{-4}$ \\
\bottomrule
\end{tabular}
\end{table} 

\begin{tcolorbox}[title={\begin{pseudo}\label{pseudo:EDAC}EDAC Network Details\end{pseudo}}]
\textbf{Critic $Q$ Networks:}

$\triangleright$ EDAC uses 10 critic networks with the same architecture.
\begin{verbatim}
l1 = Linear(state_dim + action_dim, 256)
l2 = Linear(256, 256)
l3 = Linear(256, 256)     
l4 = Linear(256, 1)     
\end{verbatim}

\textbf{Critic $Q$ Forward Pass:}
\begin{verbatim}
input = concatenate([state, action])
x = ReLU(l1(input))                  
x = ReLU(l2(x))  
x = ReLU(l3(x))  
value = l4(x)    
\end{verbatim}
\vspace{-8pt}
\hrulefill

\textbf{Policy $\pi$ Network (Actor):}
\begin{verbatim}
l1 = Linear(state_dim, 256)
l2 = Linear(256, 256)
l3 = Linear(256, 256) 
mu = Linear(256, action_dim)
log_sigma = Linear(256, action_dim)
\end{verbatim}

\textbf{Policy $\pi$ Forward Pass:}
\begin{verbatim}
x = ReLU(l1(state))
x = ReLU(l2(x))  
hidden = l3(x)  
mu, log_sigma = mu(hidden), log_sigma(hidden)
log_sigma = clip(log_sigma, -5, 2)
policy_dist = Normal(mu, exp(log_sigma))
action = policy_dist.sample()
\end{verbatim}
\end{tcolorbox}

\clearpage

\subsection{Randomized Hyperparameter Scales}\label{app:random_params}
Due to the variety of dynamics and physics of the agent in each environment, the range of randomized hyperparameters should be different. For instance, when initializing the \Ant environment, the initial noise scale is $1 \times 10^{-1}$, while it is $5 \times 10^{-3}$ for \Hopper. Because of this, we consider various ranges to scale the randomized hyperparameters. The values provided in \autoref{table:rand_param_scales} show the scales of the hyperparameters during each iteration of our algorithm.

The selection of these hyperparameter ranges is based on careful consideration of each environment's characteristics and the MuJoCo physics engine's properties \citep{todorov2012mujoco}:
\begin{itemize}
    \item Initial Noise Scale: This hyperparameter affects the initial state variability. For more stable agents like \Ant and \HalfCheetah, we start with a smaller scale ($1 \times 10^{-5}$) and gradually increase it to the default value ($1 \times 10^{-1}$). For less stable agents like \Hopper and \Walker, we begin with an even smaller scale ($5 \times 10^{-7}$) to ensure initial stability.
    \item Friction Coefficient: In our experiments, we specifically modify the friction coefficient between the agent and the ground. MuJoCo utilizes a pyramidal friction cone approximation, where this coefficient directly affects contact dynamics and determines how the agent interacts with its environment. We maintain the default friction values initially, then gradually increase them to challenge the agent's locomotion and stability. For \Ant, we make moderate increases due to its quadrupedal locomotion's reliance on ground contact. For \HalfCheetah, where smooth forward motion is key, smaller increments are used. Higher initial friction coefficients are assigned to \Hopper and \Walker to stabilize their balance, and these are increased substantially to test the agents under more challenging conditions.
    \item Agent's Mass: In MuJoCo, an agent's mass is determined by its constituent geoms. We scale the mass of all geoms uniformly to maintain the agent's mass distribution. For \Ant and \Swimmer, we use larger mass increments (up to 15x) as these agents are inherently more stable due to their multi-limbed or water-based nature. For bipedal agents like \HalfCheetah, \Hopper, and \Walker, we use smaller increments to avoid drastically altering their delicate balance dynamics.
\end{itemize}

These hyperparameter ranges are designed to gradually create an OOD scenario while maintaining feasible locomotion, allowing our \ourMethod approach to adapt progressively to more challenging scenarios.

\begin{table}[t]
\caption{Randomized hyperparameter scales used during our experiments. $\rightarrow$ shows one round of fine-tuning (iteration) using \ourMethod, i.e. $E_0 \rightarrow E_1 \rightarrow \cdots \rightarrow E_n$.}\label{table:rand_param_scales}
\centering
\small
\begin{tabular}{clll}
\toprule
Environment & \makecell[l]{Randomized\\Hyperparameter} & \makecell[l]{Original\\Scale} & Modified Scale \\
\midrule
\multirow{3}{*}{\Ant} 
& Initial Noise Scale & $1 \times 10^{-1}$ & $1 \times 10^{-5} \rightarrow 1 \times 10^{-3} \rightarrow 1 \times 10^{-1}$ \\
& Friction Coefficient & $1$ & $1 \rightarrow 1.25 \rightarrow 1.5 \rightarrow 1.75$ \\
& Agent's Mass & $1\text{x}$ & $1\text{x} \rightarrow 5\text{x} \rightarrow 10\text{x} \rightarrow 15\text{x}$ \\
\midrule
\multirow{3}{*}{\HalfCheetah} 
& Initial Noise Scale & $1 \times 10^{-1}$ & $1 \times 10^{-5} \rightarrow 1 \times 10^{-3} \rightarrow 1 \times 10^{-1}$ \\
& Friction Coefficient & $0.4$ & $0.4 \rightarrow 0.5 \rightarrow 0.6 \rightarrow 0.7$ \\
& Agent's Mass & $1\text{x}$ & $1\text{x} \rightarrow 1.05\text{x} \rightarrow 1.1\text{x} \rightarrow 1.15\text{x}$ \\
\midrule
\multirow{3}{*}{\Hopper} 
& Initial Noise Scale & $5 \times 10^{-3}$ & $5 \times 10^{-7} \rightarrow 5 \times 10^{-5} \rightarrow 5 \times 10^{-3} \rightarrow 5 \times 10^{-1}$ \\
& Friction Coefficient & $2$ & $2 \rightarrow 2.5 \rightarrow 3 \rightarrow 3.5$ \\
& Agent's Mass & $1\text{x}$ & $1\text{x} \rightarrow 1.15\text{x} \rightarrow 1.3\text{x} \rightarrow 1.45\text{x}$ \\
\midrule
\multirow{3}{*}{\Swimmer} 
& Initial Noise Scale & $1 \times 10^{-1}$ & $1 \times 10^{-5} \rightarrow 1 \times 10^{-3} \rightarrow 1 \times 10^{-1}$ \\
& Friction Coefficient & $0.1$ & $0.1 \rightarrow 0.5 \rightarrow 1 \rightarrow 1.5$ \\
& Agent's Mass & $1\text{x}$ & $1\text{x} \rightarrow 5\text{x} \rightarrow 10\text{x} \rightarrow 15\text{x}$ \\
\midrule
\multirow{3}{*}{\Walker} 
& Initial Noise Scale & $5 \times 10^{-3}$ & $5 \times 10^{-7} \rightarrow 5 \times 10^{-5} \rightarrow 5 \times 10^{-3} \rightarrow 5 \times 10^{-1}$ \\
& Friction Coefficient & $0.9$ & $0.9 \rightarrow 2 \rightarrow 3 \rightarrow 4$ \\
& Agent's Mass & $1\text{x}$ & $1\text{x} \rightarrow 1.1\text{x} \rightarrow 1.2\text{x} \rightarrow 1.3\text{x}$ \\
\bottomrule
\end{tabular}
\end{table}

\clearpage

\section{Ablation Study}\label{app:ablation}
This section presents a comprehensive ablation study to evaluate the efficacy and robustness of our proposed \ourMethod method. We conduct a series of experiments across multiple MuJoCo environments: \Ant, \HalfCheetah, \Hopper, \Swimmer, and \Walker. The study applies \ourMethod to three offline RL algorithms: Conservative Q-Learning (CQL), Advantage-Weighted Actor-Critic (AWAC), and TD3BC. We assess the impact of key randomized hyperparameters, namely initial noise scale, friction coefficient, and agent's mass, which are crucial for simulating real-world variability and testing the method's adaptability.

Our evaluation metrics encompass cumulative return during training, OOD detection accuracy, and sample efficiency. We examine these metrics across both the initial offline training phase and the subsequent fine-tuning phases where applicable. To ensure statistical significance and robustness of our findings, each configuration is tested using five random seeds. Throughout the study, we maintain consistency in all hyperparameters, network architectures, and settings, except for those specifically under investigation.

The ablation study is structured to provide insights into several key aspects of \ourMethod. 

First, in \aref{app:uncertainty}, we delve into the OOD detection capabilities of \ourMethod, a crucial component for safe and robust RL deployment. We examine how the method's uncertainty estimation, implemented through diverse critics, enables effective differentiation between ID and OOD samples. This analysis is particularly relevant for assessing the method's potential in real-world applications where encountering novel situations is inevitable.

Then, in \aref{app:performance}, we present a detailed analysis of overall performance, expanding on the results provided in the main paper. This includes cumulative return during training for all five environments, offering a comprehensive view of \ourMethod's impact across diverse locomotion tasks. 

We then investigate the impact of the balancing replay buffer mechanism, a key innovation in \ourMethod, in \aref{app:exp_balancing_rb}. This component is designed to manage the transition between offline and online learning effectively, and we present results showing its influence on learning stability and performance.

Next, in \aref{app:sample_efficiency}, we focus on sample efficiency, a critical factor in the practicality of RL algorithms. By comparing \ourMethod against baselines trained on the full state space from the outset, we demonstrate how our iterative approach to expanding the state space contributes to more efficient learning.

In \aref{app:off_dyn}, we compare \ourMethod's policy performance against state-of-the-art off-dynamics RL baselines, highlighting how explicit uncertainty integration improves robustness under dynamics shift.  

In \aref{app:hyperparams_ablation}, we present a sensitivity analysis over the diversity weight $\lambda$ and the expansion threshold $\delta$, quantifying their individual and joint effects on training stability, convergence speed, and final return.  

Lastly, in \aref{app:anymal_exps}, we detail our ANYmal-D case study, from progressive domain randomization to online OOD monitoring, illustrating practical implementation choices and safety implications on real hardware.

Throughout this ablation study, we aim to provide a nuanced understanding of \ourMethod's components and their contributions to its overall effectiveness. The results and analyses presented here complement and expand upon the findings in the main paper, offering deeper insights into the method's behavior across a range of environments and conditions.




\clearpage


\clearpage
\subsection{OOD Detection}\label{app:uncertainty}
This section provides an in-depth analysis of \ourMethod's OOD detection capabilities across various environments and randomized hyperparameters. We measure the critic variance across $100$ rollouts for both AWAC-based and CQL-based methods, comparing \ourMethod with baseline approaches. In the following figures (\autoref{fig:app_uncertainty_ant_awac} through \autoref{fig:app_uncertainty_walker2d_cql}), each column represents a fine-tuning iteration with an expanded ID range, as detailed in \autoref{app:random_params}. The \colorbox{sbOrange025}{orange line} indicates the $95\%$ CI of critic variances for ID samples, serving as our OOD detection threshold.

\begin{figure}[htbp]
    \centering
    \begin{subfigure}{}
        \centering
        \includegraphics[width=0.8\textwidth]{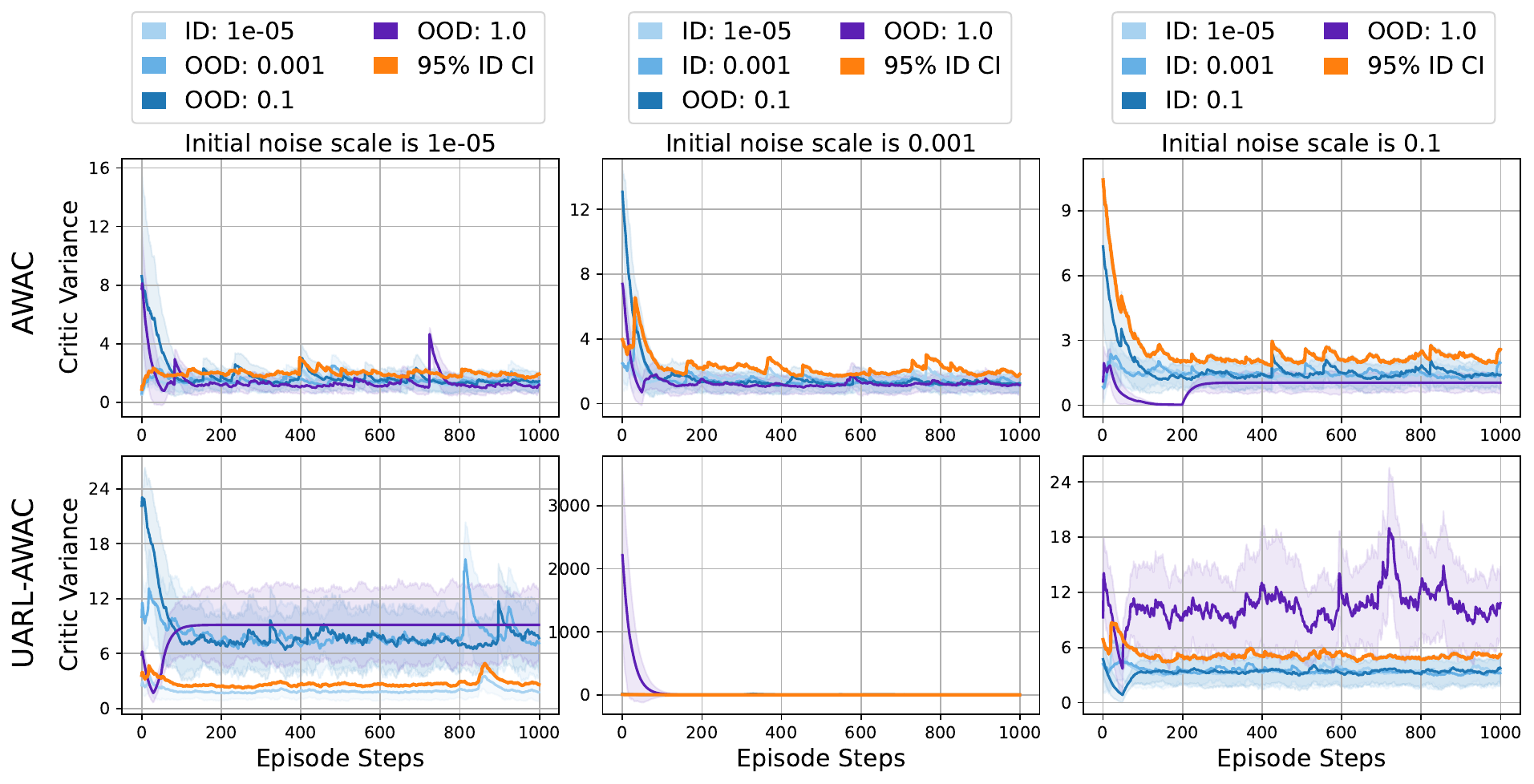}
    \end{subfigure}
    \begin{subfigure}{}
        \centering
        \includegraphics[width=0.8\textwidth]{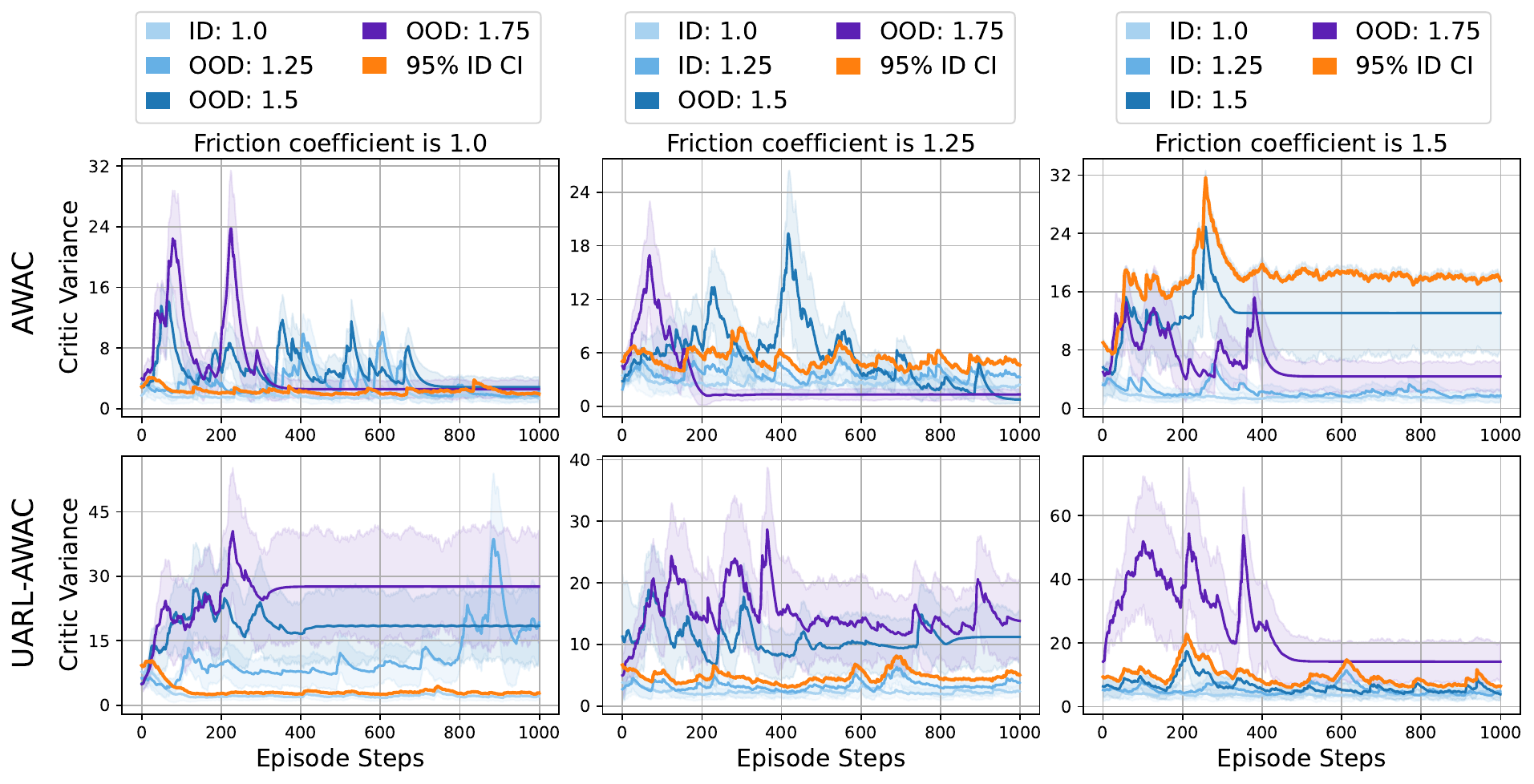}
    \end{subfigure}
    \caption{The OOD detection results for AWAC and \ourMethod-AWAC initial noise scale (top) and friction coefficient (bottom) over the \Ant environment.}
    \label{fig:app_uncertainty_ant_awac}
\end{figure}

\begin{figure}[htbp]
    \centering
    \begin{subfigure}{}
        \centering
        \includegraphics[width=0.8\textwidth]{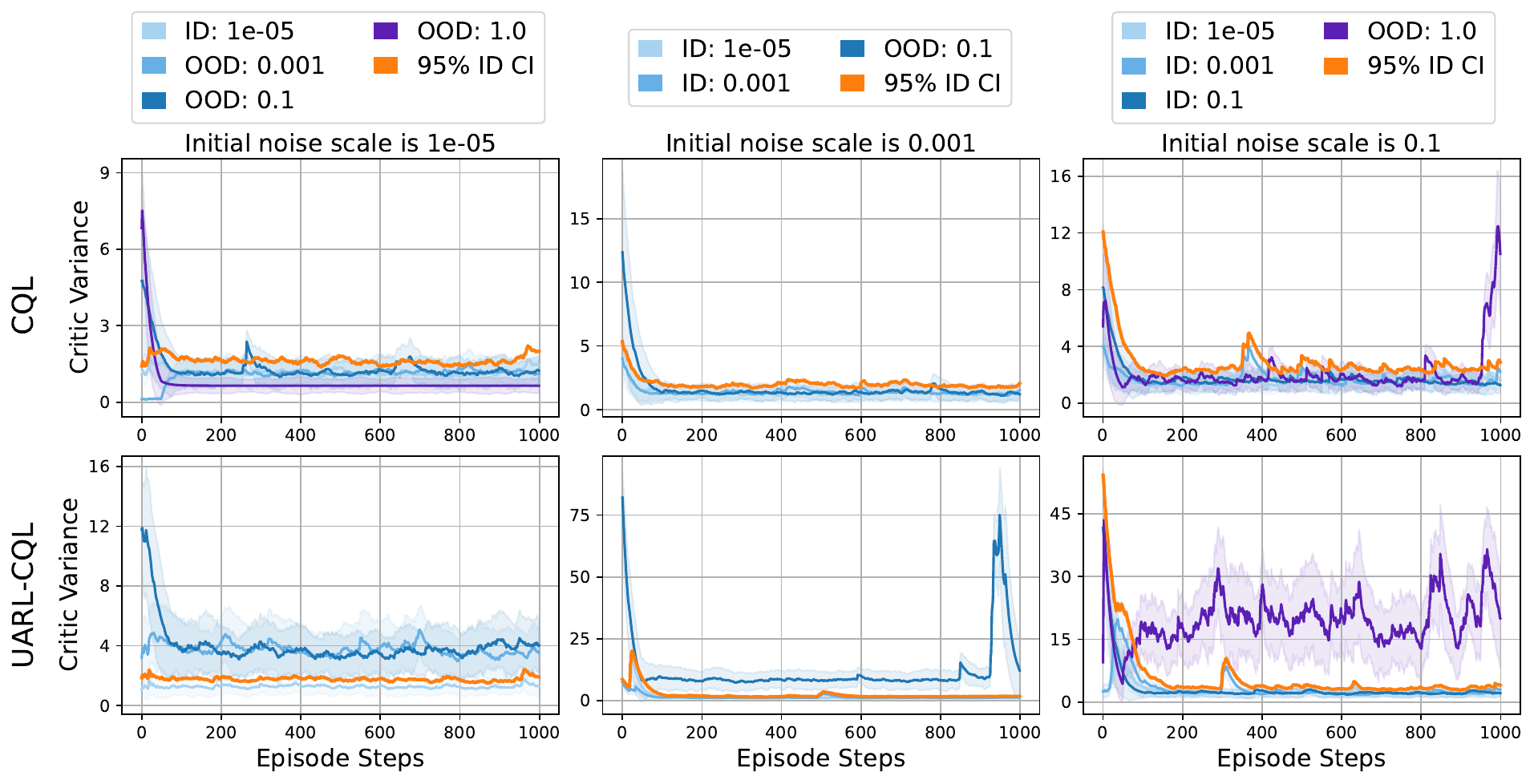}
    \end{subfigure}
    \begin{subfigure}{}
        \centering
        \includegraphics[width=0.8\textwidth]{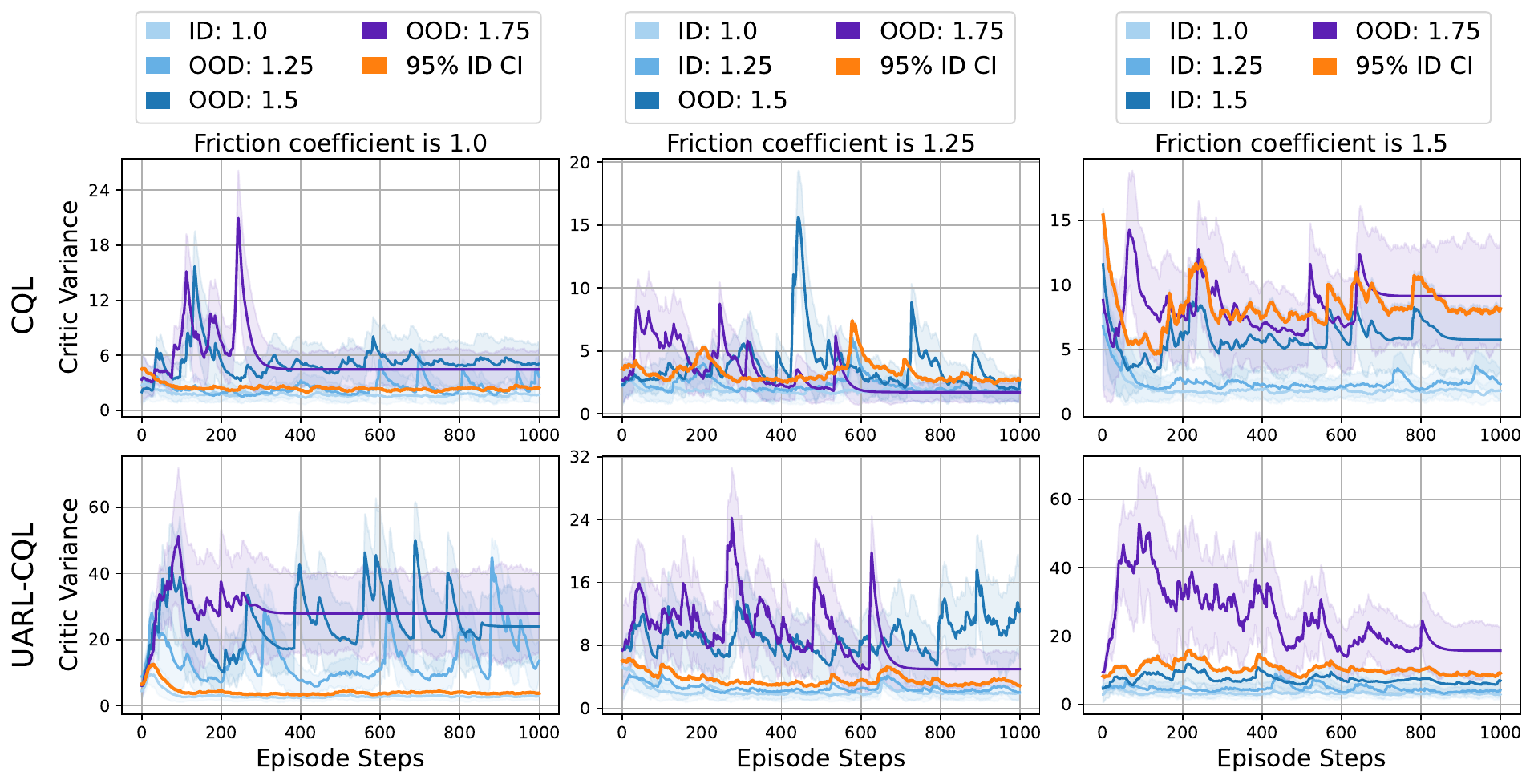}
    \end{subfigure}
    \begin{subfigure}{}
        \centering
        \includegraphics[width=0.8\textwidth]{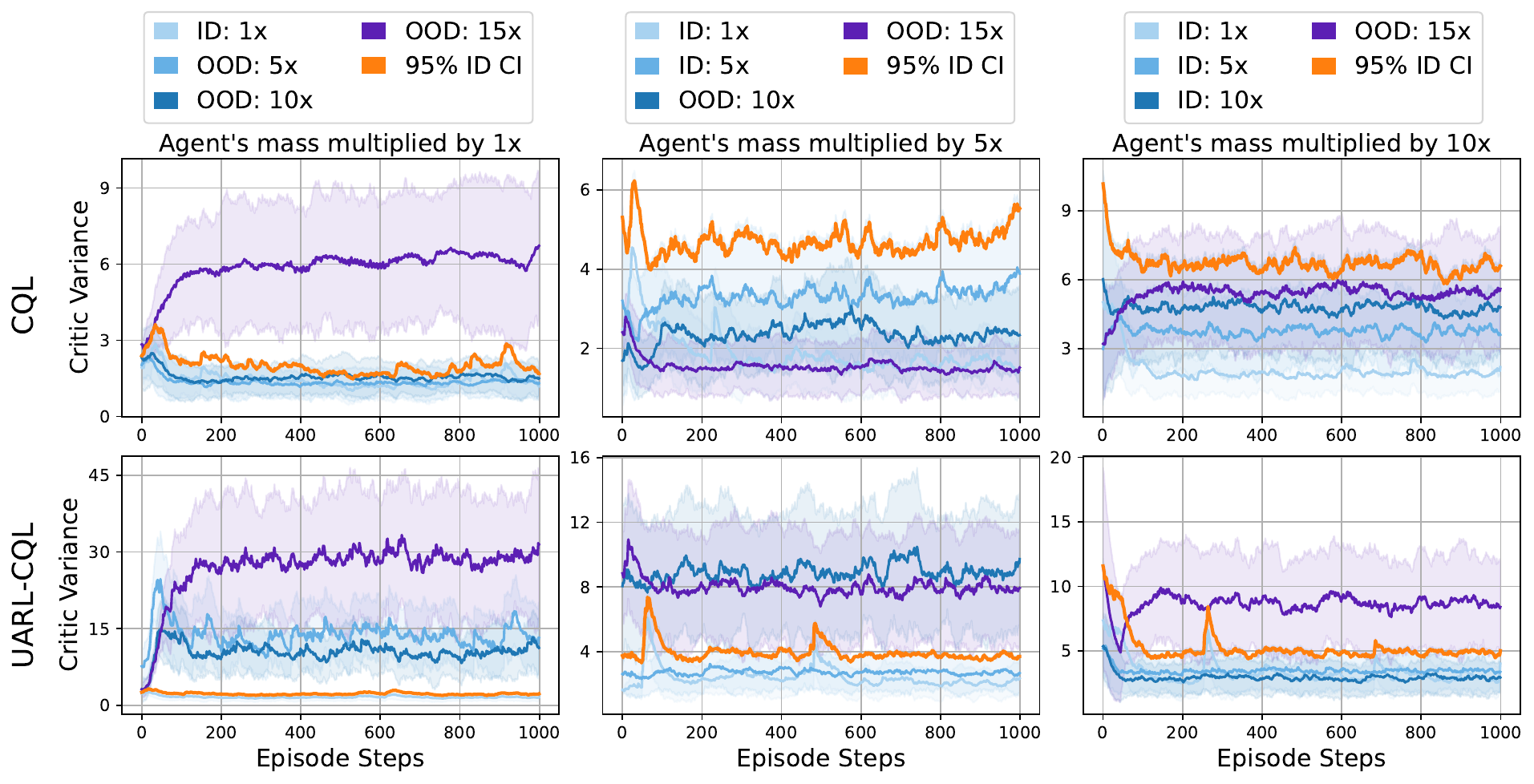}
    \end{subfigure}
    \caption{The OOD detection results for CQL and \ourMethod-CQL initial noise scale (top) and friction coefficient (middle), and agent's mass (bottom) over the \Ant environment.}
    \label{fig:app_uncertainty_ant_cql}
\end{figure}

\begin{figure}[htbp]
    \centering
    \begin{subfigure}{}
        \centering
        \includegraphics[width=0.8\textwidth]{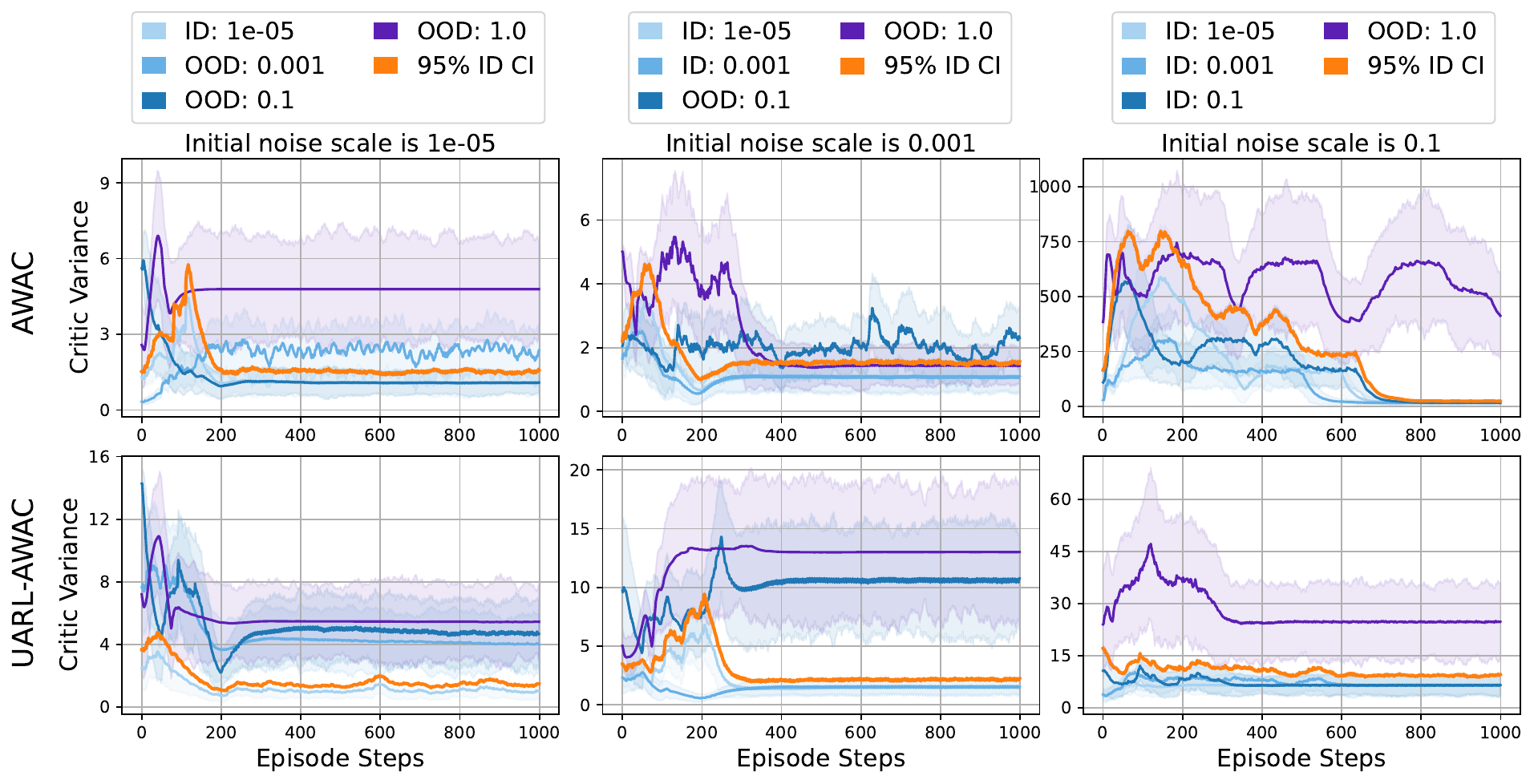}
    \end{subfigure}
    \begin{subfigure}{}
        \centering
        \includegraphics[width=0.8\textwidth]{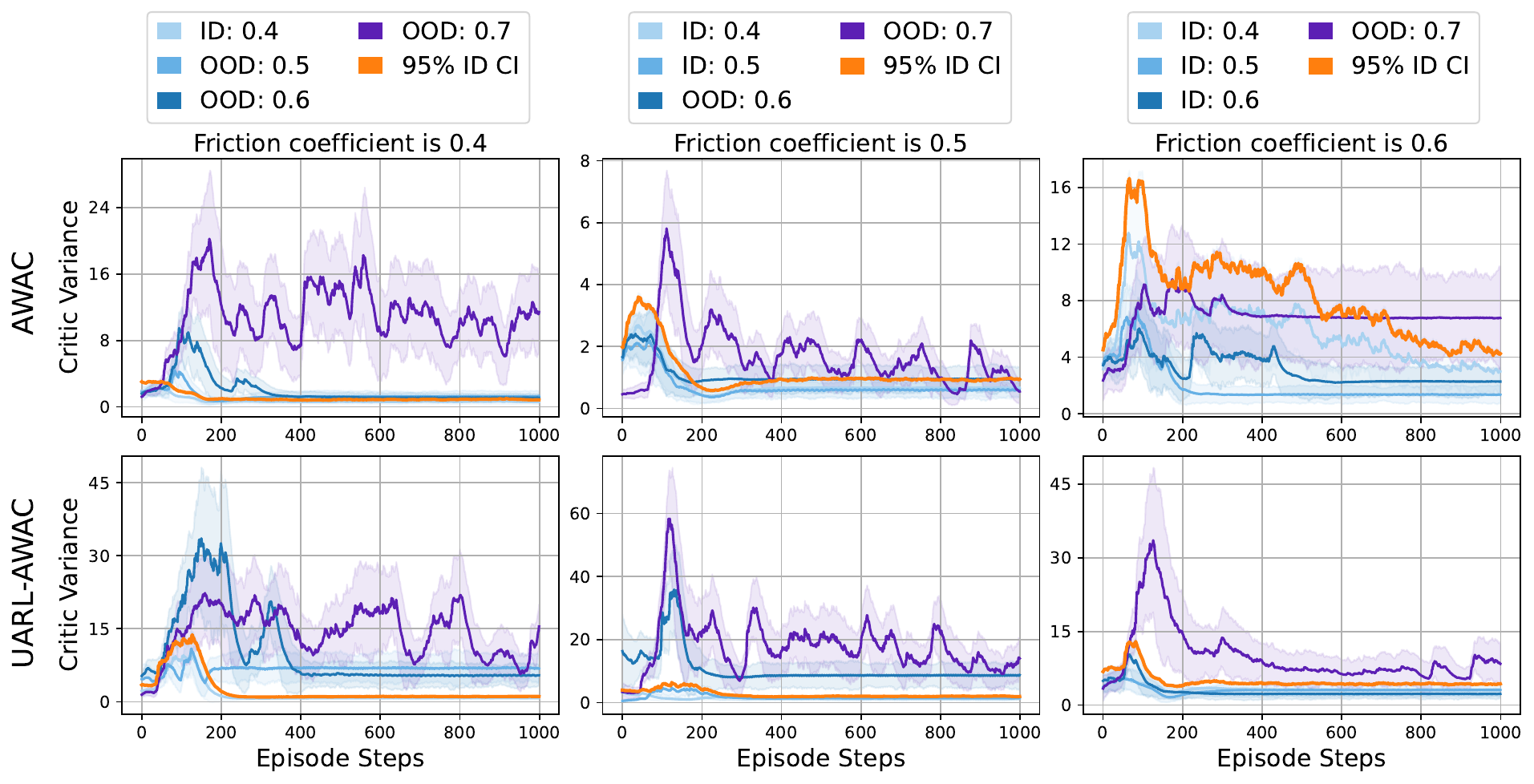}
    \end{subfigure}
    \begin{subfigure}{}
        \centering
        \includegraphics[width=0.8\textwidth]{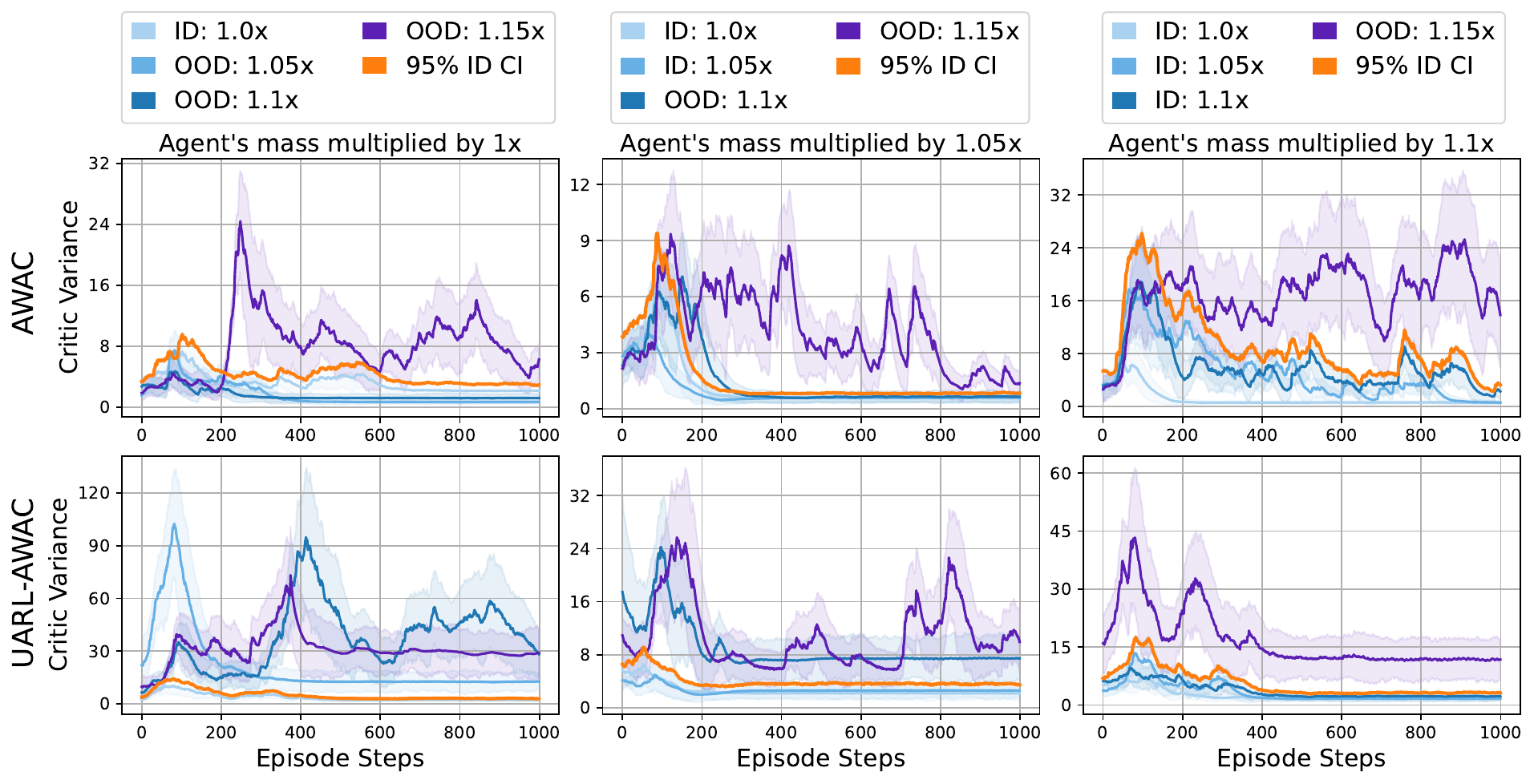}
    \end{subfigure}
    \caption{The OOD detection results for AWAC and \ourMethod-AWAC initial noise scale (top) and friction coefficient (middle), and agent's mass (bottom) over the \HalfCheetah environment.}
    \label{fig:app_uncertainty_halfcheetah_awac}
\end{figure}

\begin{figure}[htbp]
    \centering
    \begin{subfigure}{}
        \centering
        \includegraphics[width=0.8\textwidth]{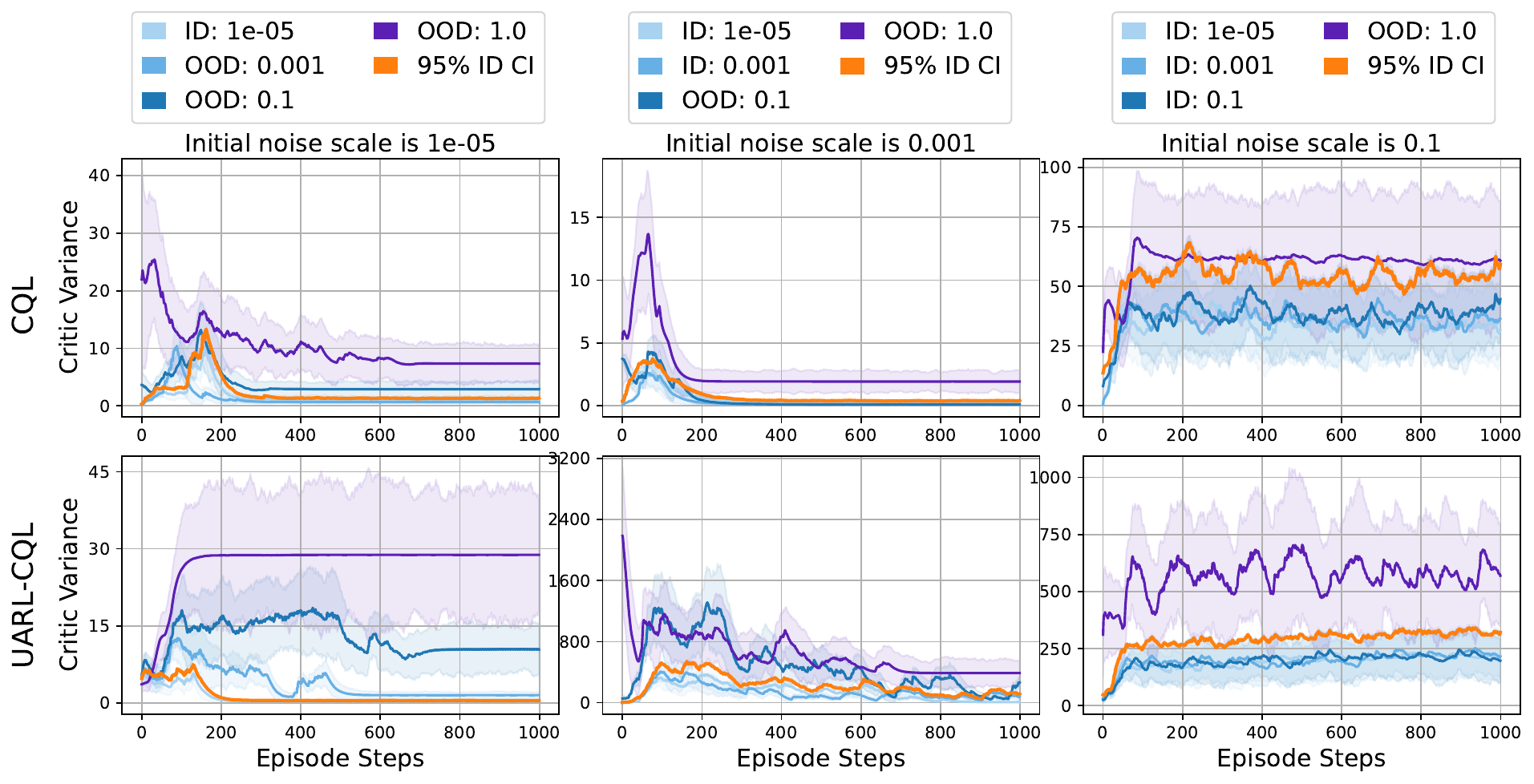}
    \end{subfigure}
    \begin{subfigure}{}
        \centering
        \includegraphics[width=0.8\textwidth]{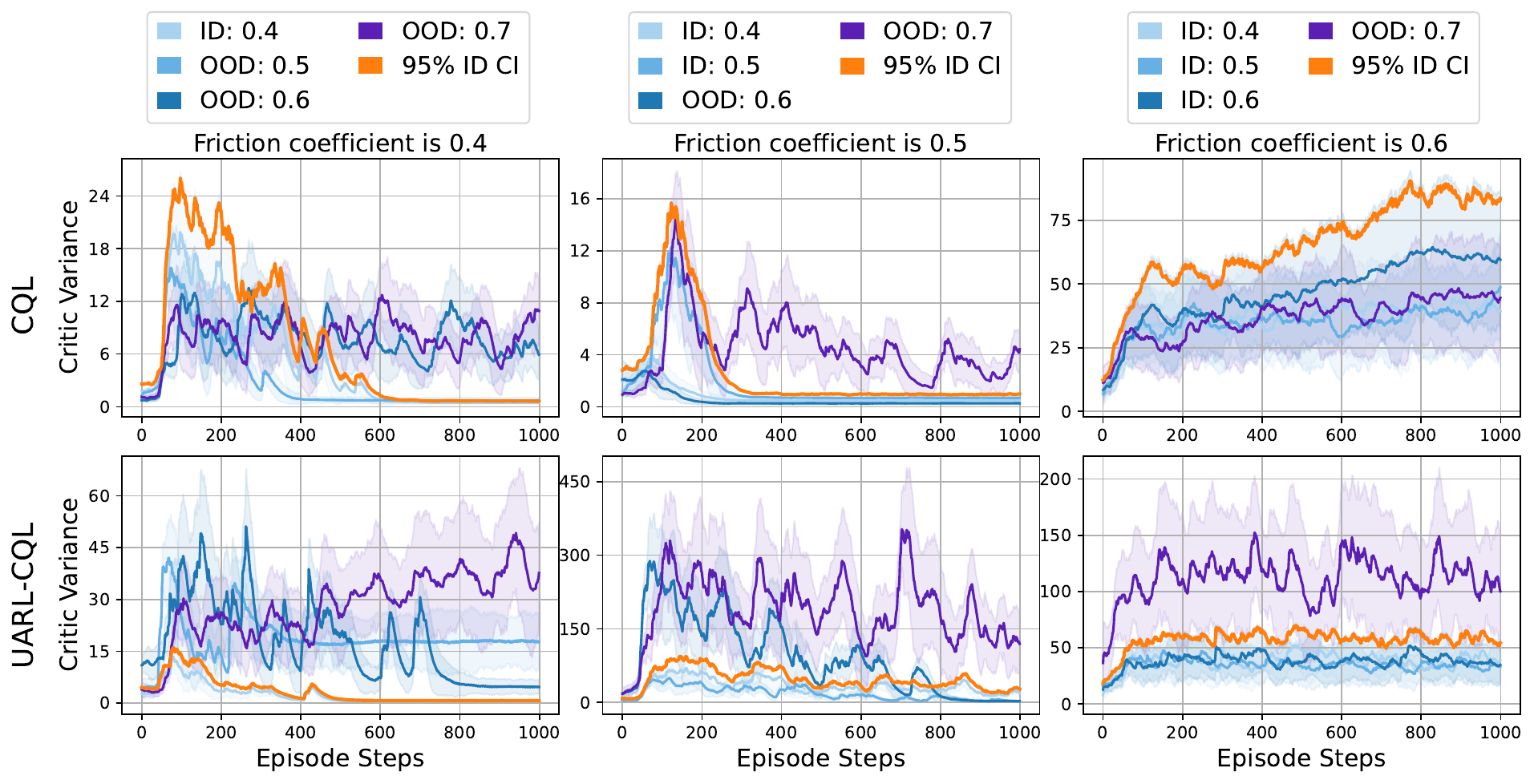}
    \end{subfigure}
    \begin{subfigure}{}
        \centering
        \includegraphics[width=0.8\textwidth]{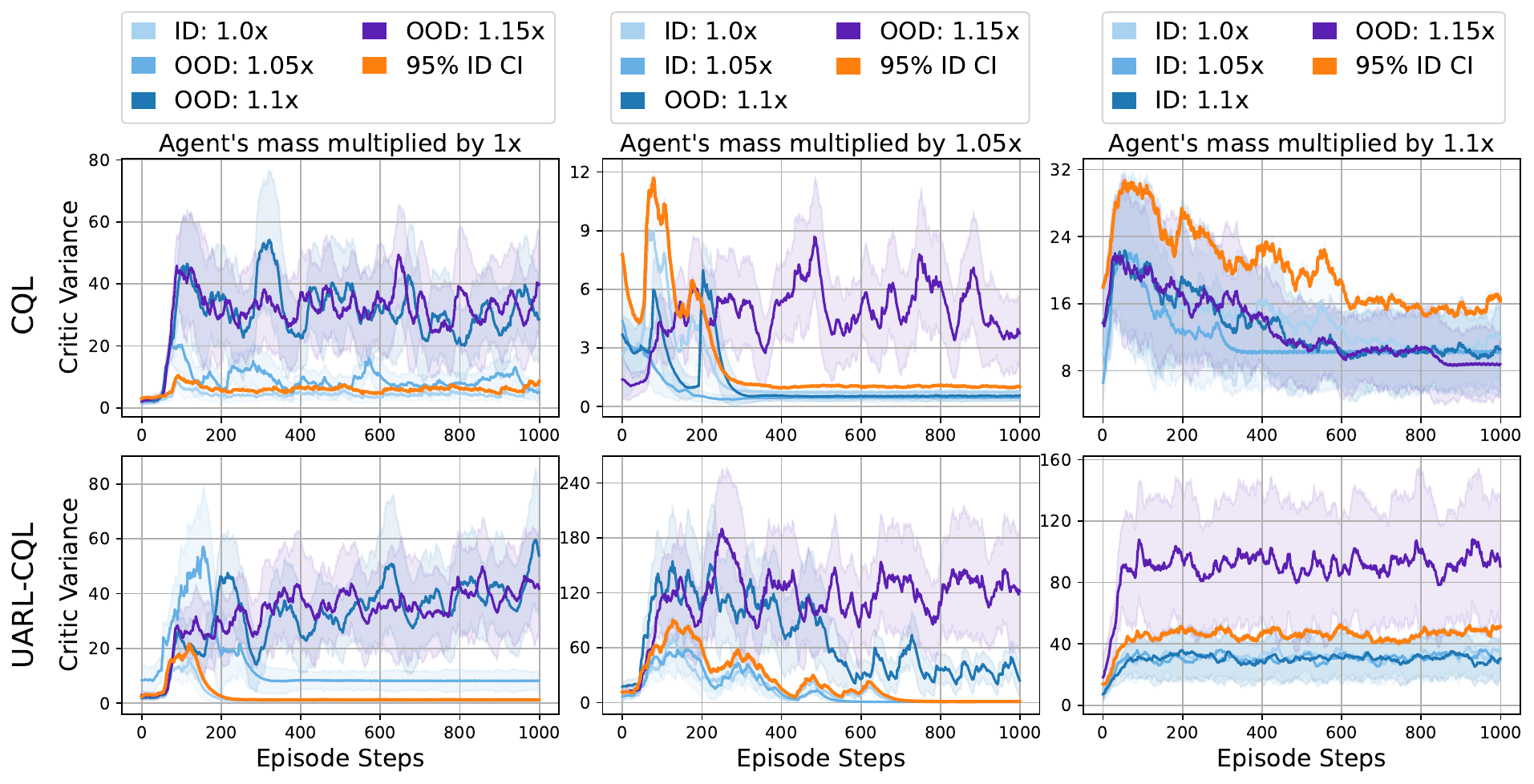}
    \end{subfigure}
    \caption{The OOD detection results for CQL and \ourMethod-CQL initial noise scale (top) and friction coefficient (middle), and agent's mass (bottom) over the \HalfCheetah environment.}
    \label{fig:app_uncertainty_halfcheetah_cql}
\end{figure}

\begin{figure}[htbp]
    \centering
    \begin{subfigure}{}
        \centering
        \includegraphics[width=0.8\textwidth]{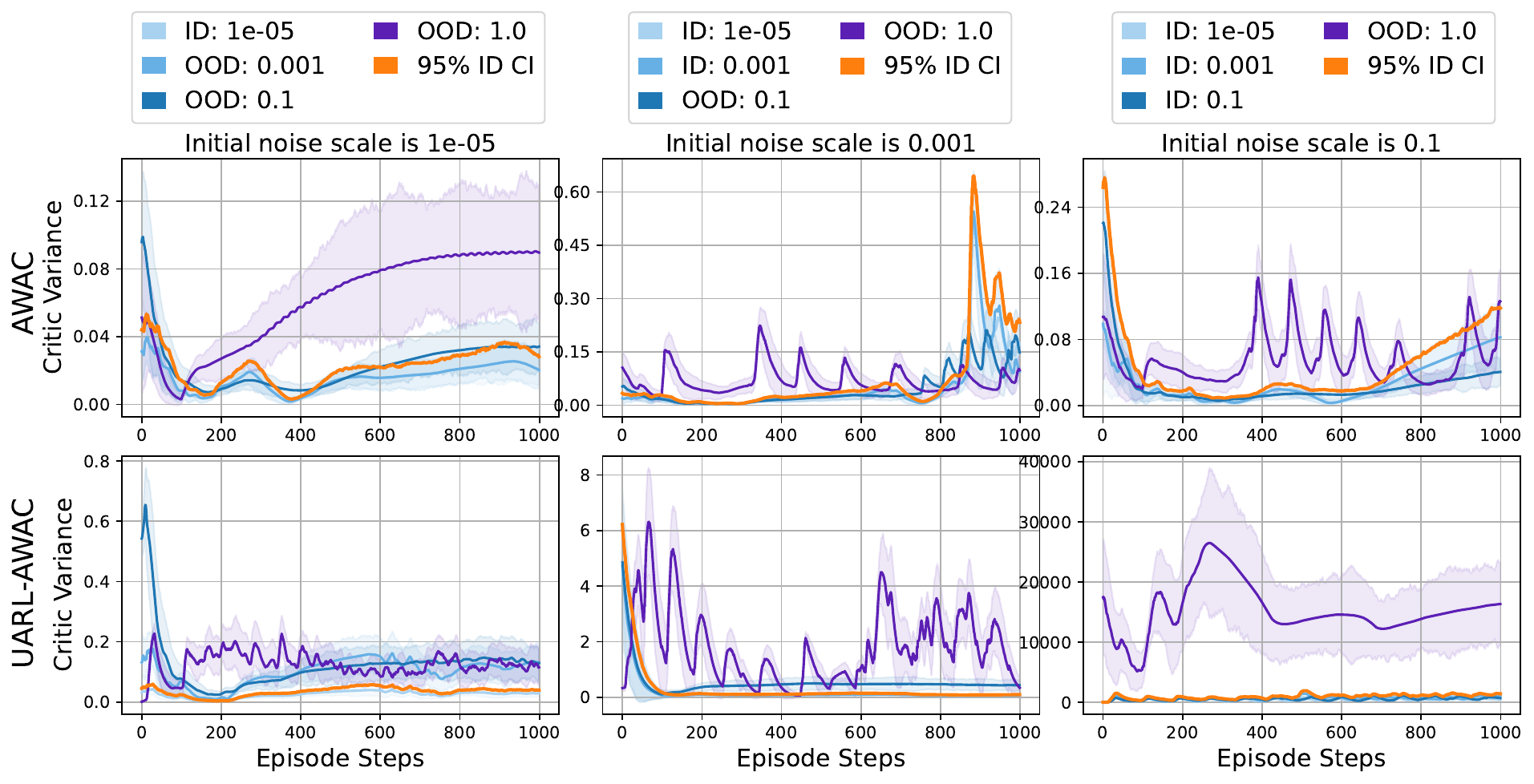}
    \end{subfigure}
    \begin{subfigure}{}
        \centering
        \includegraphics[width=0.8\textwidth]{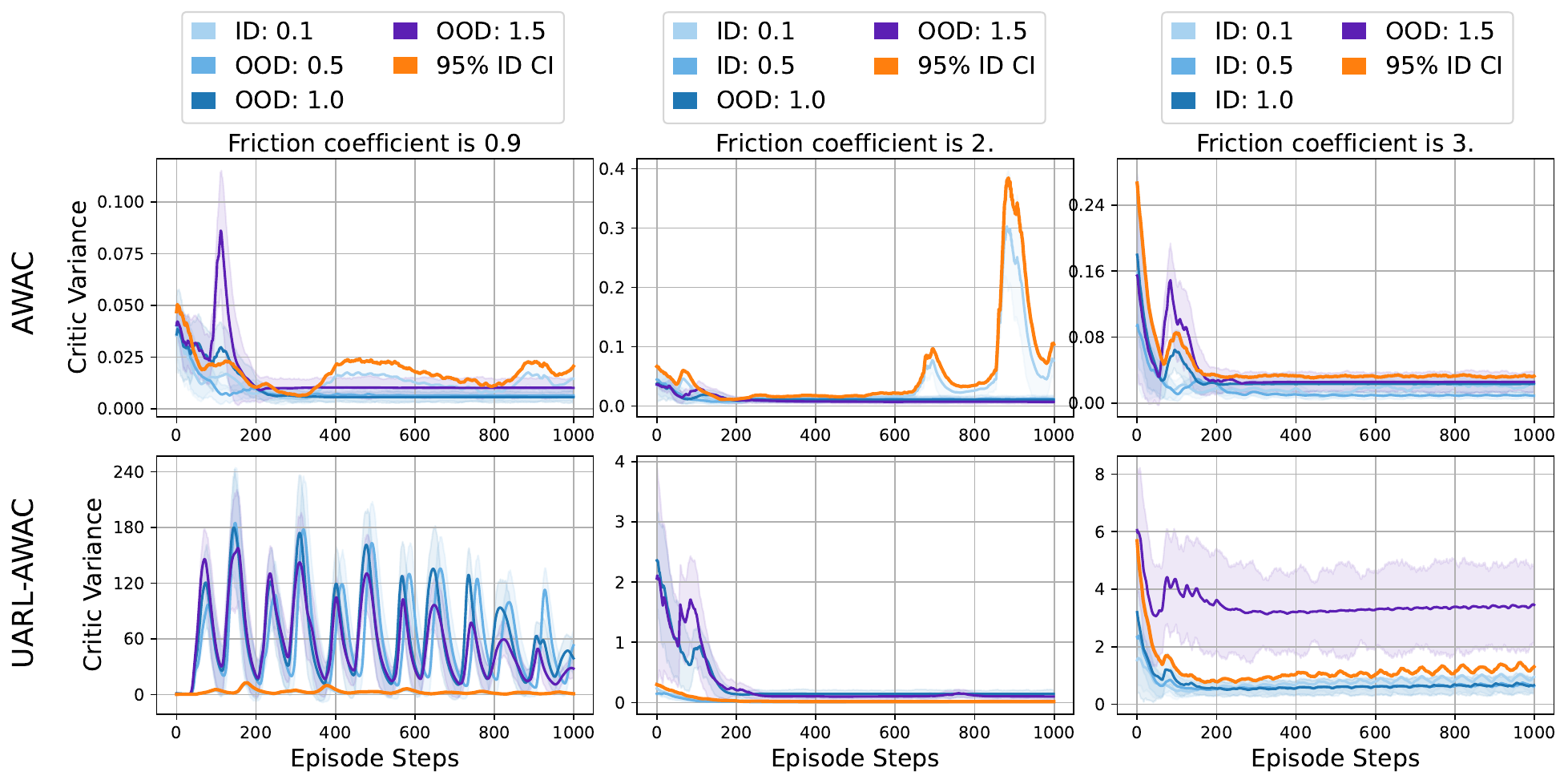}
    \end{subfigure}
    \begin{subfigure}{}
        \centering
        \includegraphics[width=0.8\textwidth]{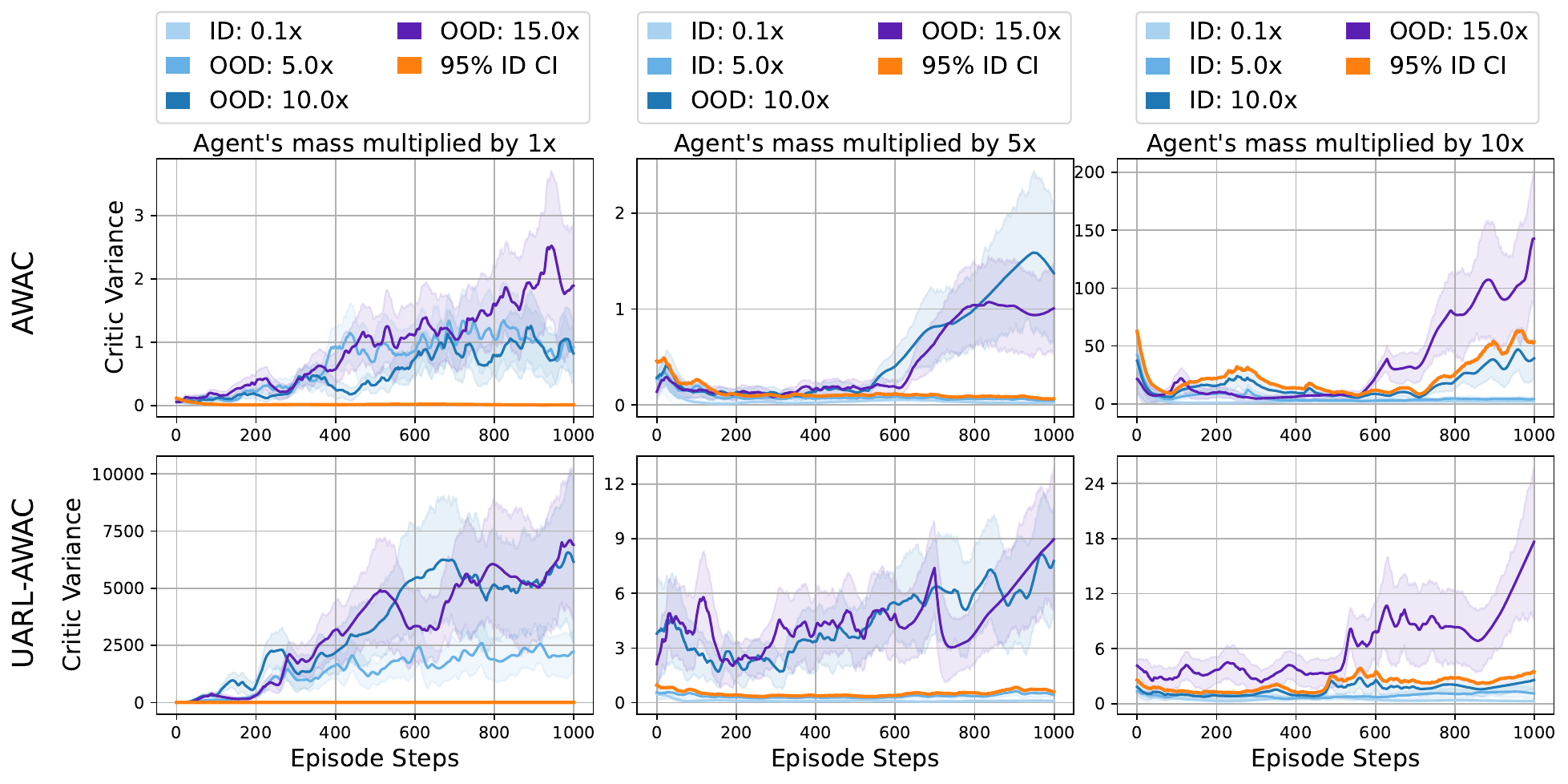}
    \end{subfigure}
    \caption{The OOD detection results for AWAC and \ourMethod-AWAC initial noise scale (top) and friction coefficient (middle), and agent's mass (bottom) over the \Swimmer environment.}
    \label{fig:app_uncertainty_swimmer_awac}
\end{figure}

\begin{figure}[htbp]
    \centering
    \begin{subfigure}{}
        \centering
        \includegraphics[width=0.8\textwidth]{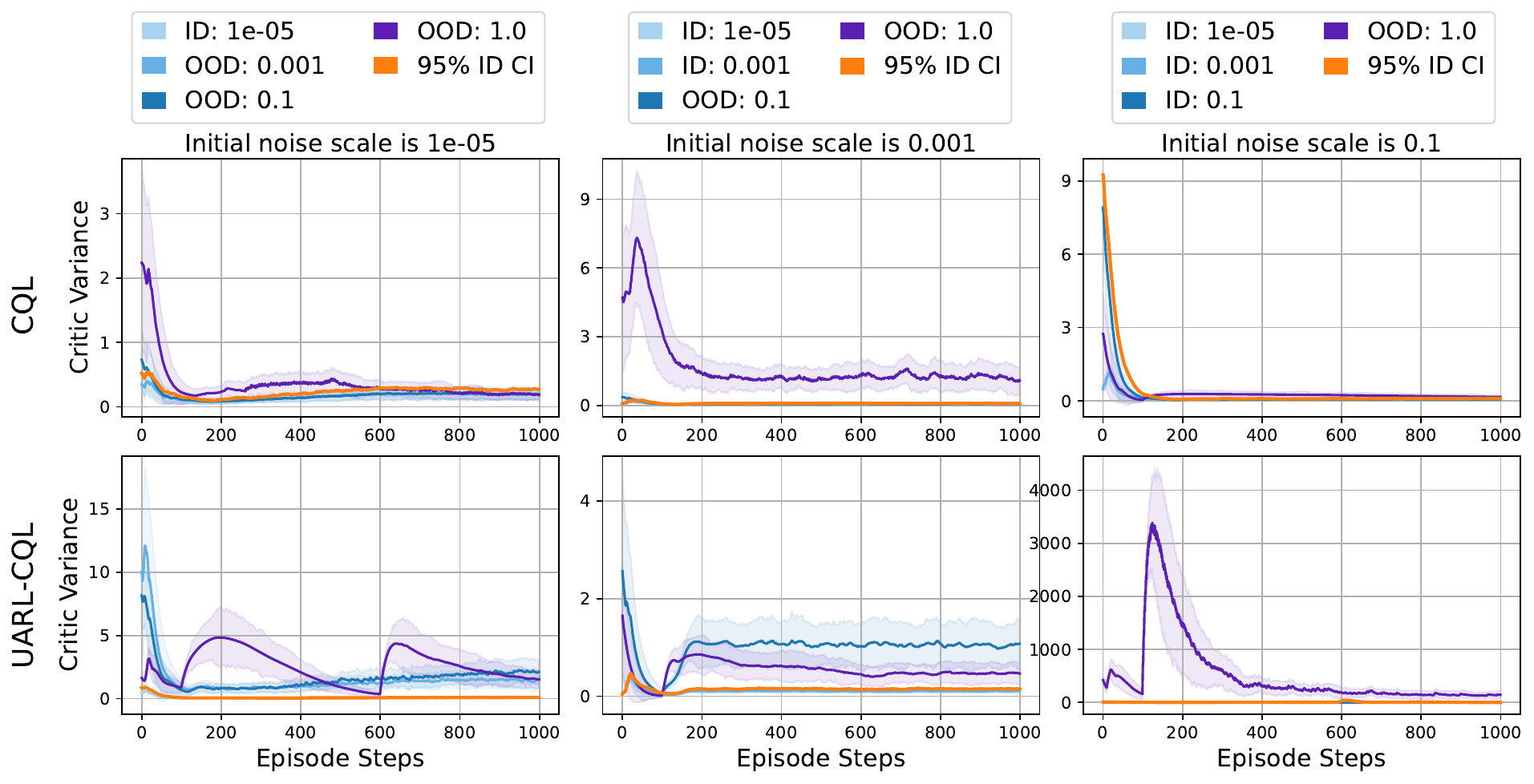}
    \end{subfigure}
    \begin{subfigure}{}
        \centering
        \includegraphics[width=0.8\textwidth]{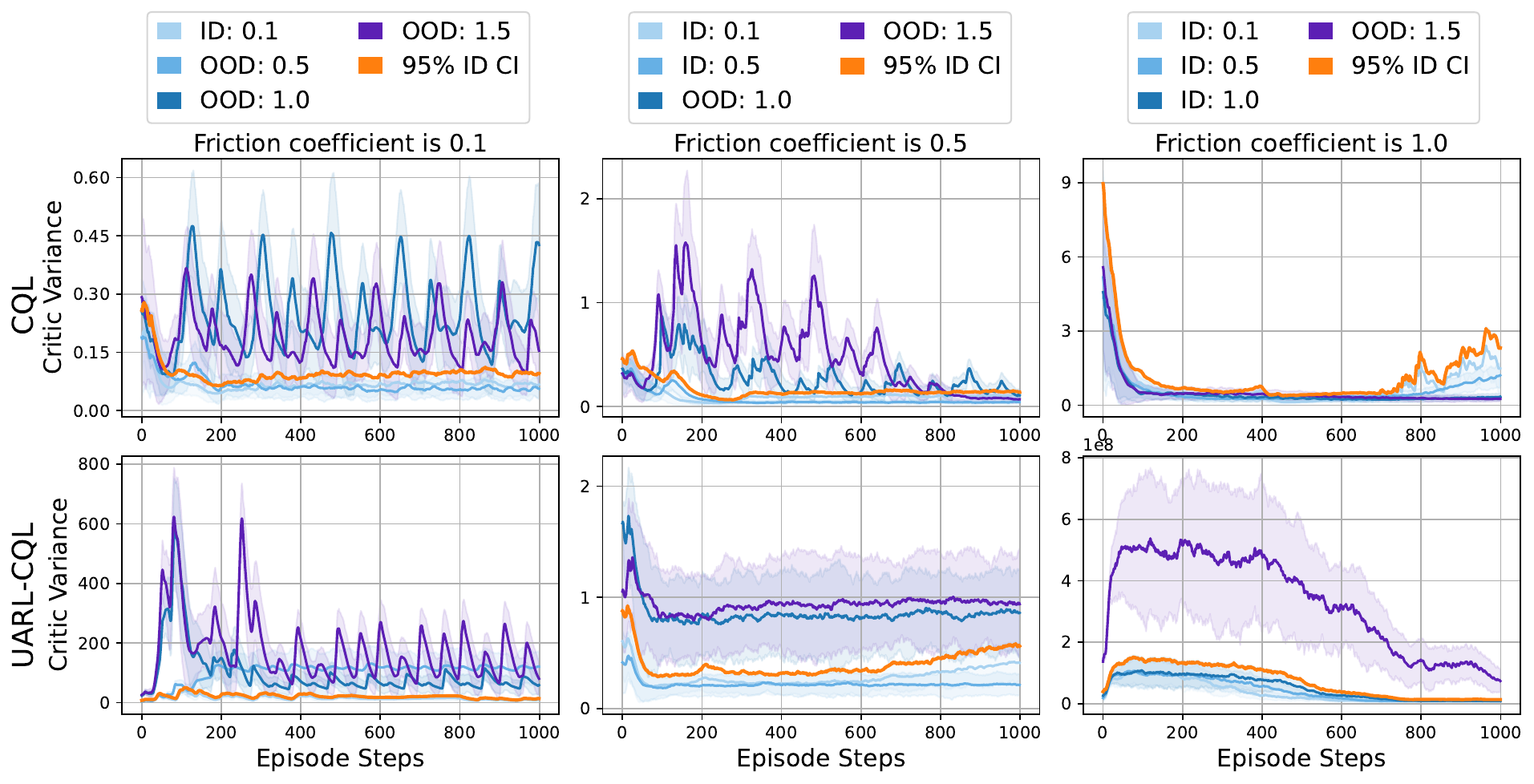}
    \end{subfigure}
    \begin{subfigure}{}
        \centering
        \includegraphics[width=0.8\textwidth]{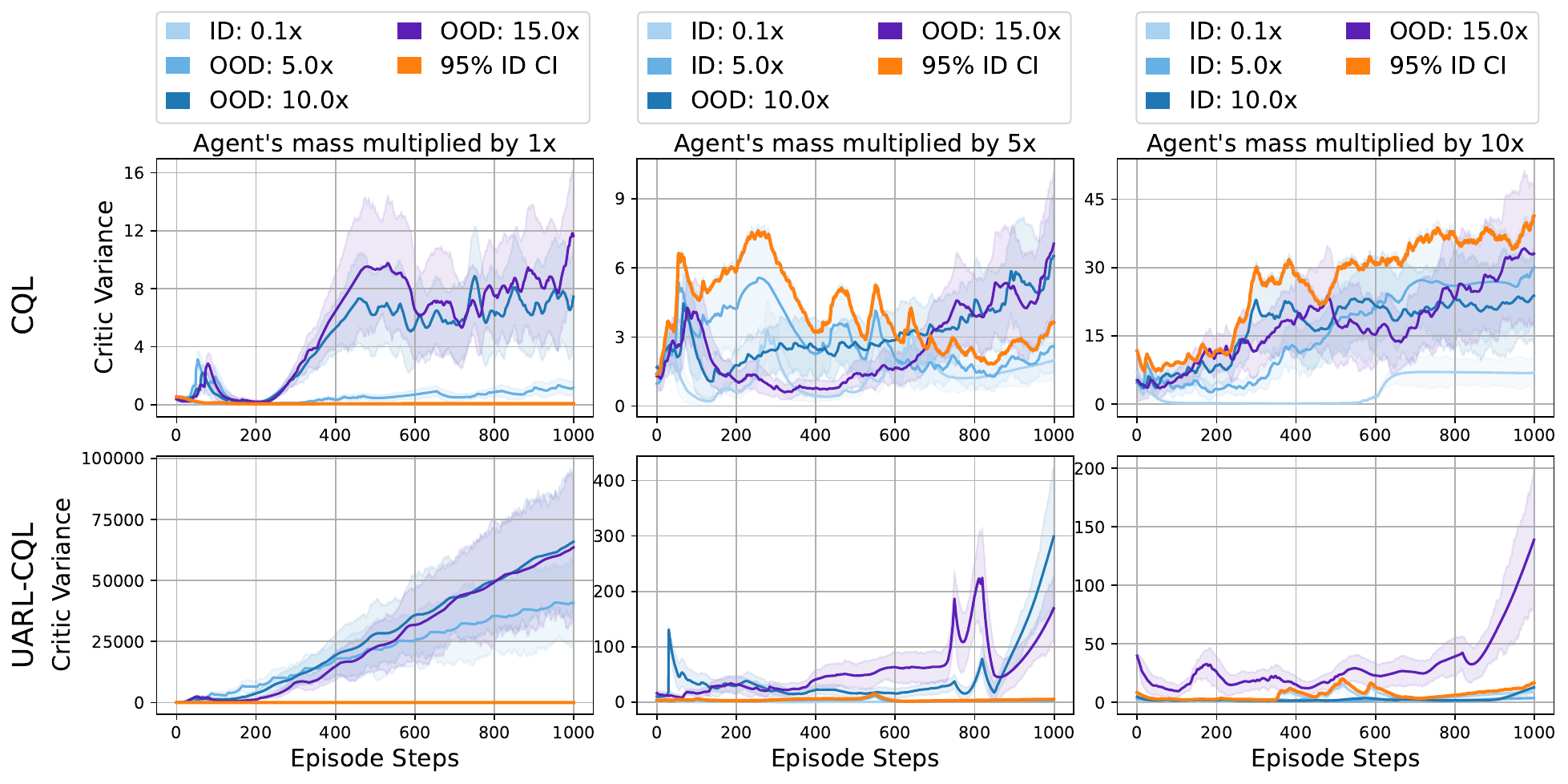}
    \end{subfigure}
    \caption{The OOD detection results for CQL and \ourMethod-CQL initial noise scale (top) and friction coefficient (middle), and agent's mass (bottom) over the \Swimmer environment.}
    \label{fig:app_uncertainty_swimmer_cql}
\end{figure}

\begin{figure}[htbp]
    \centering
    \begin{subfigure}{}
        \centering
        \includegraphics[width=0.8\textwidth]{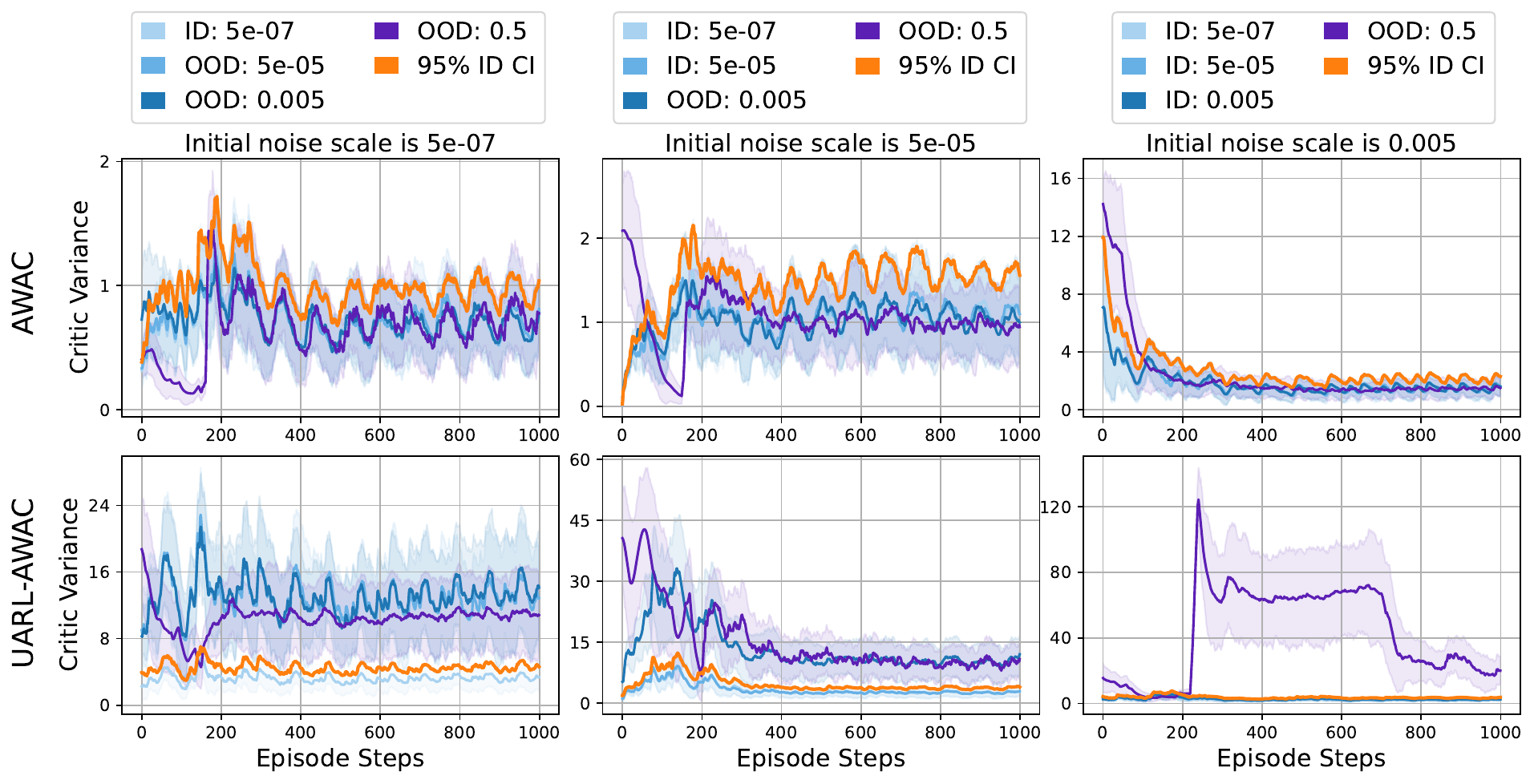}
    \end{subfigure}
    \begin{subfigure}{}
        \centering
        \includegraphics[width=0.8\textwidth]{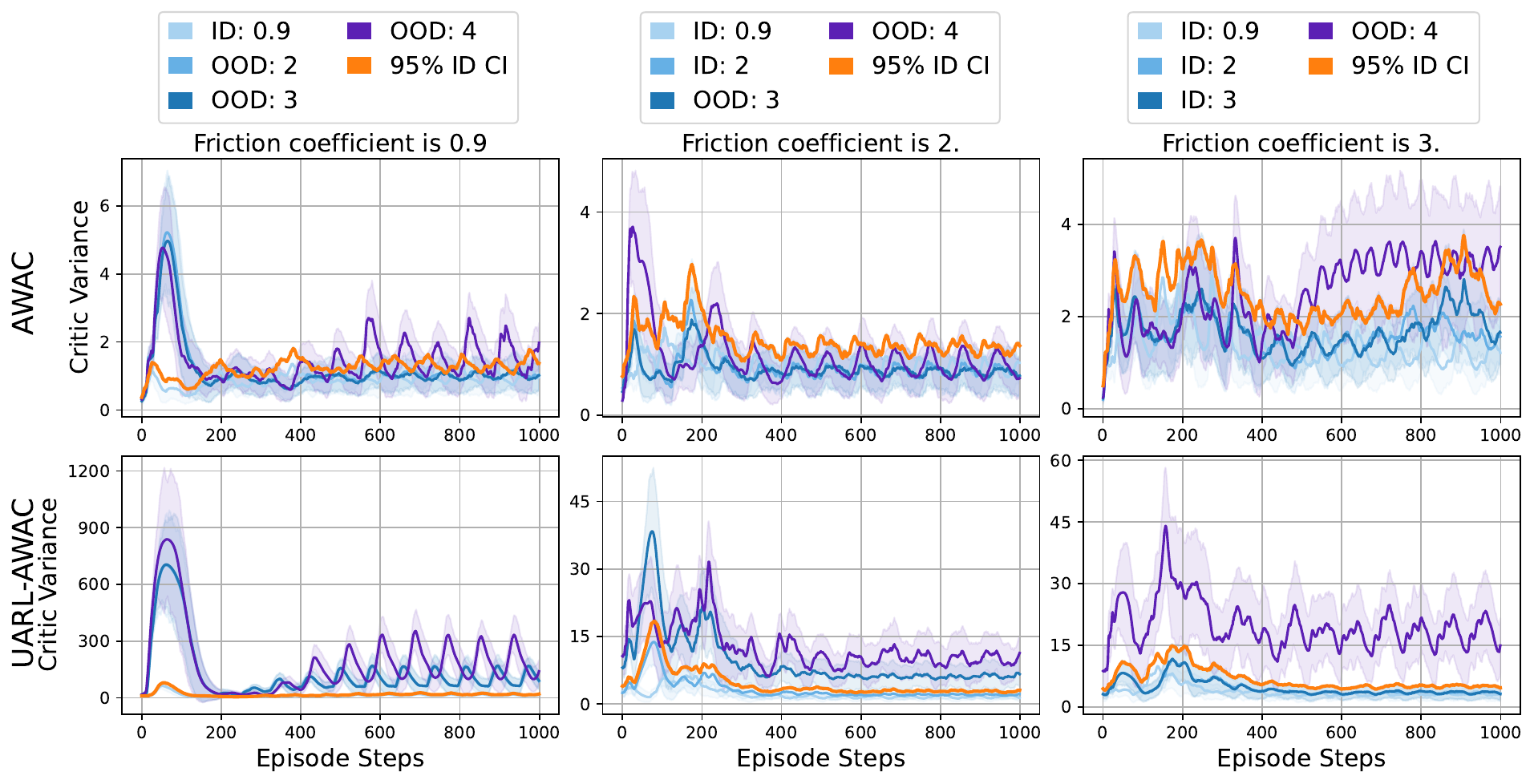}
    \end{subfigure}
    \begin{subfigure}{}
        \centering
        \includegraphics[width=0.8\textwidth]{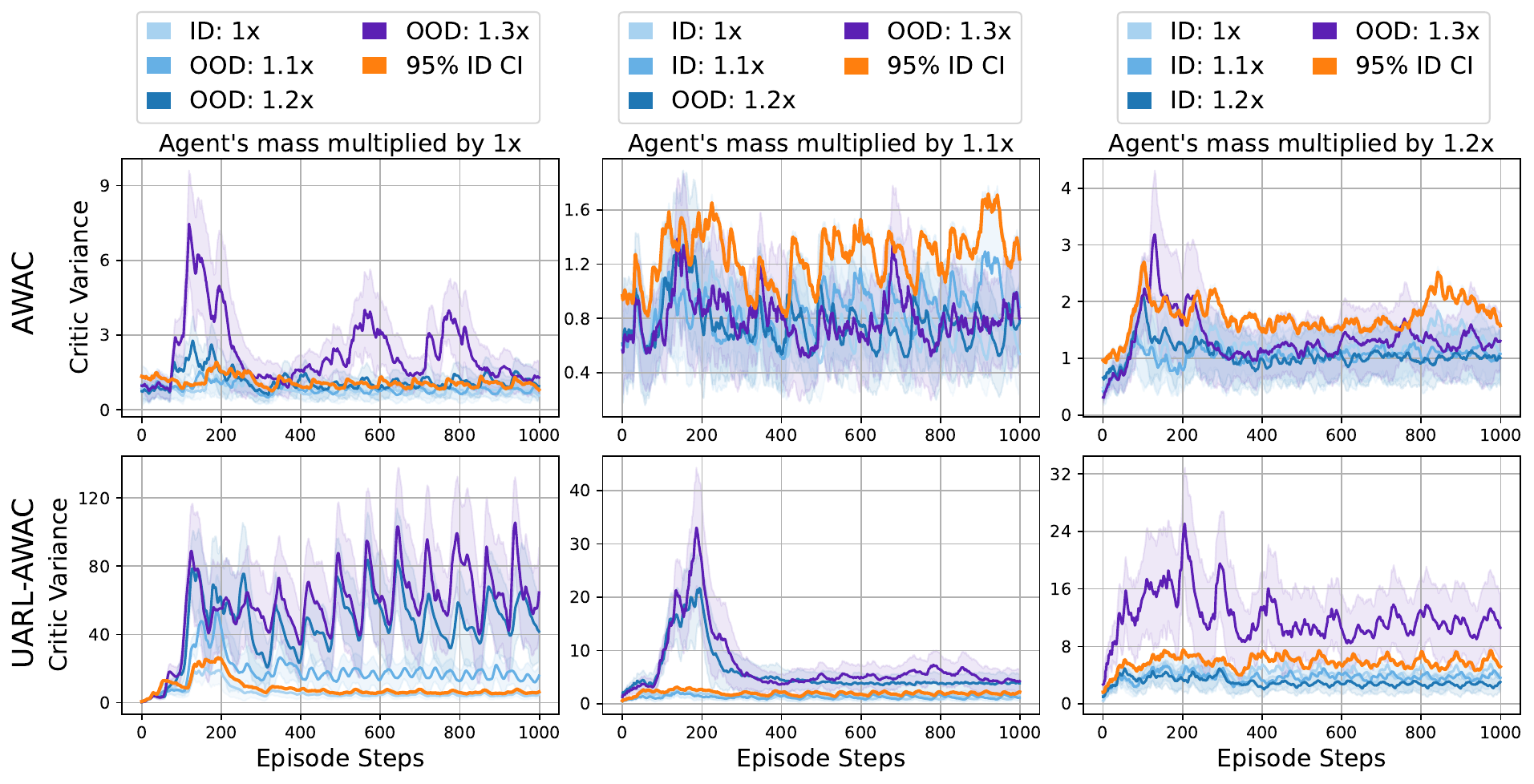}
    \end{subfigure}
    \caption{The OOD detection results for AWAC and \ourMethod-AWAC initial noise scale (top) and friction coefficient (middle), and agent's mass (bottom) over the \Walker environment.}
    \label{fig:app_uncertainty_walker2d_awac}
\end{figure}

\begin{figure}[htbp]
    \centering
    \begin{subfigure}{}
        \centering
        \includegraphics[width=0.8\textwidth]{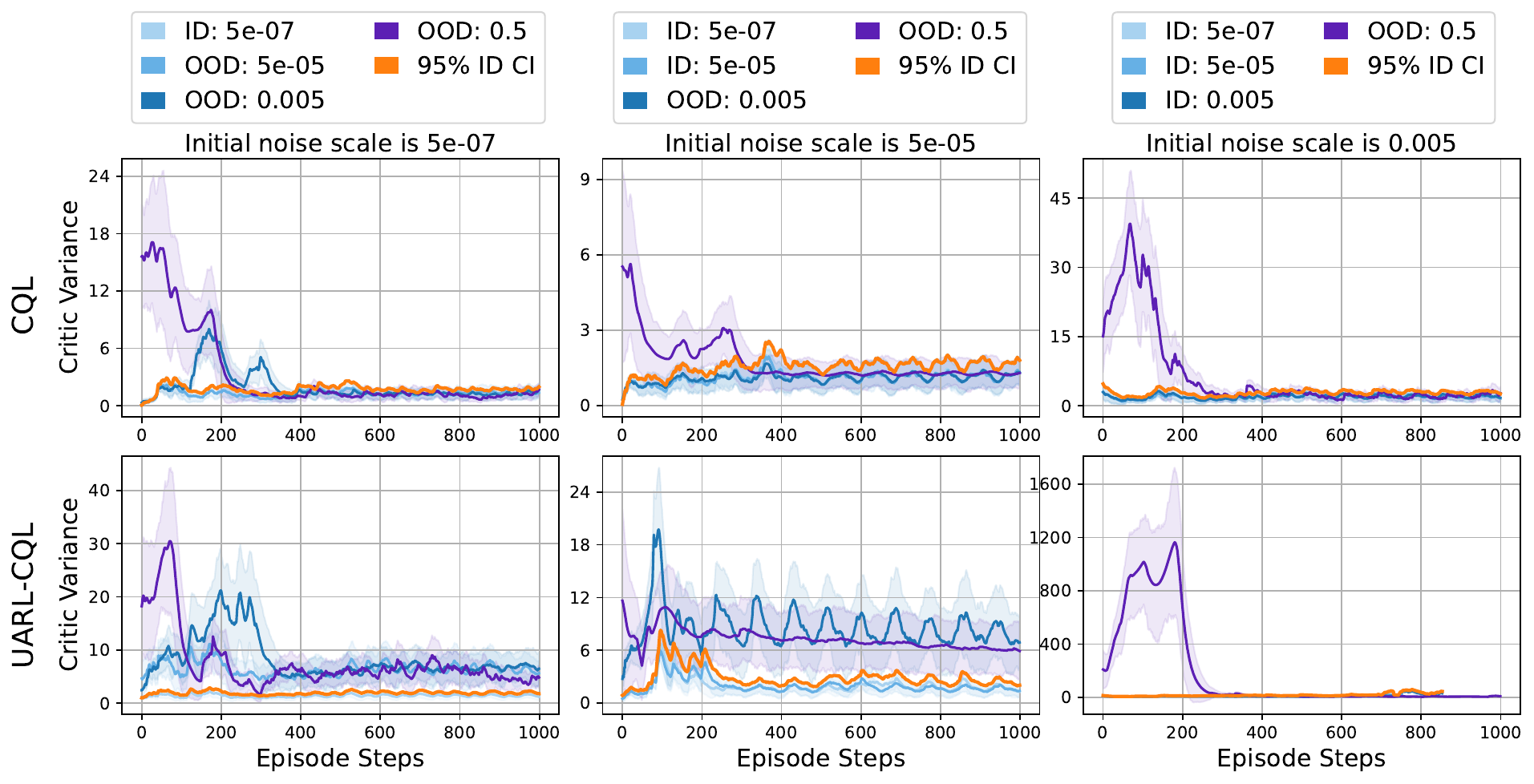}
    \end{subfigure}
    \begin{subfigure}{}
        \centering
        \includegraphics[width=0.8\textwidth]{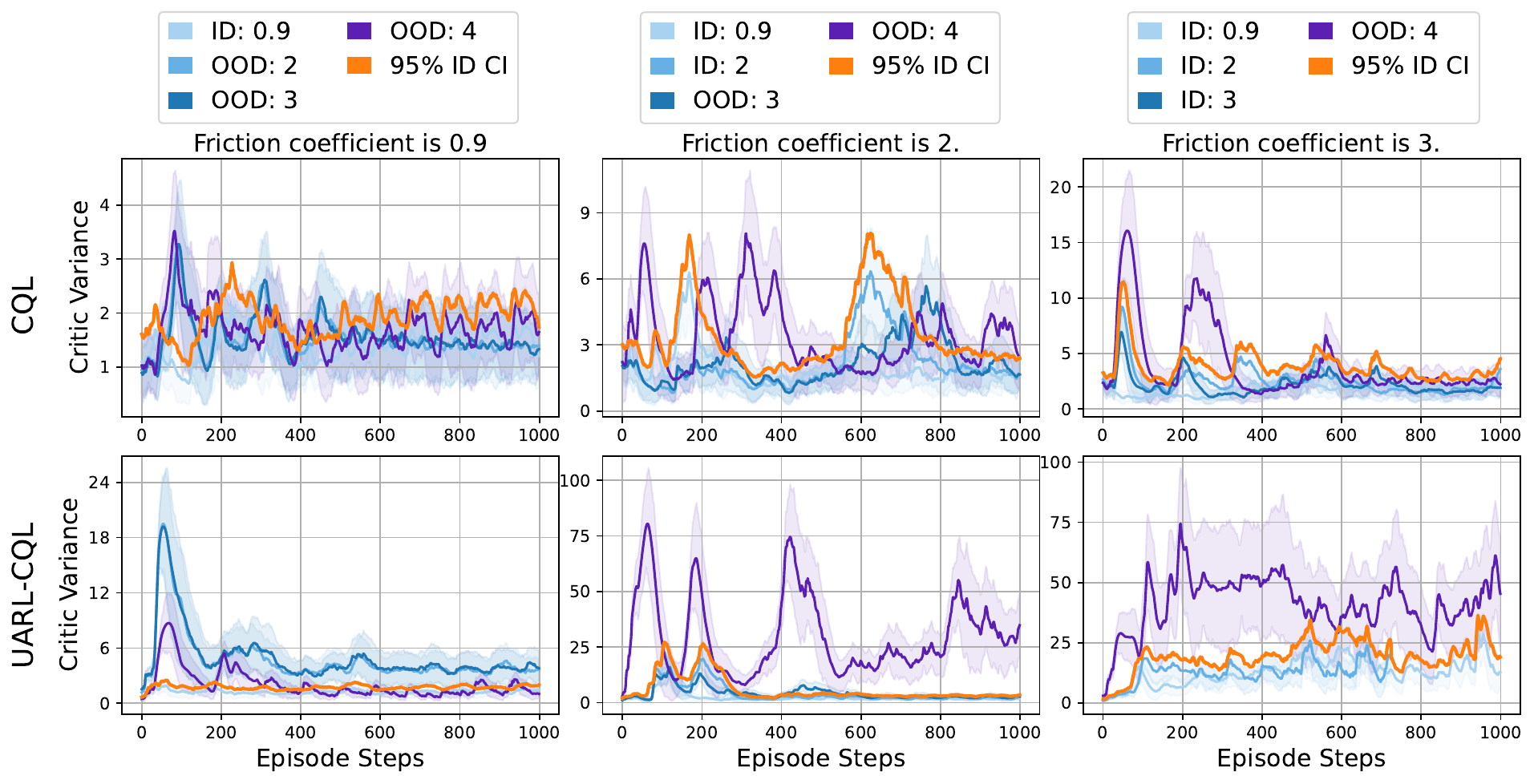}
    \end{subfigure}
    \begin{subfigure}{}
        \centering
        \includegraphics[width=0.8\textwidth]{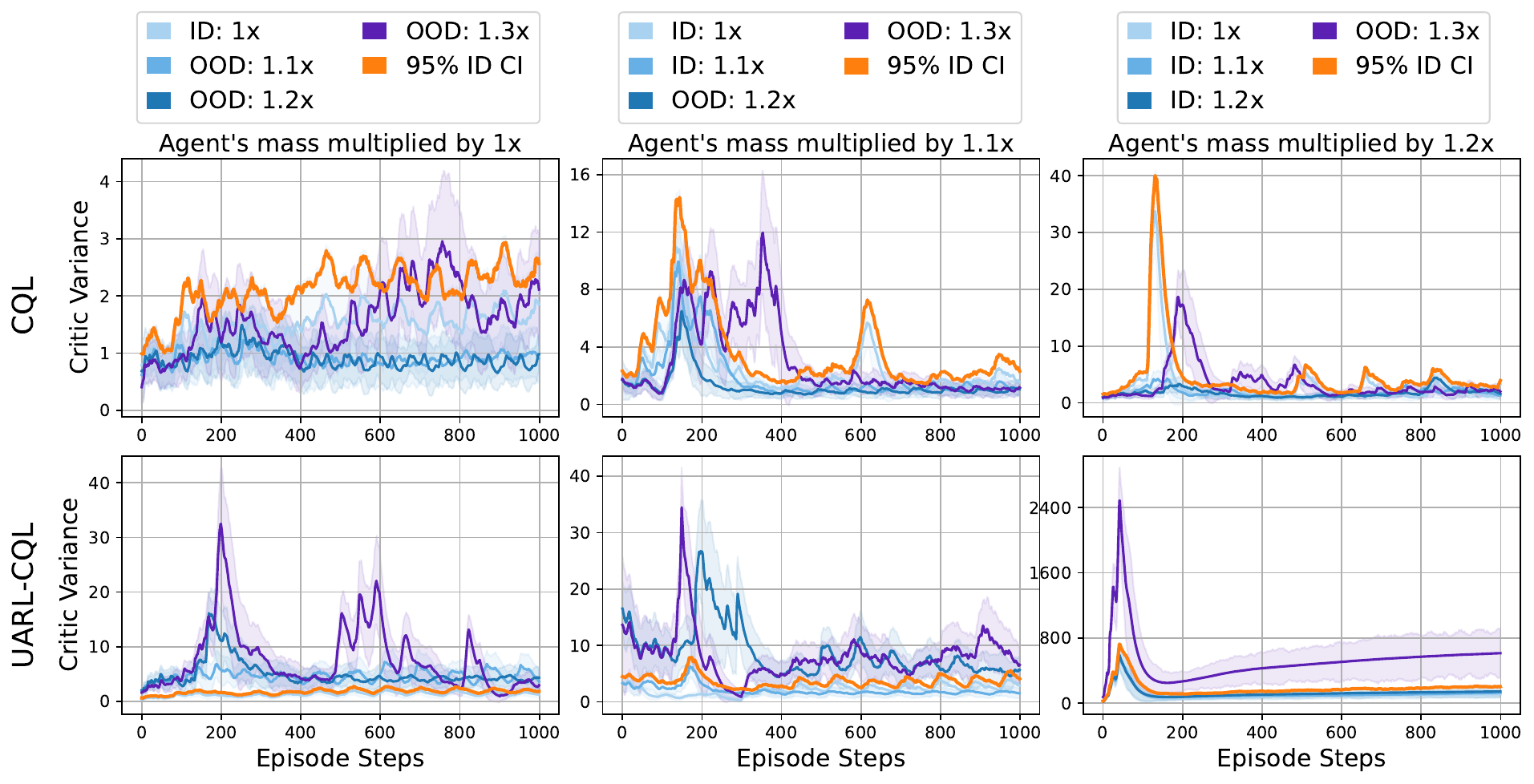}
    \end{subfigure}
    \caption{The OOD detection results for CQL and \ourMethod-CQL initial noise scale (top) and friction coefficient (middle), and agent's mass (bottom) over the \Walker environment.}
    \label{fig:app_uncertainty_walker2d_cql}
\end{figure}

\clearpage

The results show that \ourMethod consistently outperforms baseline methods in detecting OOD data across various environments (\Ant, \HalfCheetah, \Swimmer, and \Walker) using both AWAC-based and CQL-based implementations. A key advantage is its clear separation of ID and OOD samples through critic variance, something often missing in standard AWAC and CQL. \ourMethod adapts its OOD detection threshold as the ID range grows over iterations, maintaining robust performance even in changing environments. Its effectiveness remains consistent across the studied randomized hyperparameters, proving its versatility. The method excels in complex environments like \Ant and \HalfCheetah and shows improved or stable OOD detection over time, suggesting it benefits from expanded training data without losing its detection capability. These findings highlight \ourMethod's strong potential for safe, adaptive deployment in dynamic real-world scenarios.

\clearpage

\subsection{Overall Policy Performance}\label{app:performance}
This subsection presents an extended analysis of \ourMethod's performance across five MuJoCo environments: \Ant, \HalfCheetah, \Hopper, \Swimmer, \Walker. We evaluate the cumulative return during training for each environment, considering the three randomized hyperparameters: initial noise scale, friction coefficient, and agent's mass.

\autoref{fig:perf_it1} shows the performance gains in the offline training phase (\nth{1} iteration of \ourMethod) across various randomized parameters. We observe significant improvements in the \Ant and \HalfCheetah, particularly when randomizing the initial noise scale. For instance, \ourMethod-TD3BC shows a substantial performance advantage over TD3BC in both environments. Similar enhancements are evident when randomizing the friction coefficient, while performance is maintained when altering the agent's mass. Moreover, while EDAC uses an ensemble of $10$ critics \citep{tarasov2022corl}, \ourMethod achieves strong performance with just $2$ critics (default number of critics in baselines), reducing computational overhead. This efficiency makes \ourMethod well-suited for real-world applications with limited resources.

\autoref{fig:app_perf_it1} illustrates the cumulative return achieved by each agent during the offline training phase (\nth{1} iteration) across three environments: \Hopper, \Swimmer, \Walker. The results consistently demonstrate that \ourMethod either improves or maintains the performance of the baseline methods. For instance, in the \Swimmer environment, when the friction coefficient is randomized, \ourMethod significantly enhances the performance of both TD3BC and CQL baselines. Similarly, when randomizing the agent's mass in the same environment, we observe performance improvements across all baseline methods. Notably, there are no instances where the application of \ourMethod leads to a decrease in performance. This robustness is particularly evident in challenging environments like \Hopper and \Walker, where maintaining stability can be difficult. The consistent performance improvements across diverse environments and randomized hyperparameters underscore the versatility and effectiveness of our approach.

\begin{figure}[t]
    \centering
    \includegraphics[width=0.9\linewidth]{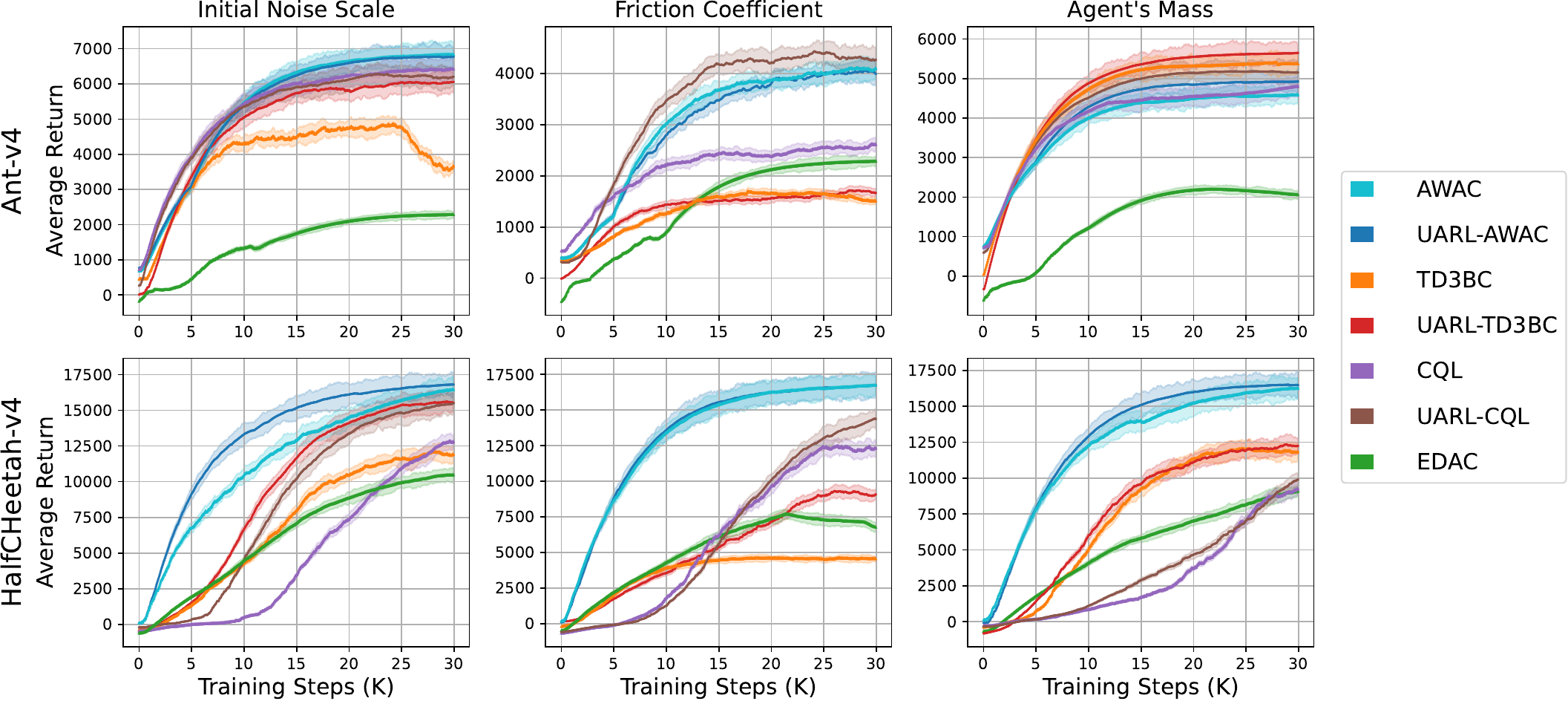}
    \caption{Offline training (\nth{1} iteration) performance, showing average return during training across the randomized parameters. Curves are smoothed for clarity.}
    \vspace{-1.em}
    \label{fig:perf_it1}
\end{figure}

\begin{figure}
    \centering
    \includegraphics[width=0.8\textwidth]{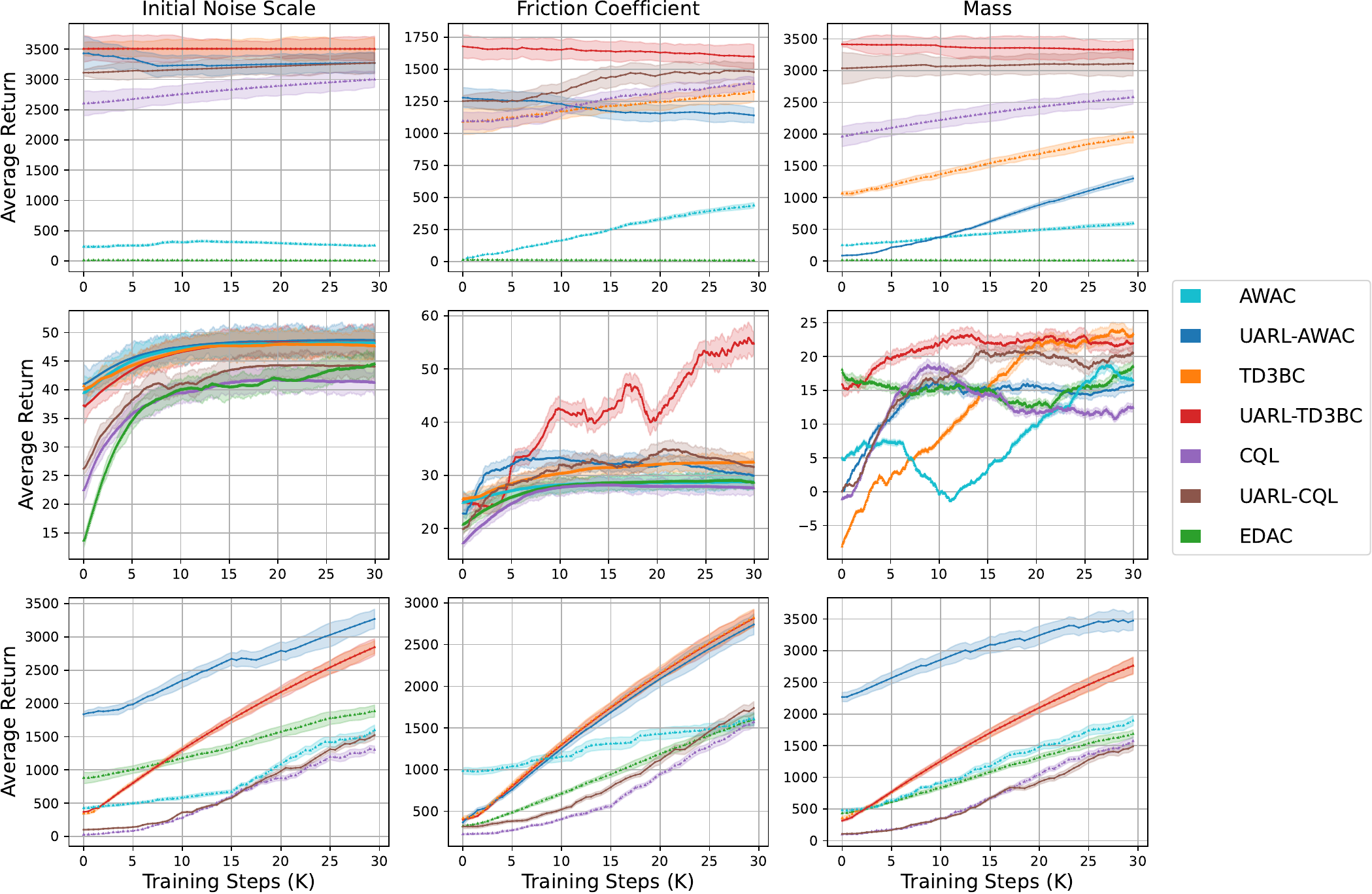}
    \caption{Offline training (\nth{1} iteration) performance in \Hopper (top), \Swimmer (middle), and \Walker (bottom) environments, showing average return during training across three randomized hyperparameters.}
    \label{fig:app_perf_it1}
\end{figure}

The fine-tuning phase (\nth{2} iteration) expands the range of the randomized parameter, collects new data from a more diverse simulation, balances the replay buffer, and fine-tunes the policy as described in \autoref{sec:uarl}. \autoref{fig:perf_it2} showcases the performance during this phase, focusing on AWAC and CQL as our fine-tuning-compatible baselines. The results demonstrate that \ourMethod continues to enhance or maintain performance during the training similar to the \nth{1} phase, and interestingly, it prevents the performance decline observed in CQL across several scenarios, which occurs in \HalfCheetah when the initial noise scale or friction coefficient is altered.

\autoref{fig:app_perf_it2} extends this analysis to the fine-tuning phase (\nth{2} iteration). Here, we observe that \ourMethod continues to demonstrate strong performance, often surpassing the baselines. This is particularly evident in the \Hopper environment, where \ourMethod-AWAC shows significant improvements over standard AWAC across all randomized hyperparameters.

\begin{figure}
    \centering
    \includegraphics[width=0.9\linewidth]{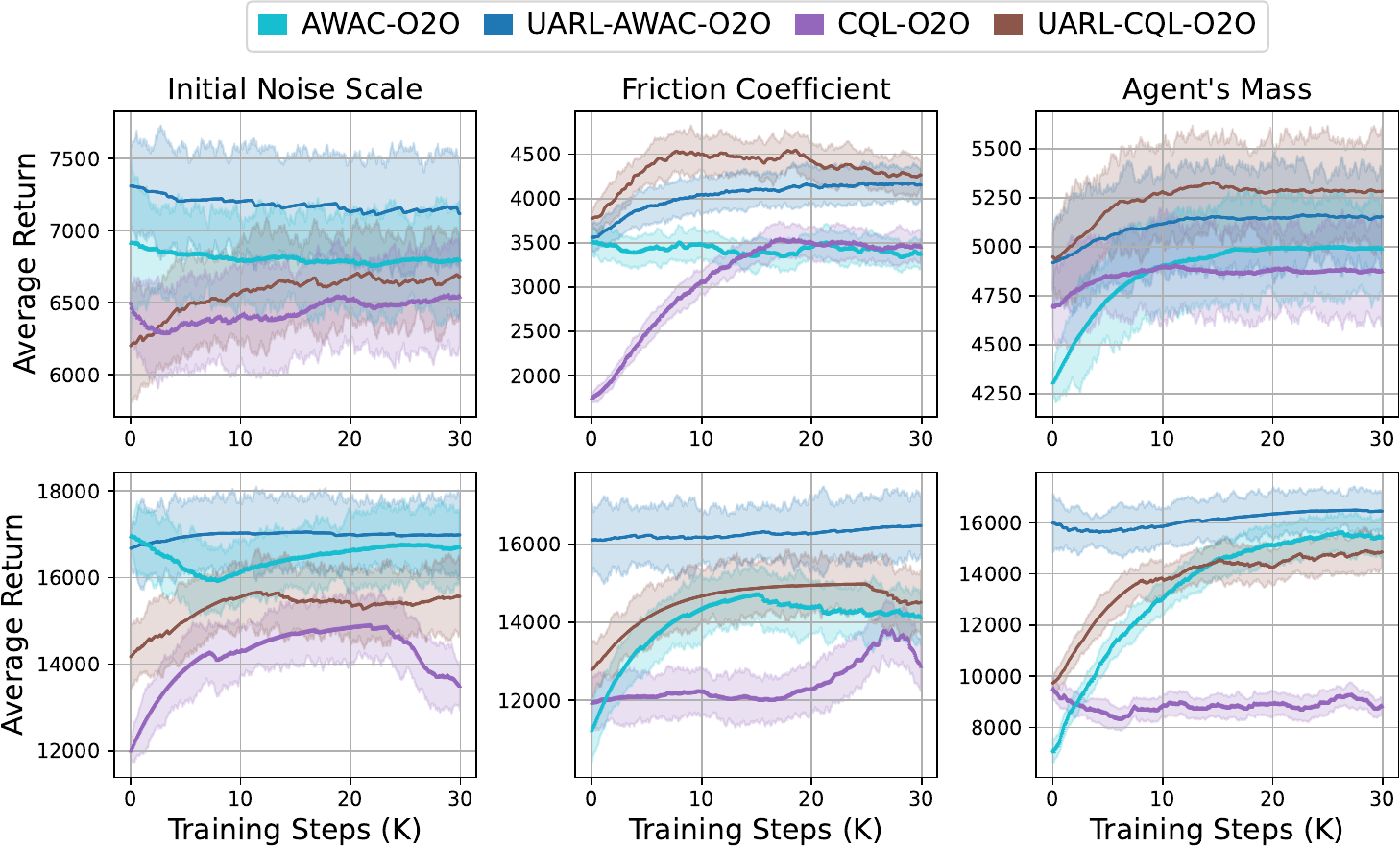}
    \caption{Fine-tuning (\nth{2} iteration) performance in \Ant (top) and \HalfCheetah (bottom) environments, showing average return during fine-tuning across three randomized parameters.}
    \vspace{-1.em}
    \label{fig:perf_it2}
\end{figure}

\begin{figure}
    \centering
    \includegraphics[width=0.8\textwidth]{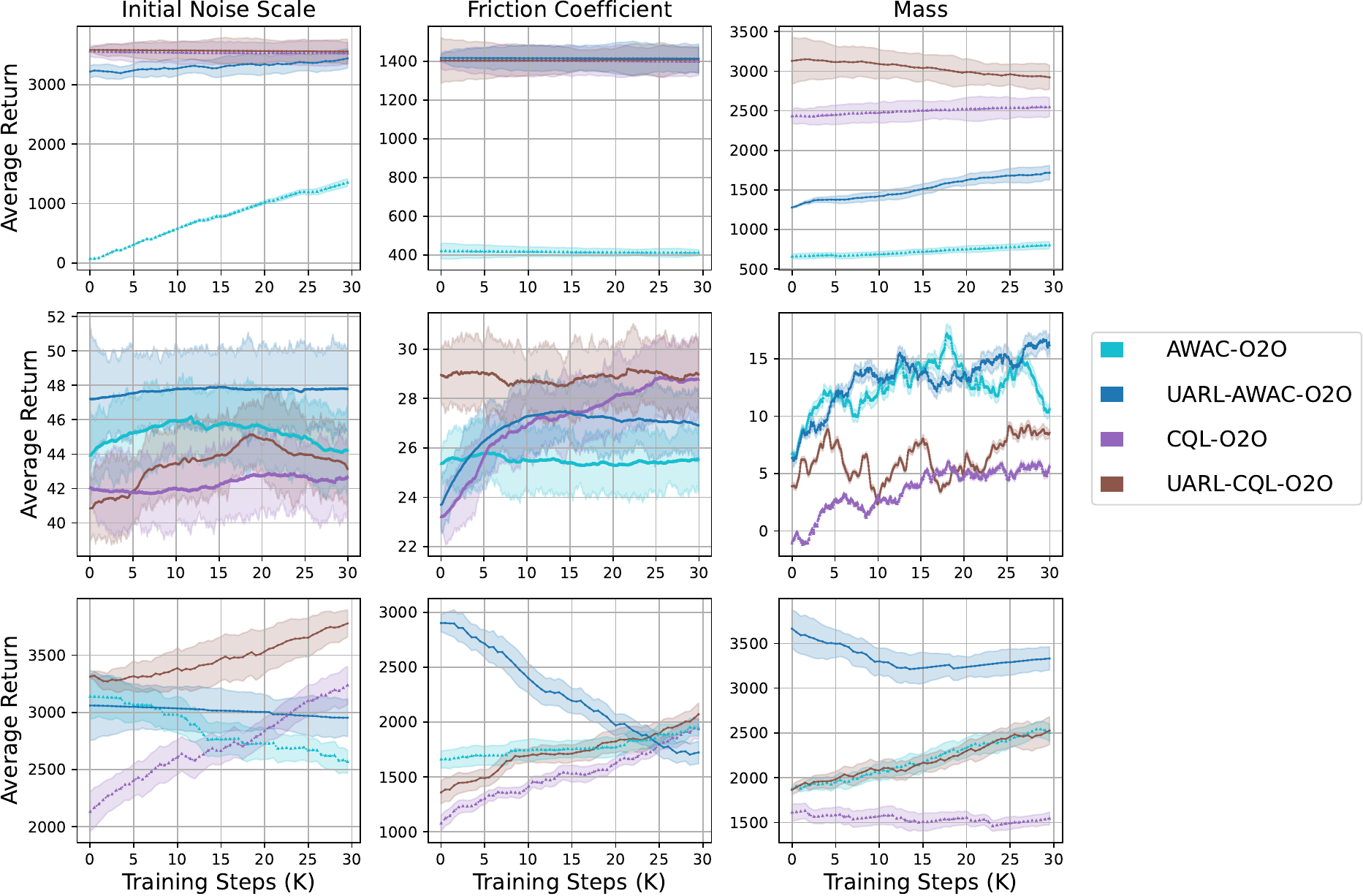}
    \caption{Fine-tuning (\nth{2} iteration) performance in \Hopper (top), \Swimmer (middle), and \Walker (bottom) environments, demonstrating average return during fine-tuning across three randomized hyperparameters.}
    \label{fig:app_perf_it2}
\end{figure}

\autoref{fig:app_perf_it3} presents the same analyses of the fine-tuning phase, but for the \textit{\nth{3} iteration} of \ourMethod. This figure demonstrates the continued effectiveness of our approach over multiple iterations. The results show that \ourMethod maintains its performance advantages and, in many cases, further improves upon the gains observed in the \nth{2} iteration. For instance, in the \Walker environment, \ourMethod-CQL exhibits consistently superior performance across all randomized hyperparameters, showcasing the method's ability to leverage accumulated knowledge effectively. In the \Swimmer environment, we observe that \ourMethod-AWAC continues to outperform the baseline AWAC, particularly when randomizing the agent's mass and initial noise scale. These results underscore the stability and long-term benefits of our approach, indicating that the performance improvements are not transient but persist and potentially amplify over multiple iterations of fine-tuning.

The results in all three figures highlight a key strength of \ourMethod: its ability to enhance the performance of existing offline RL algorithms without compromising their core functionalities. This is achieved through the introduction of diverse critics and the balancing replay buffer mechanism, which together provide more robust policy learning and effective management during fine-tuning of the policy. Furthermore, the consistent performance across different randomized hyperparameters demonstrates \ourMethod's adaptability to various environmental changes. This adaptability is crucial for real-world RL applications, where the ability to handle unexpected variations in the environment is essential for safe and effective deployment.

In summary, these extended results reinforce and expand upon the findings presented in the main paper. They provide strong evidence for the efficacy of \ourMethod across a wide range of locomotion tasks and environmental conditions, highlighting its potential as a robust and versatile approach for offline RL and fine-tuning the policy.

\begin{figure}
    \centering
    \includegraphics[width=0.8\textwidth]{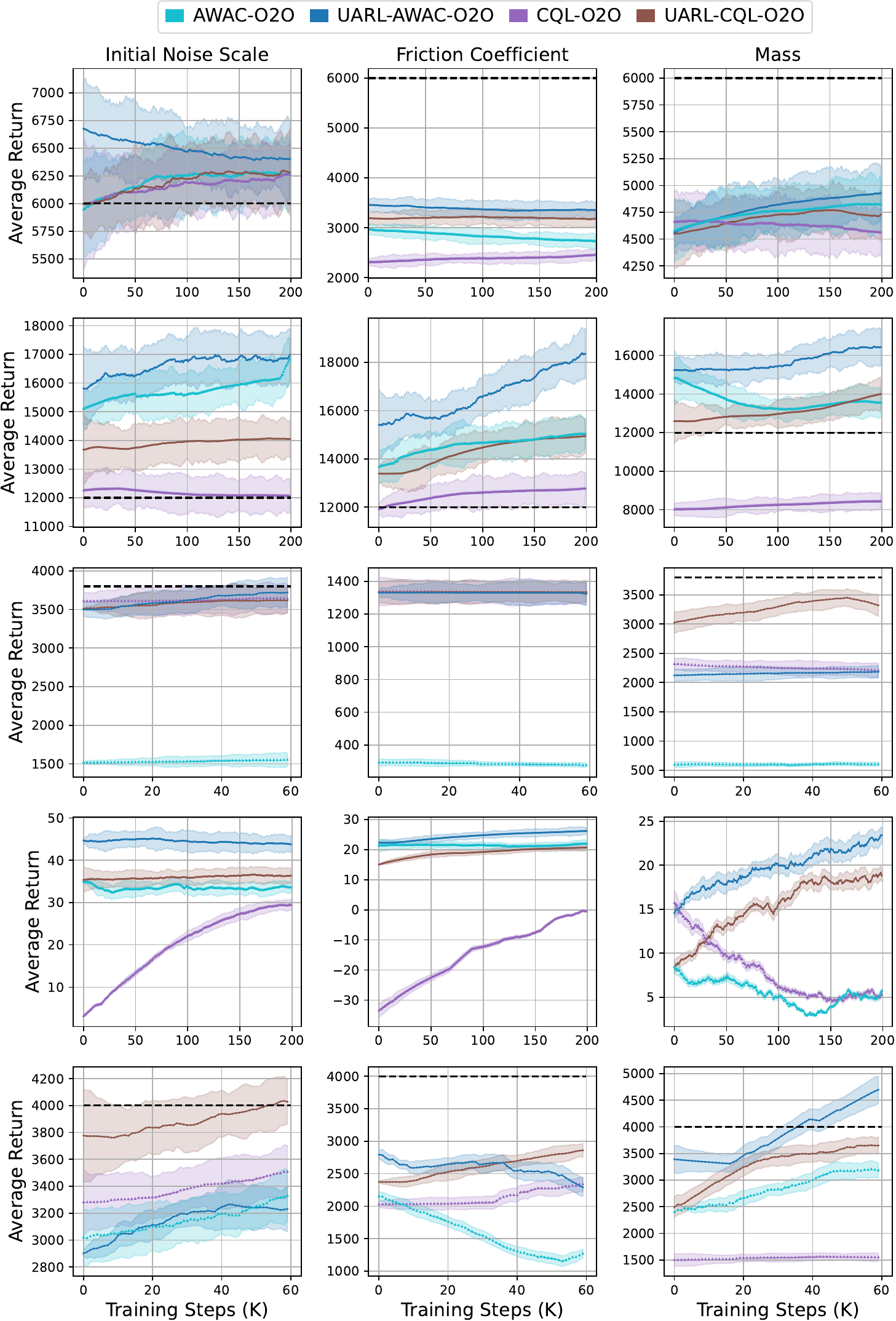}
    \caption{Fine-tuning (\nth{3} iteration) performance in various environments, demonstrating average return during fine-tuning across three randomized hyperparameters. In order, from top to bottom: \Ant, \HalfCheetah, \Hopper, \Swimmer, \Walker.}
    \label{fig:app_perf_it3}
\end{figure}

\clearpage

\subsection{Effect of Balancing Replay Buffer}\label{app:exp_balancing_rb}

Finally, we evaluate the impact of the replay buffer balancing mechanism employed in \ourMethod (\autoref{sec:finetuning}).
\autoref{fig:balancing_rb} shows the effect of removing the balancing replay buffer feature while fine-tuning the policy. In both \Ant and \HalfCheetah, the balancing mechanism consistently outperforms its counterparts across all randomized parameters. \ourMethod with balancing shows faster learning, higher average returns, and less ``unlearning'' (overriding of previously learned behaviors), both in \Ant (especially at the beginning of learning) and \HalfCheetah.
Overall, \ourMethod with balancing demonstrates increasing gains throughout the entire training process.

The results suggest that balancing the replay buffer leads to more stable and efficient learning, particularly crucial during fine-tuning where transitioning from source domain to target domain fine-tuning can be challenging. The consistent improvement across various environmental parameters indicates that the balancing mechanism's benefits are robust to task dynamics variations. In essence, our experiments show that the replay buffer balancing mechanism is a key component in enhancing \ourMethod's performance, accelerating early-stage fine-tuning, and improving overall performance across different environments and task parameters.

\begin{figure}
    \centering
    \includegraphics[width=0.6\linewidth]{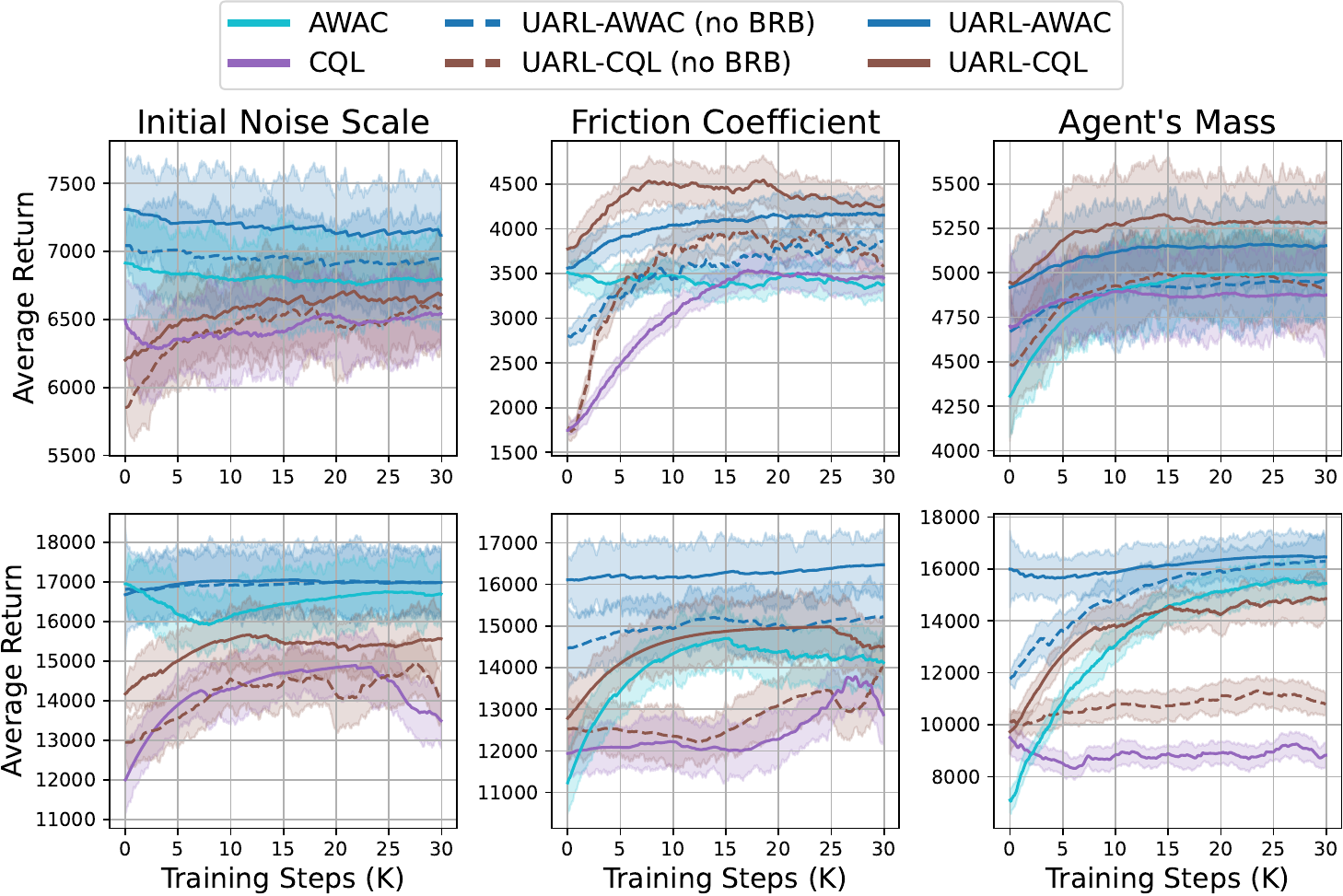}
    \caption{The impact of balancing the replay buffer (BRB) during the initial fine-tuning for all randomized parameters, on \Ant (top) and \HalfCheetah (bottom).}
    \label{fig:balancing_rb}
\end{figure}

\subsection{Sample Efficiency}\label{app:sample_efficiency}
To assess the sample efficiency of \ourMethod, we define convergence based on the number of data points required for the cumulative return to reach an acceptable level, as per the reward thresholds defined by the Gymnasium Environments \citep{towers_gymnasium_2023}. We trained \ourMethod until convergence, allowing it to progressively expand its state space. For a fair comparison, we then trained baselines on this final expanded state space from the outset, ensuring they operate in the same state space that \ourMethod ultimately reached through its iterative process.

\autoref{table:sample_eff} compares sample efficiency between \ourMethod and the baselines across different environments and randomized hyperparameters. The figure shows the difference in the number of samples required for convergence. Positive values indicate that baselines require more samples to converge. The results demonstrate that by starting with a limited state space, \ourMethod generally requires fewer samples to converge while maintaining performance as the state space expands in subsequent iterations. Out of a total of $45$ comparisons across $5$ environments, $3$ baselines, and $3$ randomized hyperparameters, and based on Welch's t-test with a significance level of $0.05$, \ourMethod statistically significantly outperforms the baselines in $35\%$ of cases and achieves non-significant improvements in $49\%$ of cases, resulting in an overall improvement in $84\%$ of comparisons. Notably, a small portion of data from the limited state space (\nth{1} iteration) is often sufficient for \ourMethod to achieve convergence while remaining performant during the fine-tuning process. This finding supports the notion that iteratively incrementing the state space and fine-tuning the policy significantly enhances sample efficiency.

In calculating \ourMethod's sample efficiency, we considered both the nominal and repulsive datasets to ensure a fair comparison with baselines trained on the expanded (nominal) state space. The results in \autoref{table:sample_eff} highlight \ourMethod's ability to learn efficiently in complex environments, potentially reducing the computational resources and time required for training robust policies. This sample efficiency, along with the earlier performance improvements, highlights the practical advantages of our approach in real-world RL applications where data collection is costly or time-consuming.

\begin{table}[h]
\centering
\small
\caption{Sample efficiency comparison: Difference in means and standard deviations (mean $\pm$ std), in thousands of samples needed for convergence (baseline vs. corresponding \ourMethod method: \colorbox{sbBlue025}{AWAC-based}, \colorbox{sbPurple025}{CQL-based}, \colorbox{sbOrange025}{TD3BC-based}). Positive values indicate \ourMethod requires fewer samples.}
\begin{tabular}{cccccc}
\toprule
Hyperparameter & \ant & \halfCheetah & \hopper & \swimmer & \walker \\
\midrule
\multirow{3}{*}{\makecell{Initial\\Noise\\Scale}} 
 & \cellcolor{sbBlue025}$1208 {\scriptstyle {\pm 859}}$ & \cellcolor{sbBlue025}$8745 {\scriptstyle {\pm 1499}}$ & \cellcolor{sbBlue025}$5062 {\scriptstyle {\pm 1986}}$ & \cellcolor{sbBlue025}$5700 {\scriptstyle {\pm 949}}$ & \cellcolor{sbBlue025}$1680 {\scriptstyle {\pm 2163}}$ \\
  & \cellcolor{sbPurple025}$501 {\scriptstyle {\pm 1009}}$ & \cellcolor{sbPurple025}$27112{\scriptstyle {\pm 1104}}$ & \cellcolor{sbPurple025}$3616 {\scriptstyle {\pm 1381}}$ & \cellcolor{sbPurple025}$-228{\scriptstyle {\pm 899}}$ & \cellcolor{sbPurple025}$1579{\scriptstyle {\pm 2440}}$ \\
  & \cellcolor{sbOrange025}$5661 {\scriptstyle {\pm 905}}$ & \cellcolor{sbOrange025}$13799 {\scriptstyle {\pm 939}}$ & \cellcolor{sbOrange025}$10993 {\scriptstyle {\pm 1843}}$ & \cellcolor{sbOrange025}$-114 {\scriptstyle {\pm 837}}$ & \cellcolor{sbOrange025}$-407 {\scriptstyle {\pm 1488}}$ \\
\midrule
\multirow{3}{*}{\makecell{Friction\\Coefficient}} 
 & \cellcolor{sbBlue025}$3857 {\scriptstyle {\pm 2278}}$ & \cellcolor{sbBlue025}$3765 {\scriptstyle {\pm 1419}}$ & \cellcolor{sbBlue025}$2202{\scriptstyle {\pm 1683}}$ & \cellcolor{sbBlue025}$-43{\scriptstyle {\pm 2094}}$ & \cellcolor{sbBlue025}$2357{\scriptstyle {\pm 1711}}$ \\
 & \cellcolor{sbPurple025}$3578 {\scriptstyle {\pm 2083}}$ & \cellcolor{sbPurple025}$17343{\scriptstyle {\pm 1940}}$ & \cellcolor{sbPurple025}$4038 {\scriptstyle {\pm 1897}}$ & \cellcolor{sbPurple025}$-376 {\scriptstyle {\pm 1811}}$ & \cellcolor{sbPurple025}$-930 {\scriptstyle {\pm 2560}}$ \\
  & \cellcolor{sbOrange025}$9401 {\scriptstyle {\pm 1589}}$ & \cellcolor{sbOrange025}$13043 {\scriptstyle {\pm 978}}$ & \cellcolor{sbOrange025}$14991 {\scriptstyle {\pm 1486}}$ & \cellcolor{sbOrange025}$14 {\scriptstyle {\pm 1736}}$ & \cellcolor{sbOrange025}$-437 {\scriptstyle {\pm 2102}}$ \\
\midrule
\multirow{3}{*}{\makecell{Agent's\\Mass}} 
 & \cellcolor{sbBlue025}$5902{\scriptstyle {\pm 1674}}$ & \cellcolor{sbBlue025}$10779{\scriptstyle {\pm 1243}}$ & \cellcolor{sbBlue025}$8138 {\scriptstyle {\pm 1414}}$ & \cellcolor{sbBlue025}$7625{\scriptstyle {\pm 2090}}$ & \cellcolor{sbBlue025}$282{\scriptstyle {\pm 2217}}$ \\
 & \cellcolor{sbPurple025}$8262 {\scriptstyle {\pm 1540}}$ & \cellcolor{sbPurple025}$1210 {\scriptstyle {\pm 2220}}$& \cellcolor{sbPurple025}$2531 {\scriptstyle {\pm 1941}}$ & \cellcolor{sbPurple025}$-8631 {\scriptstyle {\pm 2280}}$ & \cellcolor{sbPurple025}$-2646 {\scriptstyle {\pm 2133}}$ \\ 
 & \cellcolor{sbOrange025}$-1225 {\scriptstyle {\pm 1120}}$ & \cellcolor{sbOrange025}$13309 {\scriptstyle {\pm 2915}}$ & \cellcolor{sbOrange025}$2092 {\scriptstyle {\pm 1387}}$ & \cellcolor{sbOrange025}$ 3343 {\scriptstyle {\pm 1762}}$ & \cellcolor{sbOrange025}$ 1049 {\scriptstyle {\pm 1549}}$ \\
\bottomrule
\end{tabular}
\label{table:sample_eff}
\end{table}

\clearpage

\subsection{Policy Performance Comparison with Off-Dynamics RL}\label{app:off_dyn}

This section evaluates the performance of \ourMethod against off-dynamics RL methods, focusing on their ability to handle distribution shifts and detect OOD events. We compare \ourMethod with three prominent baselines: H2O \citep{niu2022trust}, VGDF \citep{xu2023crossdomain}, and a robust algorithm which leverages offline data, RLPD \citep{ball2023efficient}. These methods were selected because they share some similarities with \ourMethod in terms of leveraging ensemble-based critics while trying to learn a robust policy. However, they differ fundamentally in their assumptions about target domain accessibility and data usage, as discussed in \aref{app:related}.

\begin{figure}[t]
    \centering
    \includegraphics[width=0.4\textwidth]{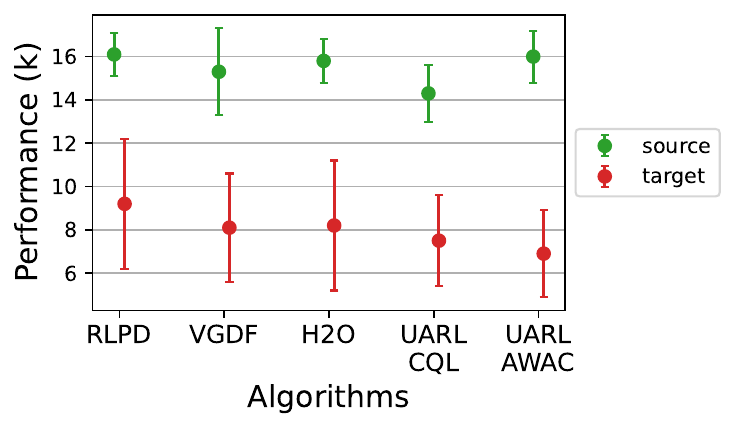}
    \caption{Performance in \HalfCheetah under domain shift. While all methods degrade when transitioning from source (friction coefficient=$0.5$) to target domain (friction coefficient=$0.8$), \ourMethod maintains competitive performance despite training exclusively on source domain data.}
    \label{fig:off_dyn_perf}
\end{figure}

\subsubsection{Experimental Setup}
We conduct experiments in the \HalfCheetah environment, where we simulate a sudden and significant change in the friction coefficient during an episode. This setup allows us to evaluate how well each method adapts to unexpected distribution shifts and detects OOD events. Specifically, the source domain (nominal environment) has a friction coefficient of $0.5$, and we collect an offline dataset $\dataset$ consisting of $999$ rollouts. The repulsive dataset $\dataset'$ is collected with a friction coefficient of $0.6$, similarly for $999$ rollouts. The target domain (real-world environment) has a friction coefficient of $0.8$, and we assume access to a limited dataset $\dataset_t$ of $99$ rollouts. 

\ourMethod is trained using only $\dataset$ and $\dataset'$, without any access to $\dataset_t$ during training. In contrast, the baselines (H2O, VGDF, and RLPD) are trained using all available data, including $\dataset_t$, as they require target domain data for policy refinement. This distinction highlights a key advantage of \ourMethod: it avoids the need for extensive target domain data, making it more practical for safety-critical applications where such data is scarce or risky to collect.

\subsubsection{Performance Analysis}
\autoref{fig:off_dyn_perf} compares the performance of all methods in the source and target domains. As expected, all methods experience a performance drop when transitioning to the target domain due to the significant change in the friction coefficient (see \aref{app:random_params} for details). However, the baselines exhibit a less severe degradation because they are explicitly trained on target domain data ($\dataset_t$). For instance, H2O fills half of its replay buffer with target domain data during training: \href{https://github.com/t6-thu/H2O/blob/main/SimpleSAC/sim2real_sac_main.py#L185-L186}{H2O's official implementation}. RLPD similarly relies on target domain data for balancing offline and online learning.

Despite not using $\dataset_t$ during training, \ourMethod achieves competitive performance, demonstrating its ability to generalize to the target domain through uncertainty-aware adaptation and iterative fine-tuning. This result underscores the efficiency of \ourMethod in leveraging limited data and its suitability for real-world applications where target domain data is scarce.

\subsubsection{OOD Detection Capabilities}
A critical advantage of \ourMethod is its ability to detect OOD events without direct interaction with the target domain. To evaluate this capability, we simulate an OOD event by abruptly changing the friction coefficient from $0.5$ to $0.8$ at a random time step between $600$ and $800$ during evaluation. We then measure the critic variance for each method to assess its ability to detect this change.

\autoref{fig:off_dyn_detect} shows the critic variance across episodes. Only \ourMethod-based approaches exhibit a sharp increase in critic variance when the friction coefficient changes, effectively detecting the OOD event. In contrast, the baselines fail to reliably detect the change, as their critic variances remain within the range of ID parts of the episode. This result highlights \ourMethod's superior OOD detection capabilities, which are crucial for ensuring safety in real-world deployments.

\begin{figure}[t]
    \centering
    \includegraphics[width=\linewidth]{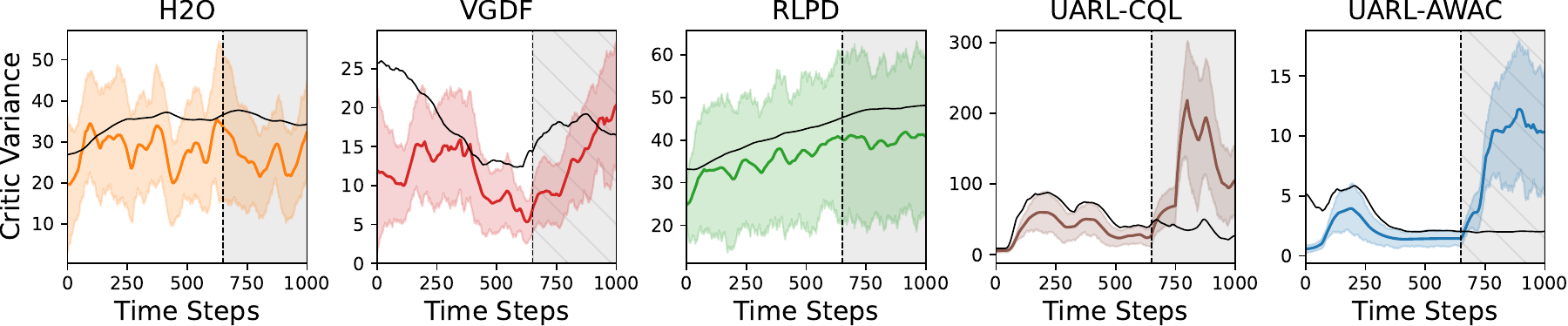}
    \caption{OOD detection in \HalfCheetah. When friction coefficient changes abruptly at step $650$ and stays (grey area), only \ourMethod reliably detects the distribution shift through critic variance spikes that exceed the $95\%$ ID CI (black lines). Baselines fail to recognize the OOD dynamics.}
    \label{fig:off_dyn_detect}
\end{figure}

The experiments demonstrate that while off-dynamics RL methods like H2O, VGDF, and RLPD can achieve robust performance within their training domains, they rely heavily on target domain data and lack explicit mechanisms for OOD detection. In contrast, \ourMethod achieves competitive performance without requiring target domain data during training and excels at detecting OOD events. These advantages make \ourMethod particularly well-suited for safety-critical applications, where detecting and responding to novel situations is paramount.

Furthermore, the reliance of off-dynamics RL methods on extensive target domain data limits their practicality in real-world scenarios. For example, H2O and RLPD require half of their replay buffers to be filled with target domain data, which may not be feasible in settings where such data is limited or expensive to collect. \ourMethod addresses this limitation by using target domain data only for validation, making it more data-efficient and scalable.

In summary, these experiments highlight the unique strengths of \ourMethod in handling distribution shifts and detecting OOD events, while also demonstrating its practical advantages over off-dynamics RL methods in terms of data efficiency and safety.

\clearpage

\subsection{\ourMethod Hyperparameter Sensitivity Analysis}\label{app:hyperparams_ablation}
This subsection examines the sensitivity of \ourMethod to its key hyperparameters: the diversity coefficient $\lambda$ and the diversity scale $\delta$. These hyperparameters balance the standard RL objective with the goal of promoting diversity among critics. We evaluate various combinations of $\lambda \in \{ 1\%, 5\%, 10\%, 15\%, 20\% \}$ and $\delta \in \{ 10^{-3}, 10^{-2}, 10^{-1} \}$ across two settings: \Ant, focusing on the agent's mass hyperparameter, and \HalfCheetah, focusing on the initial noise scale. We tested these combinations with three baseline algorithms (AWAC, CQL, and TD3BC), using $5$ random seeds for each configuration, resulting in a total of $450$ experimental runs.

The hyperparameter selection process is grounded in a systematic exploration of the trade-off between diversity regularization and policy optimization. By varying the diversity coefficient $\lambda$ and scale $\delta$, we aim to understand how these parameters influence the learning dynamics of \ourMethod. The chosen ranges reflect a careful consideration of the potential impact of diversity-promoting mechanisms on the RL objective.

\autoref{fig:app_hyperparam_ant} and \autoref{fig:app_hyperparam_cheetah} demonstrate the influence of varying $\lambda$ and $\delta$ values on agent performance during training, on \Ant's mass and \HalfCheetah's initial noise scale, respectively. While our chosen \textit{default} configuration performs well, the results indicate the potential for further performance enhancement through careful hyperparameter tuning. The configurations with the most extreme values, specifically $\lambda \in \{\textcolor[HTML]{9c8e11}{15\%}, \textcolor[HTML]{9847b5}{20\%} \}$ and $\delta = 10^{-1}$, tend to have the most detrimental effect on performance. At higher diversity coefficients, the introduced regularization becomes increasingly aggressive, potentially introducing noise that disrupts the learning process. This suggests an inherent trade-off where excessive diversity constraints can impede the algorithm's ability to converge to optimal behavior. This suggests that moderate values for these hyperparameters generally yield better results, with room for fine-tuning within this range to optimize performance for specific environments or baseline algorithms. The observed performance degradation under extreme configurations can be attributed to an amplified diversity loss, which overshadows the RL objective, thereby reducing the \ourMethod's capacity to effectively optimize its policy.

This analysis highlights the robustness of \ourMethod across a range of hyperparameter values while also indicating opportunities for optimization in specific environments or with particular baseline algorithms. The varied performance across different hyperparameter combinations suggests that fine-tuning these hyperparameters could lead to enhanced results in certain scenarios, though the default configuration provides a strong baseline performance across the tested environments and algorithms.

\begin{figure}[t]
    \centering
    \includegraphics[width=\textwidth]{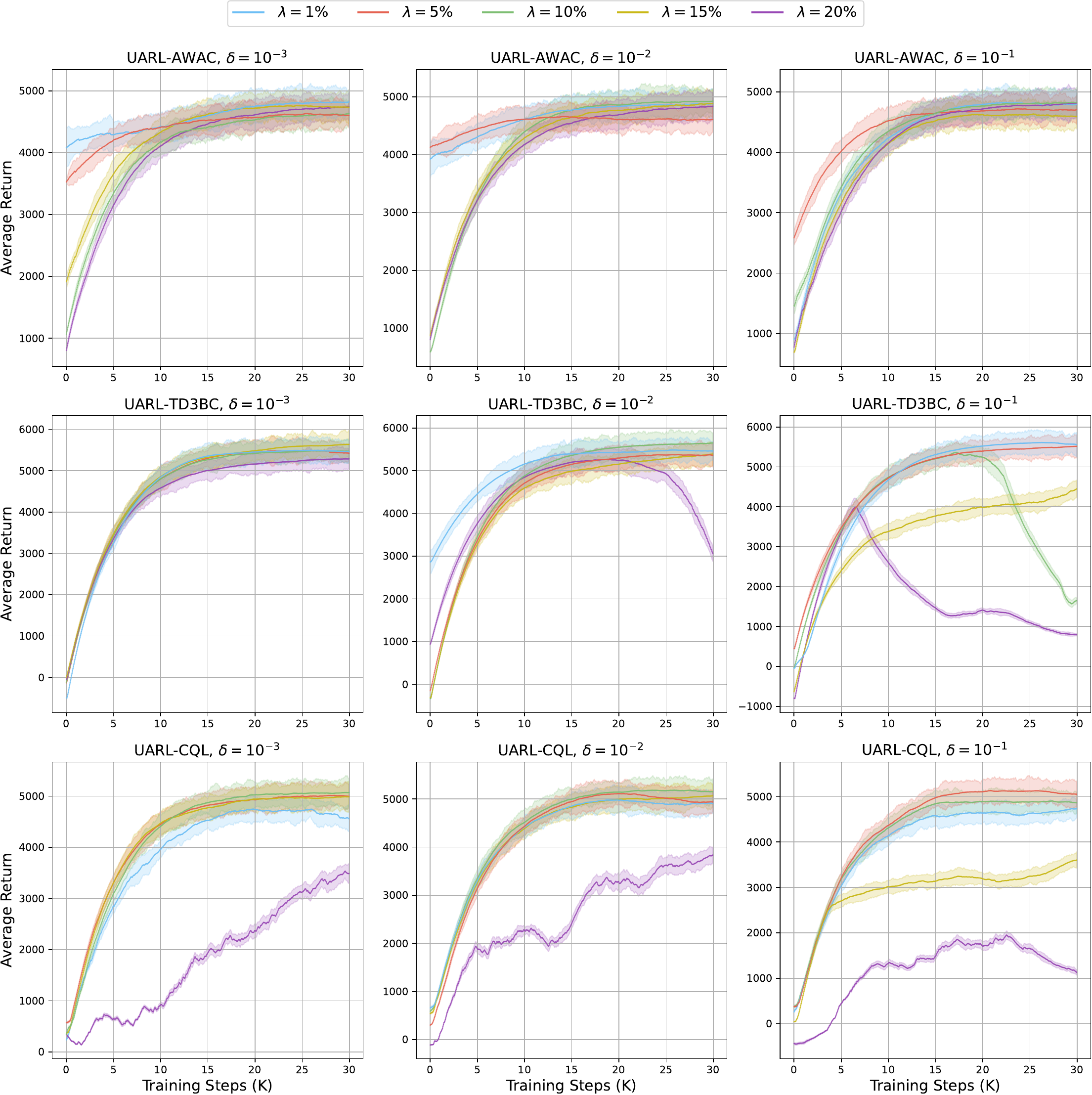}
    \caption{
    Performance sensitivity to hyperparameters $\lambda$ and $\delta$ during training in the \Ant environment, with a focus on the agent's mass hyperparameter. Each row represents a baseline algorithm (AWAC, TD3BC, CQL), while each column corresponds to a fixed value of $\delta$ as indicated. The \textit{default} configuration used in the main paper ($\lambda = 10\%$ and $\delta = 10^{-2}$) is depicted in the middle column, highlighted by the \textcolor[HTML]{7dbe72}{green line}.}
    \label{fig:app_hyperparam_ant}
\end{figure}

\begin{figure}[t]
    \centering
    \includegraphics[width=\textwidth]{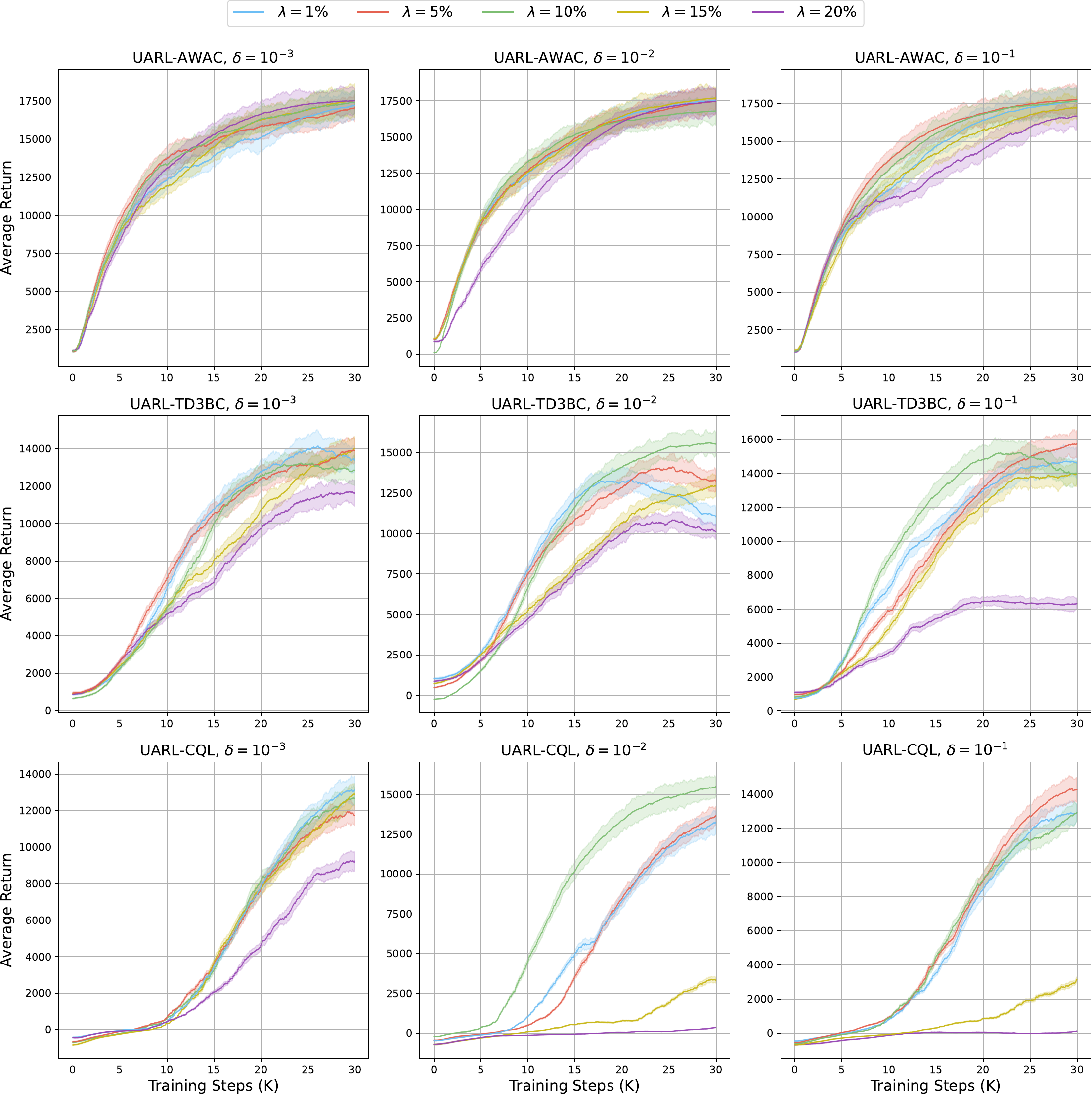}
    \caption{
    Performance sensitivity to hyperparameters $\lambda$ and $\delta$ during training in the \HalfCheetah environment, with a focus on the initial noise scale hyperparameter. Each row represents a baseline algorithm (AWAC, TD3BC, CQL), while each column corresponds to a fixed value of $\delta$ as indicated. The \textit{default} configuration used in the main paper ($\lambda = 10\%$ and $\delta = 10^{-2}$) is depicted in the middle column, highlighted by the \textcolor[HTML]{7dbe72}{green line}.}
    \label{fig:app_hyperparam_cheetah}
\end{figure}

\clearpage

\subsection{Additional ANYmal Details}
\label{app:anymal_exps}

Thus far we have evaluated \ourMethod on MuJoCo tasks with low-dimensional legs and modest dynamics discrepancies. To demonstrate scalability to modern field robots, we now consider the
ANYmal-D quadruped \citep{7758092}, a 12-DoF robot that must
remain safe under hardware constraints such as motor current,
temperature, and large exogenous pushes.
We employ NVIDIA Isaac Gym \citep{rudin2022learning} to run $4096$ parallel instances of ANYmal (flat terrain, no obstacles) on a single NVIDIA RTX 4090 GPU (\autoref{fig:legged_gym}).  Observations are the standard proprioceptive state comprising base linear and angular velocities, gravity vector in the robot frame, previous command, and joint positions and velocities used in \citep{rudin2022learning}; the policy outputs desired joint positions that are converted to torques by a learned actuator network \citep{hwangbo2019learning}. The reward combines policy command tracking, orientation stabilisation, energy penalties and joint smoothness, exactly the default in \citep{rudin2022learning}.

This expanded experimental setup highlights how \ourMethod seamlessly extends to a new domain and real-world robotics, guiding systematic domain randomization and uncertainty-aware adaptation for robust quadruped locomotion.

\begin{figure}[t]
    \centering
    \includegraphics[width=0.7\textwidth]{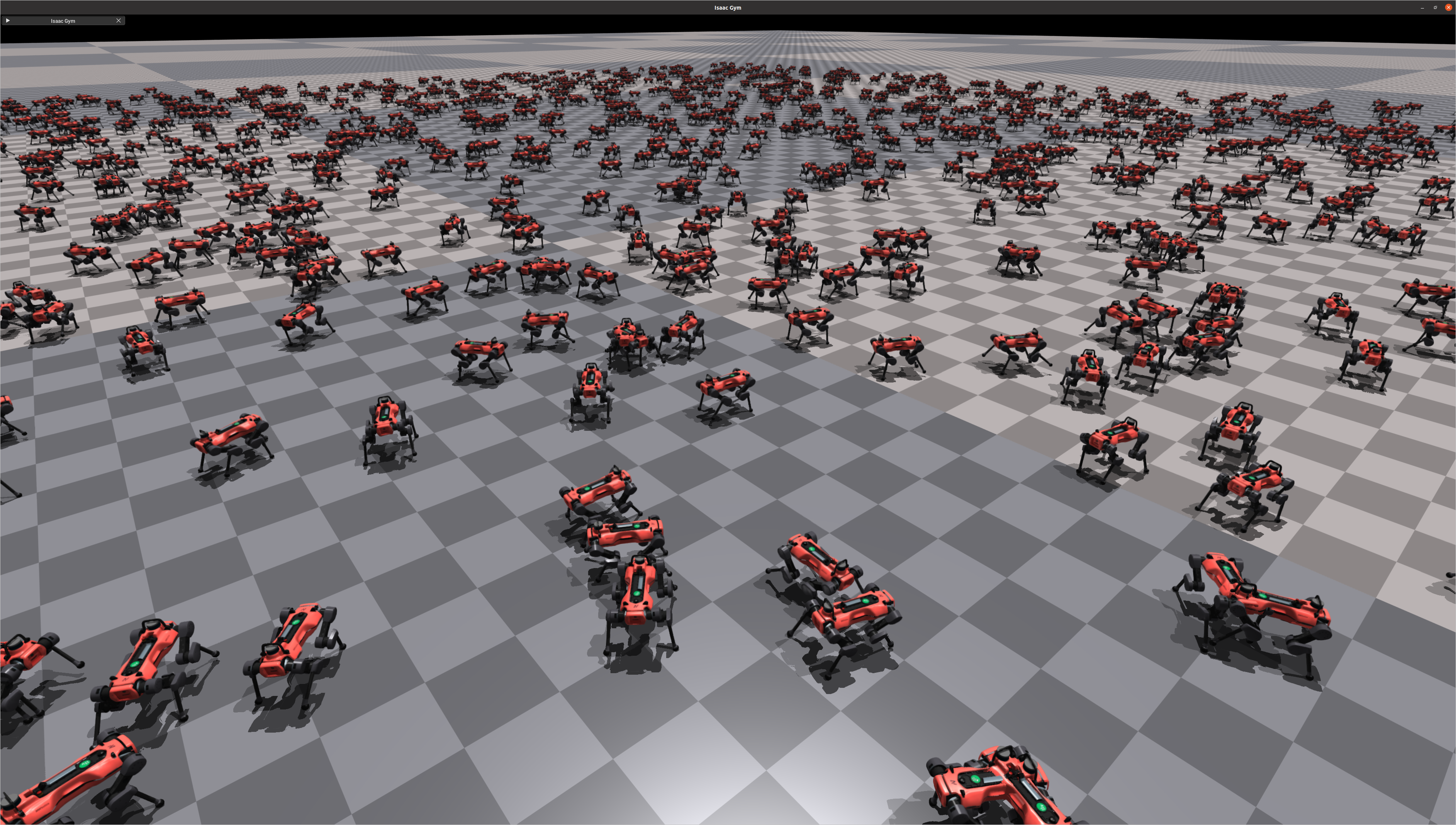}
    \caption{
    Isaac Gym running $4096$ parallel ANYmal simulations.
    Each agent receives randomized dynamics for domain randomization
    curriculum, enabling massive state transitions' on one GPU.
    }
    \label{fig:legged_gym}
\end{figure}

\subsubsection{Progressive Domain Randomization}\label{app:anymal_dr}
We randomize two key parameters, one at a time, and let critic variance decide when to expand the interval:
\begin{itemize}
\item Ground friction initial range $[1.5, 2.0]$ $\rightarrow$ expand by $\Delta=-0.5$ each side in each iteration.
\item Added base mass initial range $[-1,1]\mathrm{kg}$ $\rightarrow$ expand symmetrically by
  $\pm2\mathrm{kg}$ in each iteration.
\end{itemize}

For each parameter we treat the current interval as the \emph{nominal} domain and the first ring outside it as the \emph{repulsive} domain, exactly matching the notation $\dataset, \dataset'$ in \autoref{sec:uarl} (source domain).
When the ensemble variance on the target domain dataset $\dataset_t$ (target domain) exceeds the deployment threshold $\tau$ (\autoref{sec:exp_uncertainty}), we enlarge the nominal interval by one $\Delta$ step and resume training.

\subsubsection{Algorithmic Details}\label{app:anymal_algo}
Because Isaac Gym generates on-policy roll-outs, we swap our offline back-bones for PPO \citep{schulman2017proximal} (RSL-RL implementation \citep{rudin2022learning}) and make the minimal changes needed to inject \ourMethod:
\begin{itemize}
\item Two independent critics share the actor policy; we
  clip the TD target by minimum of Q values to curb over-estimation, and make our method applicable.
\item We append the diversity regulariser   $\mathcal L_{\text{div}}^{\text{RL}}$ (\eqref{eq:denn_rl}) to PPO's value loss (coefficient $\lambda$ chosen exactly as in MuJoCo: keep it $\approx10\%$ of total).
\item PPO's entropy bonus and clipping hyperparameters are left  unchanged; replay-buffer weighting is \emph{not} used because the   data are on-policy.
\end{itemize}

We train with $4096$ parallel workers for $10^3$ training steps. Except mentioned otherwise, the rest of the experiment setup is the same as descibed in \autoref{sec:exp}, i.e. using the same number of seeds for training and reporting $95\%$ CI.

To guide the evaluation, we proceed in four steps. First, we confirm that \ourMethod does not degrade raw locomotion performance relative to PPO (\aref{app:anymal_perf}). Next, we demonstrate our method's online OOD detection capability in simulation (\aref{app:anymal_ood}). We then discuss the safety implications of ignoring uncertainty in real-world deployments (\aref{app:anymal_safety}). 

\subsubsection{Locomotion Performance}\label{app:anymal_perf}

\autoref{fig:anymal_performance} compares vanilla PPO with \ourMethod-PPO on the nominal domain during the full training and fine-tuning phases on $100$ evaluation episodes.
Both agents reach the cumulative reward threshold of $20$ and plateau at virtually identical performance; no degradation is observed when the diversity loss is added.  This confirms that \ourMethod's uncertainty term does not hamper raw locomotion quality.

\begin{figure}[t]
  \centering
  \includegraphics[width=\textwidth]{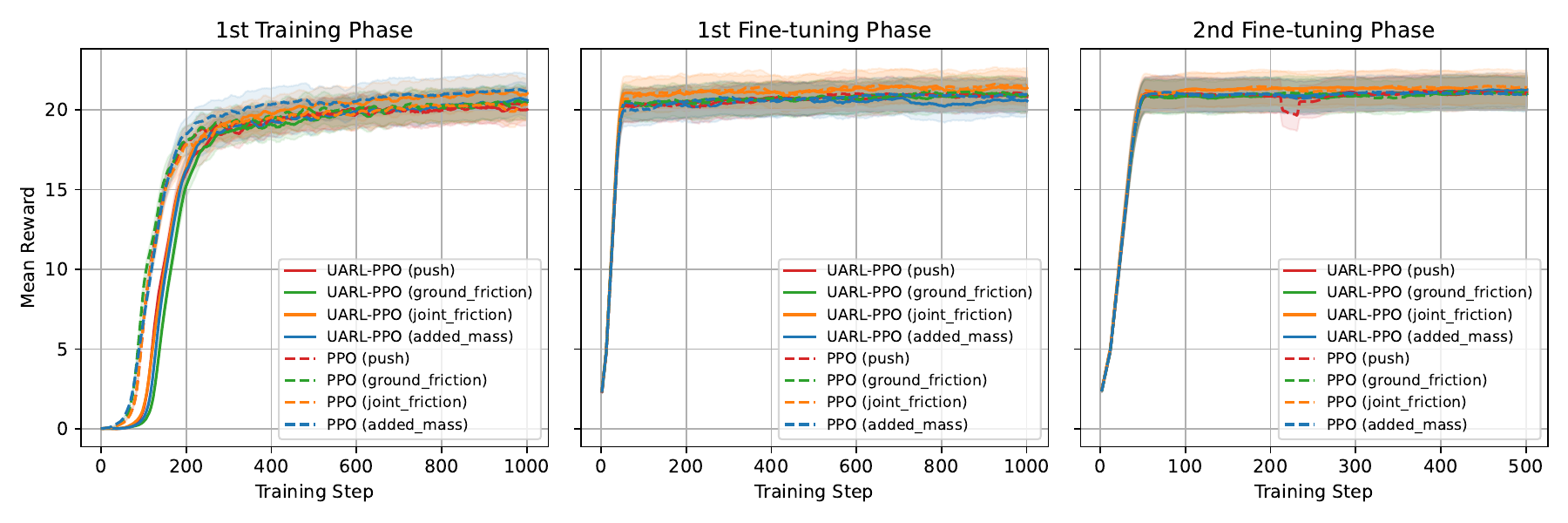}
  \caption{\small
    Mean episode return (smoothed) for PPO and \ourMethod-PPO on ANYmal, across different randomized parameters and iterations. Shaded bands are $95\%$ CI over five seeds.}
  \label{fig:anymal_performance}
\end{figure}

\subsubsection{Online OOD Detection}\label{app:anymal_ood}
We now stress the trained agents in two unseen conditions, each pushing \emph{one} domain parameter outside the final nominal range, following the protocol of \autoref{sec:exp_uncertainty}.
A timestep is flagged as OOD when the critic variance $\sigma^2(s_t,a_t)$ exceeds $\tau$, the $95\%$ envelope of ID variance. \autoref{fig:ood_anymal_mass} demonstrates the OOD detection performance concerning the mass added to the base of the robot, and \autoref{fig:ood_anymal_ground_friction} demonstrates the OOD detection performance concerning the reduced ground friction (slippery floor). Both results tell the same story: \ourMethod-PPO produces a clean separation between \colorbox{sbBlue025}{ID} and \colorbox{sbRed025}{OOD} curves, while baseline PPO exhibits heavy overlap and thus cannot be relied upon for safety monitoring. These results confirm that our uncertainty-aware diversity regularizer provides a robust, online signal for OOD monitoring in quadruped locomotion.

\begin{figure}[t]
    \centering
    \includegraphics[width=\textwidth]{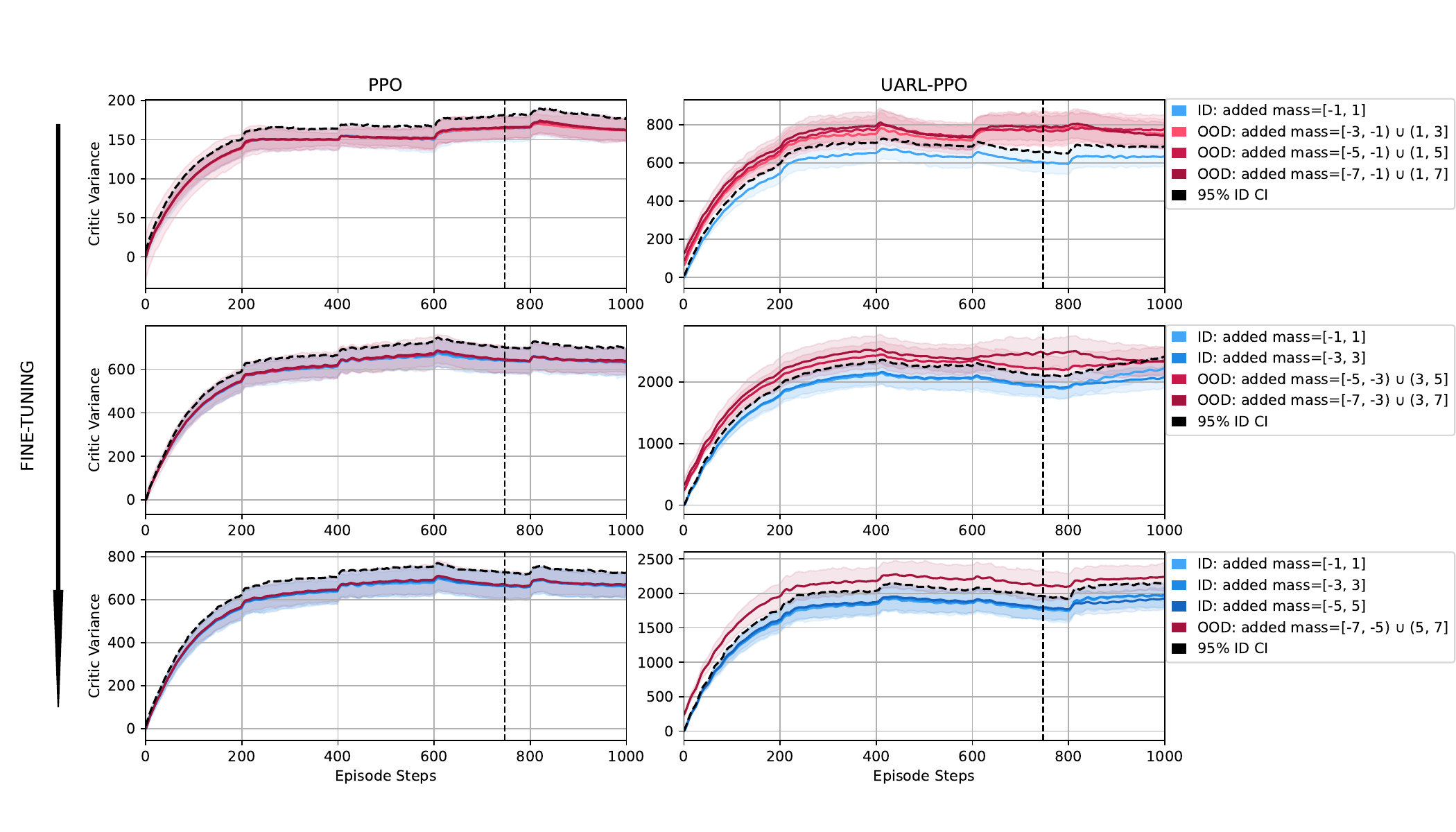}
    \caption{Critic variance across $100$ rollouts in the ANYmal environment for PPO and \ourMethod-PPO methods. The randomized hyperparameter is the added base mass. Each column represents an algorithm, and each row represents a fine-tuning iteration with an expanded ID range shown in the plot's legend. The black dashed line indicates the $95\%$ CI of critic variances for ID samples, serving as an OOD detection threshold. Ideally, a method that is appropriately ``OOD-aware'' should have all the blue-shade curves below the black dashed curve (low uncertainty on ID episodes) and all the red-shade curves above (high uncertainty on OOD episodes). \ourMethod-PPO consistently distinguishes ID from OOD samples, while baseline struggles to do so.
    }

    \label{fig:ood_anymal_mass}
\end{figure}

\begin{figure}[t]
    \centering
    \includegraphics[width=\textwidth]{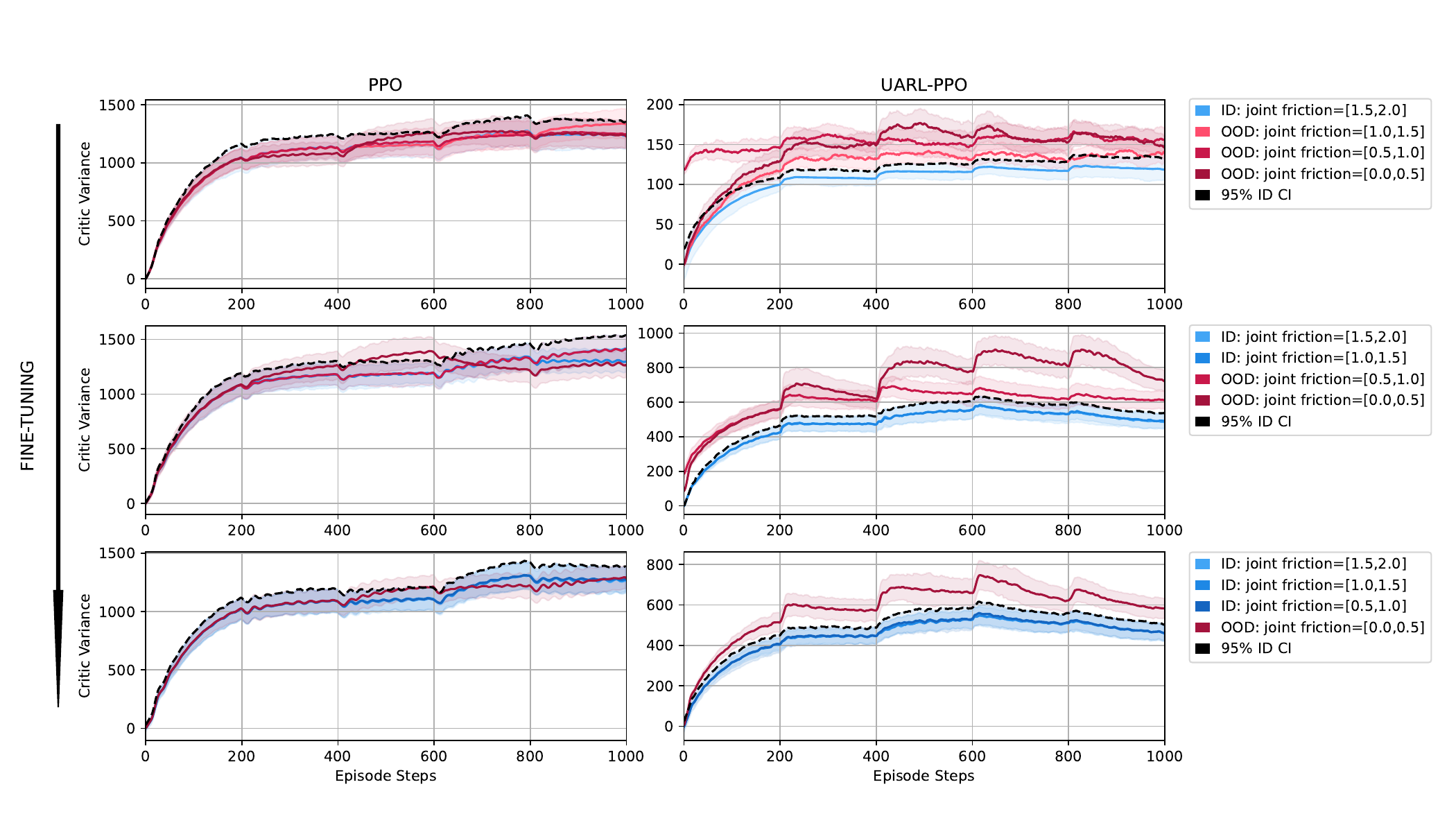}
    \caption{Critic variance across $100$ rollouts in the ANYmal environment for PPO and \ourMethod-PPO methods. The randomized hyperparameter is the ground friction. Each column represents an algorithm, and each row represents a fine-tuning iteration with an expanded ID range shown in the plot's legend. The black dashed line indicates the $95\%$ CI of critic variances for ID samples, serving as an OOD detection threshold. Ideally, a method that is appropriately ``OOD-aware'' should have all the blue-shade curves below the black dashed curve (low uncertainty on ID episodes) and all the red-shade curves above (high uncertainty on OOD episodes). \ourMethod-PPO consistently distinguishes ID from OOD samples, while baseline struggles to do so.}

    \label{fig:ood_anymal_ground_friction}
\end{figure}

\subsubsection{Safety Implications of Ignoring Policy Uncertainty}\label{app:anymal_safety}
While raw locomotion performance (e.g., forward velocity, command following) is often the primary focus in simulation, deploying an ``optimized'' policy without any measure of epistemic uncertainty can be hazardous on real hardware. In particular, imperfect domain randomization, due to mis-specified initial parameter bounds, leads to biased environmental estimates that manifest as unsafe behaviors:
\begin{itemize}
  \item Over-estimated mass range: If the simulated payload mass is randomized to values higher than the true robot mass, the learned policy may adopt ``over-lift'' gaits: the base is held unrealistically high, shifting the center of mass beyond the robot's stability margin and causing tip-overs (\autoref{fig:over_under} (a)).
  \item Under-estimated mass range: Conversely, randomizing mass to values below the true mass drives the policy to crouch excessively close to the ground. Such low stances increase the risk of foot collisions, excessive joint torques, and premature mechanical wear (\autoref{fig:over_under} (b)).
  \item Over-estimated friction: If ground friction is randomized toward overly high values, the policy learns to rely on traction that does not exist on hardware. When confronted with the true, lower friction surface, even nominal gaits can slip catastrophically, leading to falls (\autoref{fig:over_under} (c)).
\end{itemize}

These failure modes are invisible to standard return-based validation in simulation, since the policy ``walks'' despite operating under incorrect dynamics. By contrast, our uncertainty-aware deployment gate (\autoref{sec:uarl}) uses critic-ensemble variance on proxy real-world data as an explicit safety check, ensuring that only policies whose epistemic uncertainty falls below the measured threshold are deployed to the actual robot. This mechanism prevents over/under confident behaviors that could compromise balance, damage hardware, or endanger bystanders.

\begin{figure}[htbp]
    \centering
    \begin{subfigure}{}
        \centering
        \includegraphics[width=0.25\textwidth]{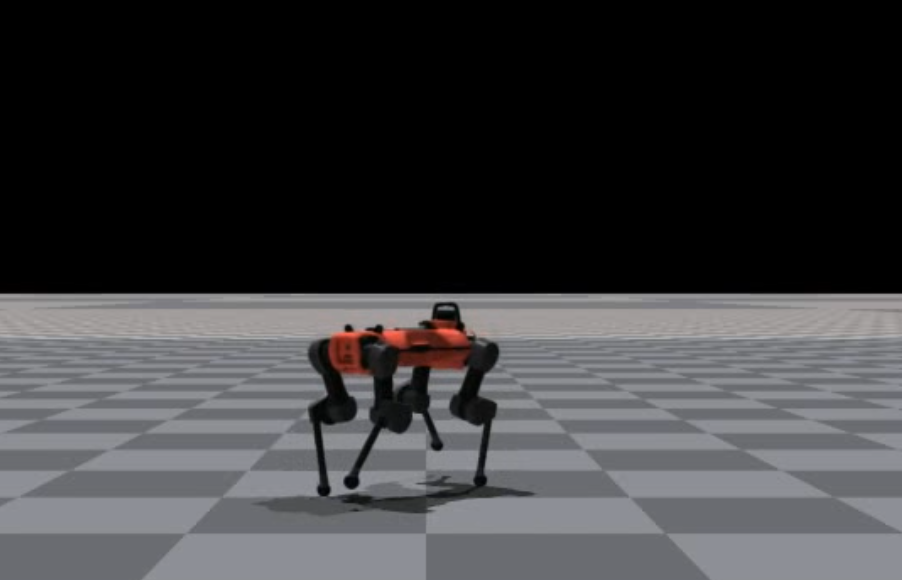}
    \end{subfigure}
    \begin{subfigure}{}
        \centering
        \includegraphics[width=0.25\textwidth]{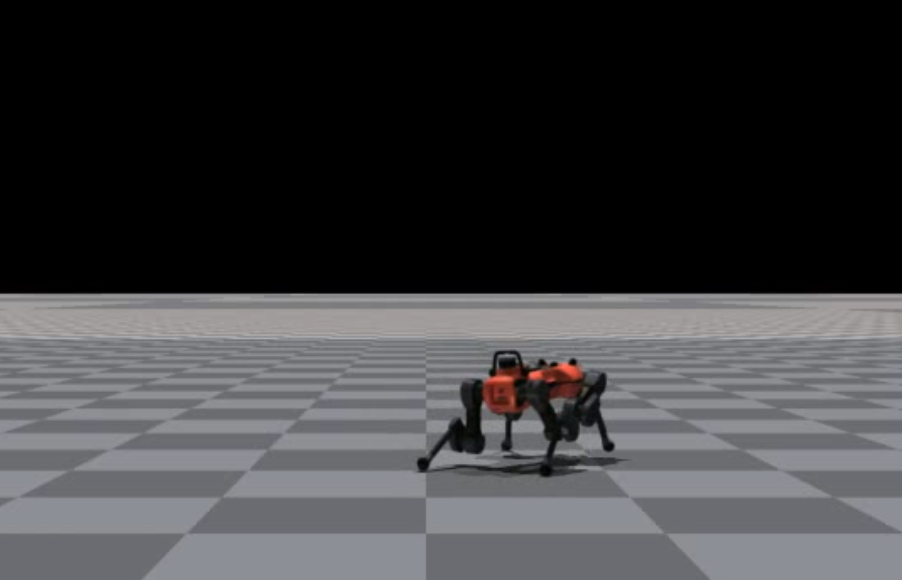}
    \end{subfigure}
    \begin{subfigure}{}
        \centering
        \includegraphics[width=0.25\textwidth]{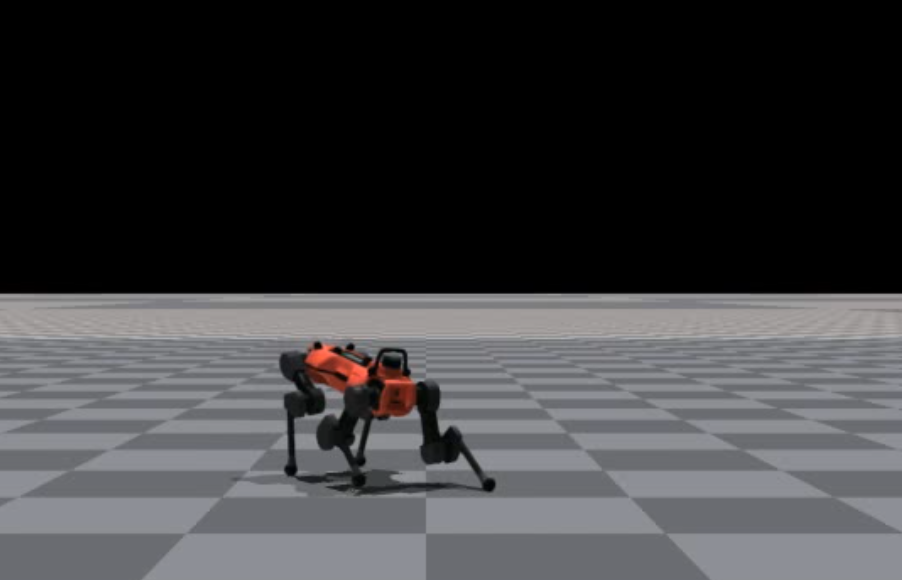}
    \end{subfigure}
    \caption{Safety-critical failure modes resulting from mis-specified domain randomization bounds: (a) over-estimated payload mass induces ``over‐lift'' gaits, shifting the robot's center of mass beyond its stability margin and causing tip-overs; (b) under-estimated payload mass drives excessively low stances, increasing the risk of foot collisions, high joint torques, and mechanical wear; (c) over-estimated ground friction leads to reliance on traction that does not exist on hardware, resulting in catastrophic slips and falls.}
    \label{fig:over_under}
\end{figure}

\clearpage
\newpage

\section*{NeurIPS Paper Checklist}

\begin{enumerate}

\item {\bf Claims}
    \item[] Question: Do the main claims made in the abstract and introduction accurately reflect the paper's contributions and scope?
    \item[] Answer: \answerYes{} 
    \item[] Justification: See \autoref{sec:uarl} and \autoref{sec:exp}.

\item {\bf Limitations}
    \item[] Question: Does the paper discuss the limitations of the work performed by the authors?
    \item[] Answer: \answerYes{} 
    \item[] Justification: See \autoref{sec:conc_limit}, \aref{app:limits}, and \aref{app:theory}.

\item {\bf Theory assumptions and proofs}
    \item[] Question: For each theoretical result, does the paper provide the full set of assumptions and a complete (and correct) proof?
    \item[] Answer: \answerYes{} 
    \item[] Justification: See \aref{app:theory}.

    \item {\bf Experimental result reproducibility}
    \item[] Question: Does the paper fully disclose all the information needed to reproduce the main experimental results of the paper to the extent that it affects the main claims and/or conclusions of the paper (regardless of whether the code and data are provided or not)?
    \item[] Answer: \answerYes{} 
    \item[] Justification: See \autoref{sec:exp}, \aref{app:exp_setup}, and \aref{app:ablation}.

\item {\bf Open access to data and code}
    \item[] Question: Does the paper provide open access to the data and code, with sufficient instructions to faithfully reproduce the main experimental results, as described in supplemental material?
    \item[] Answer: \answerYes{} 
    \item[] Justification: See the last paragraph of \autoref{sec:intro}.

\item {\bf Experimental setting/details}
    \item[] Question: Does the paper specify all the training and test details (e.g., data splits, hyperparameters, how they were chosen, type of optimizer, etc.) necessary to understand the results?
    \item[] Answer: \answerYes{} 
    \item[] Justification: See \autoref{sec:exp}, \aref{app:exp_setup}, and \aref{app:ablation}.

\item {\bf Experiment statistical significance}
    \item[] Question: Does the paper report error bars suitably and correctly defined or other appropriate information about the statistical significance of the experiments?
    \item[] Answer: \answerYes{} 
    \item[] Justification: See \autoref{sec:exp} and \aref{app:ablation}.

\item {\bf Experiments compute resources}
    \item[] Question: For each experiment, does the paper provide sufficient information on the computer resources (type of compute workers, memory, time of execution) needed to reproduce the experiments?
    \item[] Answer: \answerYes{} 
    \item[] Justification: See \aref{app:limits}.
    
\item {\bf Code of ethics}
    \item[] Question: Does the research conducted in the paper conform, in every respect, with the NeurIPS Code of Ethics \url{https://neurips.cc/public/EthicsGuidelines}?
    \item[] Answer: \answerYes{} 
    \item[] Justification: \justificationTODO{}

\item {\bf Broader impacts}
    \item[] Question: Does the paper discuss both potential positive societal impacts and negative societal impacts of the work performed?
    \item[] Answer: \answerNA{} 
    \item[] Justification: \justificationTODO{}
    
\item {\bf Safeguards}
    \item[] Question: Does the paper describe safeguards that have been put in place for responsible release of data or models that have a high risk for misuse (e.g., pretrained language models, image generators, or scraped datasets)?
    \item[] Answer: \answerNA{} 
    \item[] Justification: \justificationTODO{}

\item {\bf Licenses for existing assets}
    \item[] Question: Are the creators or original owners of assets (e.g., code, data, models), used in the paper, properly credited and are the license and terms of use explicitly mentioned and properly respected?
    \item[] Answer: \answerYes{} 
    \item[] Justification: \justificationTODO{}

\item {\bf New assets}
    \item[] Question: Are new assets introduced in the paper well documented and is the documentation provided alongside the assets?
    \item[] Answer: \answerNA{} 
    \item[] Justification: \justificationTODO{}

\item {\bf Crowdsourcing and research with human subjects}
    \item[] Question: For crowdsourcing experiments and research with human subjects, does the paper include the full text of instructions given to participants and screenshots, if applicable, as well as details about compensation (if any)? 
    \item[] Answer: \answerNA{} 
    \item[] Justification: \justificationTODO{}

\item {\bf Institutional review board (IRB) approvals or equivalent for research with human subjects}
    \item[] Question: Does the paper describe potential risks incurred by study participants, whether such risks were disclosed to the subjects, and whether Institutional Review Board (IRB) approvals (or an equivalent approval/review based on the requirements of your country or institution) were obtained?
    \item[] Answer: \answerNA{} 
    \item[] Justification: \justificationTODO{}

\item {\bf Declaration of LLM usage}
    \item[] Question: Does the paper describe the usage of LLMs if it is an important, original, or non-standard component of the core methods in this research? Note that if the LLM is used only for writing, editing, or formatting purposes and does not impact the core methodology, scientific rigorousness, or originality of the research, declaration is not required.
    \item[] Answer: \answerNA{} 
    \item[] Justification: \justificationTODO{}

\end{enumerate}

\end{document}